\documentclass{article}

\usepackage{microtype}
\usepackage{makecell}
\usepackage{graphicx}
\usepackage{subfigure}
\usepackage{adjustbox}
\usepackage{multirow}
\usepackage{multicol}
\usepackage{booktabs} 
\usepackage{enumitem}

\usepackage{hyperref}

\usepackage[accepted]{icml2025}

\usepackage{ragged2e}
\usepackage{setspace}

\usepackage{amsmath}
\usepackage{amssymb}
\usepackage{mathtools}
\usepackage{amsthm}
\usepackage{bbm}
\usepackage{dsfont}
\usepackage{bm}
\usepackage{caption}
\usepackage[capitalize,noabbrev]{cleveref}

\usepackage{mdframed}
\definecolor{magenta1}{cmyk}{0, 0.3, 0, 0}
\definecolor{cyan1}{cmyk}{0.3, 0, 0, 0}
\newcommand{\magentabg}[1]{\colorbox{magenta1}{#1}}
\newcommand{\cyanbg}[1]{\colorbox{cyan1}{#1}}

\definecolor{FPTcolor}{cmyk}{0, 0.85, 0.85, 0.50}
\newcommand{\FPT}{{\normalfont\texttt{\color{FPTcolor}FPT}}}
\newcommand{\FPTone}{{\normalfont\texttt{\color{FPTcolor}FPT1}}}
\newcommand{\FPTtwo}{{\normalfont\texttt{\color{FPTcolor}FPT2}}}
\newcommand{\grad}{{\normalfont\texttt{grad}}}
\newcommand{\err}{\textnormal{err}}

\usepackage{pgf, tikz, tikz-cd, pgfplots}
\newdimen\nodeDist
\newcommand{\bbE}{\mathbb E}

\newcommand{\bbI}{\mathbb I}

\newcommand{\bbP}{\mathbb P}	

\newcommand{\bbR}{\mathbb R}

\newcommand\cB{\mathcal B}

\newcommand\cF{\mathcal F}

\newcommand\cI{\mathcal I}
\newcommand\cJ{\mathcal J}

\newcommand\cL{\mathcal L}

\newcommand\cN{\mathcal N}
\newcommand\cO{\mathcal O}

\newcommand\cQ{\mathcal Q}
\newcommand\cR{\mathcal R}
\newcommand\cS{\mathcal S}
\newcommand\cT{\mathcal T}
\newcommand\cU{\mathcal U}

\newcommand\cX{\mathcal X}

\newcommand{\norm}[1]{\left\lVert#1\right\rVert}

\DeclareMathOperator*{\argmax}{arg\,max}
\DeclareMathOperator*{\argmin}{arg\,min}

\newcommand{\var}{\text{Var}}

\newcommand{\cov}{\text{Cov}}

\DeclareFontEncoding{FMS}{}{}
\DeclareFontSubstitution{FMS}{futm}{m}{n}
\DeclareFontEncoding{FMX}{}{}
\DeclareFontSubstitution{FMX}{futm}{m}{n}
\DeclareSymbolFont{fouriersymbols}{FMS}{futm}{m}{n}
\DeclareSymbolFont{fourierlargesymbols}{FMX}{futm}{m}{n}
\DeclareMathDelimiter{\VERT}{\mathord}{fouriersymbols}{152}{fourierlargesymbols}{147}

\theoremstyle{plain}
\newtheorem{theorem}{Theorem}[section]
\newtheorem{proposition}[theorem]{Proposition}
\newtheorem{lemma}[theorem]{Lemma}

\theoremstyle{definition}

\theoremstyle{remark}

\usepackage[textsize=tiny]{todonotes}

\icmltitlerunning{Generalized Random Forests using Fixed-Point Trees}

\begin{document}

\twocolumn[
\icmltitle{Generalized Random Forests using Fixed-Point Trees}




\icmlsetsymbol{equal}{*}
\begin{icmlauthorlist}
\icmlauthor{David Fleischer}{equal,math}
\icmlauthor{David A. Stephens}{equal,math}
\icmlauthor{Archer Y. Yang}{equal,math,mila}
\end{icmlauthorlist}

\icmlaffiliation{math}{Department of Mathematics and Statistics, McGill University, Montreal, Canada}
\icmlaffiliation{mila}{Mila - Quebec AI Institute, Montreal, Quebec, Canada}

\icmlcorrespondingauthor{Archer Y. Yang}{archer.yang@mcgill.ca}

\icmlkeywords{Generalized Random Forests, Fixed-Point Methods, Ensemble Methods, Causal Inference, Computational Efficiency}

\vskip 0.3in
]



\printAffiliationsAndNotice{}  

\begin{abstract}
We propose a computationally efficient alternative to generalized random forests (GRFs) for estimating heterogeneous effects in large dimensions. While GRFs rely on a gradient-based splitting criterion, which in large dimensions is computationally expensive and unstable, our method introduces a fixed-point approximation that eliminates the need for Jacobian estimation. This gradient-free approach preserves GRF’s theoretical guarantees of consistency and asymptotic normality while significantly improving computational efficiency. We demonstrate that our method achieves a speedup of multiple times over standard GRFs without compromising statistical accuracy. Experiments on both simulated and real-world data validate our approach. Our findings suggest that the proposed method is a scalable alternative for localized effect estimation in machine learning and causal inference applications.

\end{abstract}

\section{Introduction}
In many real-world machine learning (ML) applications, practitioners seek to estimate how quantities of interest vary across different feature subgroups rather than assuming uniform effects. For example, medical interventions and policy treatments often have heterogeneous impacts across subpopulations, making localized estimation crucial for improving outcomes \citep{imai2013estimating, knaus2021machine, murdoch2019definitions, lee2020deep}. Similarly, individualized recommendation systems adapt to user-specific features to enhance performance \citep{kohavi2013online}.

A key example of localized estimation arises in causal inference, where modern applications prioritize individualized treatment effects over average treatment effects \citep{neyman1923applications, rubin1974estimating}. The double machine learning framework \citep{chernozhukov2018double} unifies various ML-based causal estimation methods, including lasso \citep{belloni2017program}, random forests \citep{athey2019generalized, cevid2022distributional}, boosting \citep{powers2018some}, deep learning \citep{johansson2016learning, shalit2017estimating}, and general-purpose meta-algorithms \citep{nie2021quasi, kunzel2019metalearners}, all of which focus on capturing variation over feature space.

Generalized random forests (GRFs) \citep{athey2019generalized, wager2018estimation} have emerged as a powerful tool for such tasks, leveraging adaptive partitioning with problem-specific moment conditions instead of standard loss-based splits. GRFs apply broadly to a wide range of important statistical models -- local linear regression \citep{friedberg2020local}, survival analysis and missing data problems \citep{cui2023estimating}, nonparametric quantile regression, heterogeneous treatment effect estimation, and nonlinear instrumental variables regression \citep{athey2016recursive, athey2019generalized}. Unlike local linear models \citep{fan1995local, fan1996local, friedberg2020local} or kernel-based models \citep{staniswalis1989kernel, severini1994quasi, lewbel2007local, speckman1988kernel, robinson1988root} which suffer from the curse of dimensionality \citep{robins1997toward}, the tree-based approach of GRF offers a more scalable solution.

However, GRFs' gradient-based approach \citep{athey2019generalized} becomes computationally expensive and unstable in large dimensions due to the reliance on Jacobian estimators for tree splitting. To address this, we propose a gradient-free approach based on fixed-point iteration, eliminating the need for Jacobian estimation while retaining GRF's theoretical guarantees of consistency and asymptotic normality. Our method significantly improves computational efficiency while maintaining statistical accuracy, achieving significant speedups in experiments on simulated and real-world datasets.

\section{Background and Related Work}\label{sec:review-of-grf}
Given data $(X_i,O_i) \in \cX \times \cO$, GRF estimates a target function $\theta^*(x)$, defined as the solution to an estimating equation of the form
\begin{equation}\label{eqn:est-eqn} 
    0 = \bbE_{O|X}\left[ \psi_{\theta^*(x),\nu^*(x)}(O) \mid X =x\right], 
\end{equation}
for all $x \in \cX$, where $\psi$ is a score function that identifies the true $(\theta^*(x),\nu^*(x))$ as the root of \eqref{eqn:est-eqn}, and $\nu^*(x)$ is an optional nuisance function. GRF can be understood from a nearest-neighbor perspective as approximating $\theta^*(x)$  through a locally parametric $\theta^*$ within small neighborhoods of test point $x$. Suppose $L(x) \subset \{X_i\}_{i=1}^n$ is a subset of training observations of the covariates found in a region around $x \in \cX$ over which $\theta^*(x)$ can be well-approximated by a local parameter. Observations $X_i \in L(x)$ serve as local representatives for $x$ in estimating $\theta^*(x)$ such that, given sufficiently many training samples in a small enough neighborhood of $x$, an empirical version of \eqref{eqn:est-eqn} over $X_i \in L(x)$ defines an estimator $\hat\theta_{L(x)}$ that approaches $\theta^*(x)$,
\begin{align}
\label{eqn:emp-local-model-L}
(\hat\theta_{L(x)},\hat\nu_{L(x)}) \in& \argmin_{\theta,\nu}  \norm{ \sum^n_{i=1}\frac{\mathds 1 (X_i \in L(x))}{|L(x)|} \cdot \psi_{\theta,\nu}(O_i) }.
\end{align}
In GRF, the set of local representatives $L(x)$ is determined by tree-based partitions which divide the input space into disjoint regions, or leaves. The training samples $X_i$ that fall in the same leaf as $x$ form the subset $L(x)$. However, single trees are known to have high variance with respect to small changes in the training data \citep{amit1997shape, breiman1996bagging, breiman2001random, dietterich2000experimental}, leading to estimates \eqref{eqn:emp-local-model-L} that do not generalize well to values of $x$ that are not part of the training set. GRF improves its estimates by leveraging an estimating function that averages many estimating functions of the form \eqref{eqn:emp-local-model-L}. Specifically, let $L_b(x)$ denote the set of training covariates that fall in the same leaf as $x$, identified by a tree trained on an independent subsample of the data, indexed by $b = 1,\ldots, B$. The GRF estimator is obtained by aggregating the individual estimating functions \eqref{eqn:emp-local-model-L} across a forest of $B$ independently trained trees, i.e. the solution to the following forest-averaged estimating equation:
\begin{equation}\label{eqn:forest-avg-est-eqn}
    (\hat\theta(x),\hat\nu(x)) \in \argmin_{\theta,\nu} \norm{ \frac{1}{B}\sum^B_{b=1}\left(\sum^n_{i=1} \alpha_{bi}(x)  \psi_{\theta,\nu}(O_i) \right) }.
\end{equation}
where $\alpha_{bi}(x) \coloneqq \frac{\mathds 1(X_i \in L_b(x))}{|L_b(x)|}$. Define observational weights $\alpha_i(x)$ that measure the relative frequency with which training sample $X_i$ falls in the same leaf as $x$, averaged over $B$ trees:
\begin{equation}\label{eqn:grf-weights}
    \alpha_i(x) \coloneqq \frac{1}{B} \sum^B_{b = 1}\alpha_{bi}(x),  
\end{equation}
for $i = 1,\ldots,n$.  Then, the solution $(\hat\theta(x),\hat\nu(x))$ to the forest-averaged model \eqref{eqn:forest-avg-est-eqn} is equivalent to solving the following locally weighted estimating equation
\begin{equation}\label{eqn:emp-weighted-est-eqn}
(\hat\theta(x),\hat\nu(x)) \in \argmin_{\theta,\nu} \norm{\sum^n_{i = 1}\alpha_i(x) \psi_{\theta,\nu}(O_i)}.
\end{equation}
\citet{athey2019generalized} present \eqref{eqn:emp-weighted-est-eqn} as the definition of the GRF estimator, motivated in part by the mature analyses of local kernel methods \citep{newey1994kernel} alongside more recent work on tree-based partitioning and estimating equations \citep{athey2016recursive,zeileis2007generalized, zeileis2008model}. The GRF algorithm for estimating $\theta^*(x)$ can be summarized as a two-stage procedure. {\bf Stage I:} Use trees to calculate weight functions $\alpha_i(x)$ for any test observation $x \in \cX$, measuring the relative importance of the $i$-th training sample to estimating $\theta^*(\cdot)$ near $x$. {\bf Stage II:} Given a test observation $x \in \cX$, compute estimate $\hat\theta(x)$ of $\theta^*(x)$ by solving the locally weighted empirical estimating equation \eqref{eqn:emp-weighted-est-eqn}.

Our contribution improves the computational cost of Stage I by introducing a more efficient procedure to train the trees. Training the forest is the most resource-intensive step of GRF, and the cost of each split in the existing approach scales quadratically with the dimension of $\theta^*(x)$. We adopt a gradient-free splitting mechanism and significantly reduce both the time and memory demands of Stage I. Crucially, solving Stage II with weights $\alpha_i(x)$ following our streamlined Stage I produces an estimator $\hat\theta(x)$ that preserves the finite-sample performance and asymptotic guarantees of GRF.

\section{Our Method}
In this section we describe the details of our accelerated algorithm for GRF. We closely follow the approach of \citet{athey2019generalized}, and define $\hat\theta(x)$ as the solution to a locally weighted problem \eqref{eqn:emp-weighted-est-eqn} with weighting functions $\alpha_i(x)$ of the form \eqref{eqn:grf-weights}. The weight functions are induced by a collection of local subsets $\{L_b(x)\}_{b=1}^B$, such that each subset $L_b(x)$ is determined by the partition rules of a tree trained on a subsample. The construction of each tree, in turn, is determined by recursive splits of the subsample based on a splitting criterion designed to identify regions of $\cX$ that are homogeneous with respect to $\theta^*(x)$. Therefore, to fully specify the weight functions $\alpha_i(x)$, we must describe a feasible criterion for producing a split of $\cX$.

\subsection{The target tree-splitting criterion for Stage I}\label{sec:grf-target-criterion}

In GRF, the goal of Stage I is to use recursive tree-based splits of the training data to induce a partition over the input space. Each split starts with a parent node $P\subset\cX$ and results in child nodes $C_1, C_2\subset\cX$, defined by a binary, axis-aligned splitting rule of the form $C_1 = \{X_i\,:\,X_{i,\ell}\leq t\}$ and $C_2 = \{X_i\,:\,X_{i,\ell} > t\}$, where $\ell$ denotes a candidate splitting feature/axis and $t \in \bbR$ the splitting threshold. For a parent $P$ and any child nodes $C_1,C_2$ of $P$, let $(\hat\theta_P,\hat\nu_P)$ and $(\hat\theta_{C_j},\hat\nu_{C_j})$ denote local solutions analogous to \eqref{eqn:emp-local-model-L} defined over the samples in $P$ and $C_j$, respectively:
\begin{equation}\label{eqn:local-est-P}
(\hat\theta_P,\hat\nu_P) \in \argmin_{\theta,\nu}\norm{\sum_{\{i:X_i\in P\}} \psi_{\theta,\nu}(O_i) },
\end{equation}
\begin{equation}    (\hat\theta_{C_j},\hat\nu_{C_j}) \in \argmin_{\theta,\nu}\norm{\sum_{\{i:X_i\in C_j\}} \psi_{\theta,\nu}(O_i) },  \label{eqn:local-est-Cj}
\end{equation}
for $j = 1,2$. A strategy to split $P$ into two subsets of greater homogeneity with respect to $\theta^*(\cdot)$ is as follows: Find child nodes $C_1$ and $C_2$ such that the total deviation between the local solutions $\hat\theta_{C_j}$ and the target $\theta^*(X)$ is minimized, conditional on $X \in C_j$, $j = 1,2$. A natural measure of deviation is the squared-error loss,
\begin{align*}\label{eqn:crit-square-error}
    \err(C_1,C_2) \coloneqq & \sum_{j=1,2} \bbP\left(X \in C_j \mid X \in P\right) \\
    & \quad\times~\bbE\left[\norm{ \theta^*(X) - \hat\theta_{C_j}}^2 \;\middle|\; X\in C_j\right],
\end{align*}
such that the resulting split $(C_1,C_2)$ corresponds to least-squares optimal solutions $\hat\theta_{C_1}$ and $\hat\theta_{C_2}$. However, $\err(C_1,C_2)$ is intractable since $\theta^*(\cdot)$ is unknown. GRF considers a criterion that measures heterogeneity across a pair of local solutions over a candidate split
\begin{equation}\label{eqn:target-het-criterion}
    \Delta(C_1, C_2) \coloneqq \frac{n_{C_1}n_{C_2}}{n_P^2} \norm{\hat\theta_{C_1} - \hat\theta_{C_2}}^2,
\end{equation}
where $n_{C_1}$, $n_{C_2}$, and $n_P$ denote the number of observations in $C_1$, $C_2$, and $P$, respectively. In particular, rather than minimizing $\err(C_1,C_2)$, one can seek a split of $P$ such that the cross-split heterogeneity between $\hat\theta_{C_1}$ and $\hat\theta_{C_2}$ is maximized. \citet{athey2019generalized} observe that $\err(C_1,C_2)$ and $\Delta(C_1,C_2)$ are coupled according to $\err(C_1,C_2) = K(P) - \bbE\left[\Delta(C_1,C_2)\right] + o(r^2)$, where $r > 0$ is a small radius term tied to the sampling variance, and $K(P)$ does not depend on the split of $P$. That is, splits that maximize $\Delta(C_1,C_2)$ -- which emphasize the heterogeneity of $\hat\theta_{C_j}$ across a split -- will asymptotically minimize $\err(C_1,C_2)$, which aims to improve the homogeneity of $\hat\theta_{C_j}$ within a split. 

Although the criterion $\Delta(C_1,C_2)$ is computable, evaluating it is very computationally expensive since it requires solving \eqref{eqn:local-est-Cj} to obtain $\hat\theta_{C_1},\hat\theta_{C_2}$ for all possible splits of $P$, and closed-form solutions for $\hat\theta_{C_j}$ are generally not available except in special cases of $\psi$. Instead, GRF approximates the target $\Delta$-criterion based on a criterion of the form
\begin{equation}\label{eqn:approx-surrogate}
\widetilde\Delta^\grad(C_1,C_2) \coloneqq \frac{n_{C_1}n_{C_2}}{n_P^2} \norm{\tilde\theta^\grad_{C_1} - \tilde\theta^\grad_{C_2} }^2,
\end{equation}
where $\tilde\theta^\grad_{C_j}$ denotes a \emph{gradient-based} approximation of $\hat\theta_{C_j}$. Specifically, $\tilde\theta^\grad_{C_j}$ is a first-order approximation interpreted as the result of taking a gradient step away from the parent estimate in the direction towards the true child solution $\hat\theta_{C_j}$:
\begin{equation}\label{eqn:gradient-estimator}
\tilde\theta^\grad_{C_j} \coloneqq \hat\theta_P - \frac{1}{n_{C_j}}\sum_{\{i:X_i\in C_j\}} \xi^\top A_P^{-1} \psi_{\hat\theta_P,\hat\nu_P}(O_i),
\end{equation}
where $(\hat\theta_P,\hat\nu_P)$ is the local solution over the parent, $A_P$ is any consistent estimator of the local Jacobian matrix $\nabla_{(\theta,\nu)}\bbE[\psi_{\hat\theta_P,\hat\nu_P}(O_i)\mid X_i\in P]$, and $\xi^\top$ can be thought of as a term that selects a $\theta$-subvector from a $(\theta,\nu)$-vector, e.g. if $\theta \in \bbR^K$ and $\nu \in \bbR$, then $\xi^\top$ such that $\theta = \xi^\top (\theta,\nu)^\top$ is the rectangular diagonal matrix $\xi^\top = \renewcommand\arraystretch{0.75}\left[\begin{matrix} \bbI_K & {\bf 0} \end{matrix}\right]$. When the scoring function $\psi$ is continuously differentiable in $(\theta,\nu)$, the Jacobian estimator $A_P$ can be computed as
\begin{align}\label{eqn:AP-matrix}
    A_P &= \nabla_{(\theta,\nu)} \frac{1}{n_{P}}\sum_{\{i:X_i\in P\}} \psi_{\hat\theta_P,\hat\nu_P}(O_i) \nonumber\\
    &= \frac{1}{n_{P}}\sum_{\{i:X_i\in P\}} \nabla_{(\theta,\nu)} \psi_{\hat\theta_P,\hat\nu_P}(O_i).
\end{align}

\subsection{Limitations of gradient-based approximation}\label{sec:limitations}

The use of the Jacobian estimator $A_P$ in \eqref{eqn:gradient-estimator} introduces considerable computational challenges. First, each parent node $P$ in every tree of the forest requires a distinct $A_P$ matrix, which imposes a significant computational burden when explicitly calculating $A_P^{-1}\psi_{\hat\theta_P,\hat\nu_P}(O_i)$ to determine $\tilde\theta^\grad_{C_j}$. Second, if the local Jacobian $\nabla_{(\theta,\nu)}\bbE[\psi_{\hat\theta_P,\hat\nu_P}(O_i)\mid X_i\in P]$ is ill-conditioned, then the resulting $A_P$ estimator may be nearly singular. This instability can lead to highly variable gradient-based approximations $\tilde\theta^\grad_{C_j}$ and highly variable splits of $P$. For example, consider the following varying-coefficient model for an outcome $Y_i$ given regressors $W_i = (W_{i,1},\ldots,W_{i,K})^\top$ in the presence of mediating auxiliary covariates $X_i$:
\begin{equation}\label{eqn:example-model-limitations}
    \bbE[Y_i\mid X_i=x] = \nu^*(x) + W_i^\top \theta^*(x),
\end{equation}
where $\nu^*(\cdot)$ is a nuisance intercept function and $\theta^*(x) = (\theta^*_1(x),\ldots,\theta^*_K(x))^\top$ are the target coefficients. Models of the form \eqref{eqn:example-model-limitations} encompass time- or spatially-varying coefficient frameworks, where $(X_i,Y_i,W_i)$ represent the $i$-th sample associated with spatiotemporal values $X_i$. Such models are particularly relevant in applications like heterogeneous treatment effects; see Section~\ref{sec:applications} for a more in-depth discussion. The local estimating function $\psi_{\theta,\nu}(Y_i,W_i)$, identifying $(\theta^*(x),\nu^*(x))$ through moment conditions as in \eqref{eqn:est-eqn}, is given by:
\begin{equation*}
    \psi_{\theta,\nu}(Y_i,W_i) \coloneqq \begin{bmatrix}
        (Y_i - W_i^\top \theta - \nu) \cdot W_i \\
        Y_i - W_i^\top \theta - \nu
    \end{bmatrix}.
\end{equation*}
Consequently, the corresponding local Jacobian estimator is
\begin{align}\nonumber
     A_P =& \frac{1}{n_P}\sum_{\{i:X_i\in P\}} \nabla_{(\theta,\nu)} \psi_{\theta,\nu}(Y_i,W_i) \\
      = &  -\frac{1}{n_P}\sum_{\{i:X_i\in P\}} \begin{bmatrix}
         W_iW_i^\top & W_i^\top \\
         W_i & 1\label{eqn:AP-matrix-hte}
     \end{bmatrix}.
\end{align}
When the regressors are highly correlated, the summation over the $W_iW_i^\top$ block of the $A_P$ matrix leads to nearly singular values of $A_P$, resulting in an unstable matrix inverse $A_P^{-1}$, and therefore unstable values of $\tilde\theta^\grad_{C_j}$ and unstable splits. This issue becomes more pronounced as the number of parent samples $n_P$ decreases, as is the case at deeper levels of the tree. These challenges highlight the limitations of relying on $A_P$ as part of an approximation for the child solutions $\hat\theta_{C_j}$.

As an illustration, consider a simple varying coefficient model with primary regressors $W_{i,1}, W_{i,2} \sim \cN(0,1)$, auxiliary covariates $X_i \sim \text{Unif}(0,1)$, and outcomes $Y_i$ generated as
\begin{equation}\label{eqn:varying-coef-example}
    Y_i = \mathds 1(X_{i} > 0.5) W_{i,1} + W_{i,2} + \epsilon_i,
\end{equation}
where $\epsilon_i \sim \cN(0,1)$. Figure~\ref{fig:varying-coef} illustrates the distribution of 2000 $\widetilde\Delta^\grad$-optimal binary splits (gradient-based tree stumps) fit over 1000 samples of the varying coefficient model \eqref{eqn:varying-coef-example}, repeated over different regressor correlation levels $\textnormal{Corr}(W_{i,1},W_{i,2}) \in \{0.80, 0.81, \ldots, 0.98, 0.99\}$. It is clear that splits based on the $\widetilde\Delta^\grad$-criterion exhibit high variability when the correlation between the regressors is large. In contrast, our proposed method, discussed in the next section, does not suffer from the same problem.

\begin{figure}[t]
    \centering    
    \includegraphics[width=0.99\linewidth]{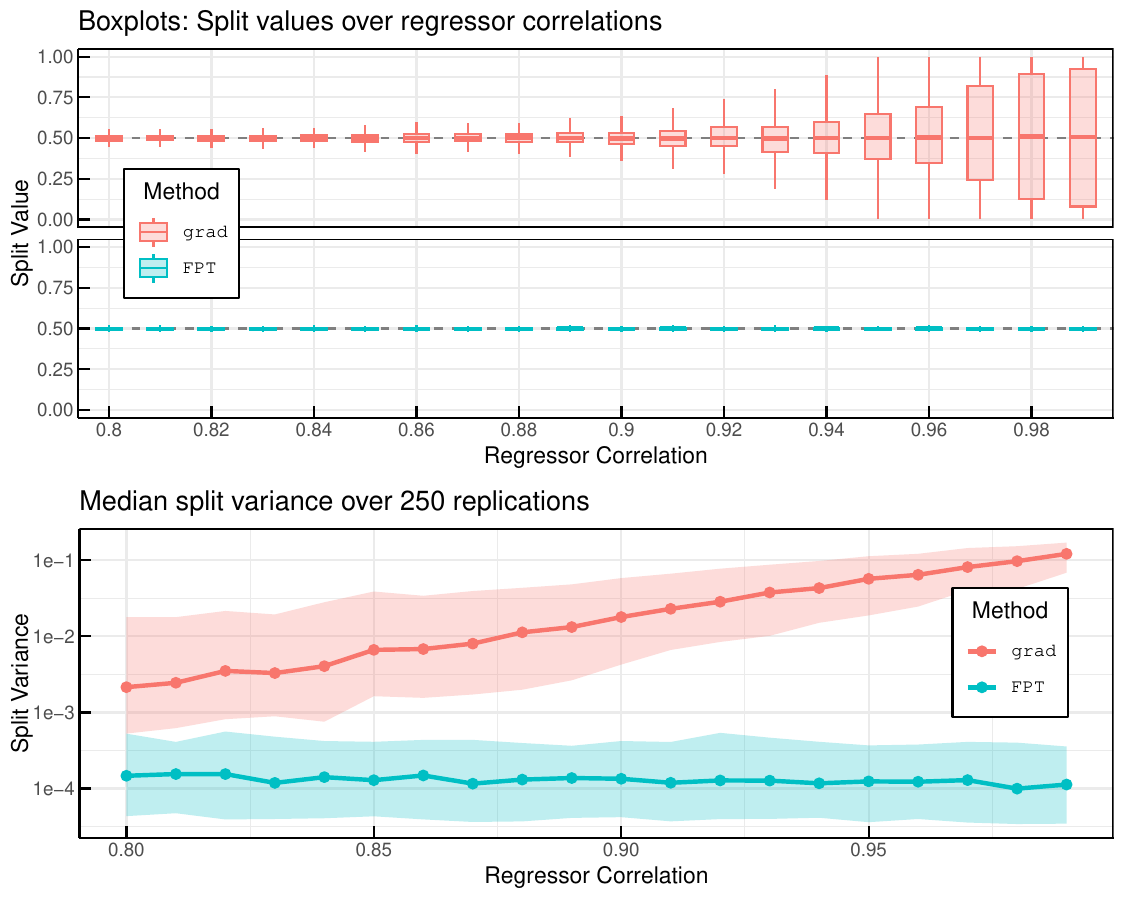}
    \caption{Splits values (top) and split variance (bottom), with 10th and 90th percentile bands, across correlations of $W_{i,1}$ and $W_{i,2}$.}
    \label{fig:varying-coef}
\end{figure}

\subsection{Fixed-point approximation}\label{sec:fp-tree-alg}

To address the limitations of gradient-based approximations, we propose a gradient-free approach based on the form of a single fixed-point iteration. Let $\Psi_{C_j}(\theta,\nu) \coloneqq \frac{1}{n_{C_j}}\sum_{\{i:X_i \in C_j\}} \psi_{\theta,\nu}(O_i)$ denote the empirical estimating function for the child solution $(\hat\theta_{C_j},\hat\nu_{C_j})$ such that \eqref{eqn:local-est-Cj} is equivalently written as:
\begin{equation}\label{eqn:Psi-Cj}
(\hat\theta_{C_j},\hat\nu_{C_j}) \in \argmin_{\theta,\nu}\norm{\Psi_{C_j}(\theta,\nu) }, \quad j = 1,2.
\end{equation}
Under mild regularity conditions,   $(\hat\theta_{C_j},\hat\nu_{C_j})$ is a $Z$-estimator that solves the estimating equation $  \Psi_{C_j}(\theta,\nu) = {\bf 0}$. Reformulating this equation as a fixed-point problem, we write:
\begin{equation}\label{eqn:fixed-point-general}
    (\theta,\nu) = \underbrace{(\theta,\nu) - \eta \Psi_{C_j}(\theta,\nu)}_{\eqqcolon f(\theta,\nu)}, \quad \eta > 0.
\end{equation}
A necessary and sufficient condition for $(\hat\theta_{C_j},\hat\nu_{C_j})$ to be a solution of \eqref{eqn:Psi-Cj} is characterized by the fixed-point problem $(\hat\theta_{C_j},\hat\nu_{C_j}) = f(\hat\theta_{C_j},\hat\nu_{C_j})$,
where $f$ is as defined in \eqref{eqn:fixed-point-general}. Iterative fixed-point methods \citep{Picard1890,Lindelof1894,Banach1922,RyuBoyd2016,yang2021flexible} solve such problems by considering an update rule of the form
\begin{equation}\label{eqn:fp-iteration-general}
    (\theta^{+},\nu^{+}) \leftarrow f(\theta,\nu).
\end{equation}
The form of \eqref{eqn:fp-iteration-general} inspires us to approximate the true child solution $\hat\theta_{C_j}$ using a single fixed-point update taken from the parent solution $\hat\theta_P$:
\begin{align}\label{eqn:fp-estimator-eta}
\tilde\theta^\FPT_{C_j} \coloneqq & \hat\theta_P - \eta\xi^\top\Psi_{C_j}(\hat\theta_P,\hat\nu_P) \nonumber\\
= & \hat\theta_P - \frac{\eta}{n_{C_j}}  \xi^\top  \sum_{\{i:X_i \in C_j\}} \psi_{\hat\theta_P,\hat\nu_P}(O_i),
\end{align}
where the product with $\xi^\top$ is interpreted similarly to its role in the gradient-based approximation \eqref{eqn:gradient-estimator} and to express the update \eqref{eqn:fp-iteration-general} solely in terms of the target $\theta$-quantity. We interpret $\tilde\theta^\FPT_{C_j}$ as an approximation of $\hat\theta_{C_j}$ obtained by taking a step from $\hat\theta_P$ in a direction that reduces the magnitude of the local estimating function $\Psi_{C_j}$. Notably, the approximation $\tilde\theta_{C_j}^\FPT$ does not involve the $A_P$ matrix, relying only on the scores $\psi_{\hat\theta_P,\hat\nu_P}(O_i)$ evaluated at the parent solutions. In general, removing the inverse $A_P^{-1}$ provides computational cost savings of $\cO(K^3)$. The corresponding splitting criterion, which uses the fixed-point approximations $\tilde\theta_{C_j}^\FPT$ as substitutes for $\hat\theta_{C_j}$ is given by
\begin{equation}\label{eqn:fp-approx-criterion}
    \widetilde\Delta^\FPT(C_1,C_2) \coloneqq \frac{n_{C_1}n_{C_2}}{n_P^2} \norm{ \tilde\theta^\FPT_{C_1} - \tilde\theta^\FPT_{C_2} }^2.
\end{equation} 

Revisiting the varying coefficient example from Section~\ref{sec:limitations}, we see that splits based on fixed-point approximations $\tilde\theta_{C_j}^\FPT$ are significantly more stable than those based on $\tilde\theta_{C_j}^\grad$. Specifically, Figure~\ref{fig:varying-coef} illustrates that splits that maximize $\widetilde\Delta^\FPT(C_1,C_2)$ are more robust to ill-conditioning in the underlying local Jacobian $\nabla_{(\theta,\nu)}\bbE[\psi_{\hat\theta_P,\hat\nu_P}(O_i)\mid X_i\in P]$, as is the case for highly correlated regressors in the varying coefficient model \eqref{eqn:varying-coef-example}, and leading to highly stable splits.

\subsection{Pseudo-outcomes}\label{sec:pseudo-outcomes}

Approximations $\tilde\theta_{C_j}$ of the form \eqref{eqn:gradient-estimator} and \eqref{eqn:fp-estimator-eta} offer an additional benefit: they enable the $\widetilde\Delta$-criteria of the form \eqref{eqn:approx-surrogate} and \eqref{eqn:fp-approx-criterion} to be efficiently optimized through a single multivariate CART split. A CART split performed with respect to vector-valued responses $\rho_i \in \bbR^K$ over a parent node $P$ produces a split $(C_1,C_2)$ that minimizes the following least-squares criterion:
\begin{equation}\label{eqn:cart-min-criterion-pseudo-outcomes}
    \sum_{\{i\,:\,X_i \in C_1\}} \norm{\rho_i - \bar \rho_{C_1}}^2 + \sum_{\{i\,:\,X_i \in C_2\}} \norm{\rho_i - \bar \rho_{C_2}}^2,
\end{equation}
where $\bar\rho_{C_j} \coloneqq \frac{1}{n_{C_j}}\sum_{\{i:X_i \in C_j\}} \rho_i$.\footnote{The multivariate CART criterion uses a sum of squares impurity measure, as in \citet{de2002multivariate, segal1992tree}.}~Equivalently, a CART split that minimizes \eqref{eqn:cart-min-criterion-pseudo-outcomes} will maximize:
\begin{equation}\label{eqn:cart-max-criterion-pseudo-outcomes}
n_{C_1}\norm{\bar\rho_{C_1}}^2 + n_{C_2} \norm{\bar\rho_{C_2}}^2.
\end{equation}
The equivalence between the split that minimizes the least-squares CART criterion \eqref{eqn:cart-min-criterion-pseudo-outcomes} and the split that maximizes \eqref{eqn:cart-max-criterion-pseudo-outcomes} is shown in Appendix~\ref{app:multivariate-cart-criteria}. GRF performs its splits by adopting gradient-based {\em pseudo-outcomes}, defined as
\begin{equation}\label{eqn:grad-pseudo-outcome}
    \rho_i^\grad\coloneqq-\xi^\top A_P^{-1}\psi_{\hat\theta_P,\hat\nu_P}(O_i)
\end{equation}
such that the gradient-based approximation $\tilde\theta_{C_j}^\grad$ in \eqref{eqn:gradient-estimator} is equivalently written:
\begin{equation*}
\tilde\theta^\grad_{C_j} = \hat\theta_P + \frac{1}{n_{C_j}}\sum_{\{i:X_i \in C_j\}} \rho_i^\grad = \hat\theta_P + \bar\rho_{C_j}^{~\grad}.
\end{equation*}
In the case of fixed-point approximation, we define fixed-point pseudo-outcomes:
\begin{equation}\label{eqn:fp-pseudo-outcome-eta}
    \rho_i^\FPT\coloneqq-\eta\xi^\top\psi_{\hat\theta_P,\hat\nu_P}(O_i), \quad \eta \neq 0,
\end{equation}
such that the fixed-point approximation $\tilde\theta_{C_j}^\FPT$ in \eqref{eqn:fp-estimator-eta} is equivalently written as
\begin{equation}\label{eqn:fp-estimator-eta-pseudo-outcome}
\tilde\theta^\FPT_{C_j} = \hat\theta_P + \frac{1}{n_{C_j}}\sum_{\{i:X_i \in C_j\}} \rho_i^\FPT = \hat\theta_P + \bar\rho_{C_j}^{~\FPT}.
\end{equation}
Substitute the above form of $\tilde\theta^\FPT_{C_j}$ into the $\widetilde\Delta^\FPT$-criterion \eqref{eqn:fp-approx-criterion} to equivalently express the criterion in terms of the $\FPT$ pseudo-outcomes:
\begin{equation}\label{eqn:fp-criterion-pseudo-outcome}
\widetilde\Delta^\FPT(C_1,C_2) = \frac{n_{C_1} n_{C_2}}{n_P^2} \norm{\bar\rho^{~\FPT}_{C_1} - \bar\rho^{~\FPT}_{C_2}}^2,
\end{equation}
where an analogous equivalence holds for $\widetilde\Delta^\grad$ in terms of the gradient-based pseudo-outcomes. We demonstrate in Lemma~\ref{lem:affine-equivalence-pseudo-outcome} (in Appendix~\ref{app:splits-via-CART-pseudo-outcomes}) that maximizing the fixed-point criterion $\widetilde\Delta^\FPT(C_1,C_2)$ is equivalent to maximizing the CART criterion \eqref{eqn:cart-max-criterion-pseudo-outcomes}, and extend this property to any $\widetilde\Delta$-style criterion induced by pseudo-outcomes that can be expressed as a split-independent linear transformation of the parent scores $\psi_{\hat\theta_P,\hat\nu_P}(O_i)$.

Note that our method does not rely on iterative fixed-point procedures at all. Instead, it uses only a single step of fixed-point approximation to simplify the pseudo-outcomes. These simplified pseudo-outcomes are then passed directly to a standard CART algorithm for splitting. The numerical convergence of our method therefore relies solely on CART's established and well-known stability, not on fixed-point iteration. CART splits on pseudo-outcomes are computationally efficient. Given a parent node $P$, the value $\rho_i = -B\psi_{\hat\theta_P,\hat\nu_P}(O_i)$ does not depend on a candidate split $(C_1,C_2)$ for any matrix $B$ that is fixed with respect to the parent. This allows much of the computation required to maximize $\widetilde\Delta^\FPT(C_1,C_2)$ to be done at the parent level, and in particular avoids re-calculating the approximations $\tilde\theta_{C_1}^\FPT$ and $\tilde\theta_{C_2}^\FPT$ across the sequence of candidate splits. Once $P$ is fixed and $\rho_i^\FPT$ are computed, the value of $\widetilde\Delta^\FPT(C_1,C_2)$ for the first candidate split requires $\cO(n_P)$ time, and the value for all other candidate splits of $P$ are queried in $\cO(1)$ time. While gradient-based pseudo-outcomes share this property, the use of fixed-point pseudo-outcomes eliminates the computational overhead and instability associated with estimating $A_P$, as discussed in Section~\ref{sec:limitations}.

We show in Lemma~\ref{lem:scale-invariance-cart} (Appendix~\ref{app:scale-invariance-eta}) that choosing different values of $\eta$ does not change the outcome of the fixed-point splitting mechanism. Specifically,
the optimal split identified by CART on pseudo-outcomes $\rho_i^\FPT$ of the form \eqref{eqn:fp-pseudo-outcome-eta} does not depend on $\eta$. This can be heuristically understood by studying how the criterion changes as a function of the candidate splits. To illustrate, we consider a VCM model of the form \eqref{eqn:example-model-limitations} for bivariate regressors $W_i$, univariate $X_i \in [0,1]$, and scalar outcomes $Y_i$. A detailed summary of the settings is found in Appendix~\ref{app:experiments-fp-tree-split-invariance}. The sequence of valid candidate child nodes obtained by a split over univariate $X_i$ can be parameterized through scalar $t$ as $C_1(t) \coloneqq \{X_i \,:\, X_i \leq t \}$ and $C_2(t) \coloneqq \{X_i \,:\, X_i > t\}$. Let $\Delta(t) \coloneqq \Delta(C_1(t), C_2(t))$ denote the parameterized target criterion \eqref{eqn:target-het-criterion}, and consider the behavior of $\Delta(t)$, $\widetilde\Delta^\grad(t)$, and two fixed-point criteria $\widetilde\Delta_1^\FPT(t)$ and $\widetilde\Delta_2^\FPT(t)$ of the form \eqref{eqn:fp-criterion-pseudo-outcome} based on pseudo-outcomes with scale factors $\eta = 1$ and $\eta = 1/\sqrt{2}$, respectively. Figure~\ref{fig:example-criteria-comparison} illustrates the different splitting criteria values plotted against the sequence of candidate splits. The visualization clearly shows that the criteria curves for $\Delta(t)$, $\widetilde\Delta^\grad(t)$, and $\widetilde\Delta_1^\FPT(t)$ with $\eta = 1$ are all very close to one another. Critically, the fixed-point criterion with $\eta = 1/\sqrt{2}$, i.e. $\widetilde\Delta_2^\FPT(t)$, although scaled differently, still identifies the same maximizing split as $\widetilde\Delta_1^\FPT(t)$. This is because CART chooses a split based on a rank ordering of the criterion over all candidate splits. The absolute scale of the CART criterion does not matter, and it is only criterion rankings over the candidates that determines the optimal split. Therefore, choosing a different scalar $\eta$ does not change the outcome of the splitting process.

Based on the scale-invariance of our splitting criterion, we now detail the recursive procedure for growing our fixed-point trees pseudo-outcomes with  $\eta = 1$.
 
\paragraph{The fixed-point tree algorithm.}

The entire fixed-point tree-growing procedure recursively applies the following two steps on a given parent node $P$:
\begin{enumerate}[label=(\roman*)]
    \item {\bf Labeling:} Solve \eqref{eqn:local-est-P} over $P$ to obtain the parent estimate $(\hat\theta_P,\hat\nu_P)$. Compute the pseudo-outcomes:
    \begin{equation}\label{eqn:fp-pseudo-outcome}
        \rho^\FPT_i \coloneqq -\xi^\top \psi_{\hat\theta_P,\hat\nu_P}(O_i),
    \end{equation}
    for all $i$ such that $X_i \in P$.
    \item {\bf Regression:} Maximize $\widetilde\Delta^\FPT(C_1,C_2)$ by performing a CART split on the pseudo-outcomes $\rho^\FPT_i$ over $P$.
\end{enumerate}

\begin{figure}[t]
    \centering
    \includegraphics[width=1\linewidth]{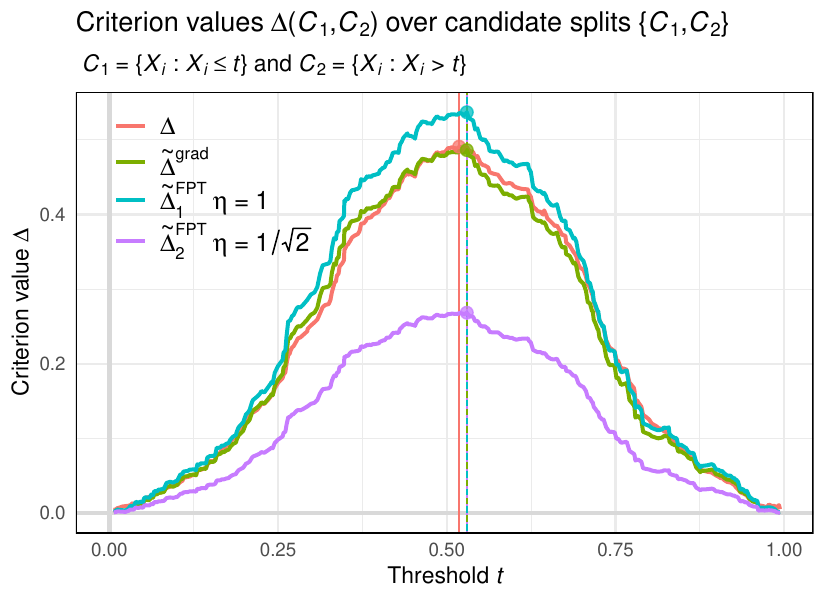}
    \caption{Criterion values across candidate splits $(C_1(t),C_2(t))$ over threshold $t \in [0,1]$. The location of the optimal split under each criterion is given by the corresponding vertical line.}
    \label{fig:example-criteria-comparison}
\end{figure}

\subsection{Estimates of $\hat\theta(x)$ for Stage II}\label{sec:stage2}

The fixed-point tree algorithm generates a single tree-based partition of $\cX$. Repeating this process over subsamples of the training data yields a forest of trees, each specifying local leaf functions $L_b(x)$. These leaf functions define the local weight functions $\alpha_i(x)$ via \eqref{eqn:grf-weights}, completing Stage I of GRF. 
The full fixed-point tree training algorithm is described in Algorithm~\ref{alg:grffpt-pseudocode-stage1-tree}, while Algorithm~\ref{alg:grffpt-pseudocode-stage1-forest} provides the pseudocode for the forest-wide Stage I procedure.

To compute the final GRF estimates $\hat{\theta}(x)$ for the target $\theta^*(x)$, we follow the standard GRF mechanism for Stage II. After the fixed-point trees are trained in Stage I, a test observation $x_0 \in \mathcal{X}$ is assigned to local leaves $L_b(x_0)$, indexed by trees $b \in \{1,\ldots,B\}$. Each leaf $L_b(x_0)$ contains the training observations that fall into the same leaf as $x_0$ in tree $b$. Using these local leaves, the forest computes training weights $\alpha_i(x_0)$ as in \eqref{eqn:grf-weights}. The final estimate $\hat{\theta}(x_0)$ is obtained by solving the locally weighted estimating equation \eqref{eqn:emp-weighted-est-eqn}.

Importantly, as discussed in Section~\ref{sec:review-of-grf}, solving for $\hat{\theta}(x_0)$ in Stage II is independent of the specific mechanism used in Stage I. The only requirement is that Stage I produces valid weights. This ensures that Stage II remains a standard weighted estimating equation, enabling the fixed-point tree algorithm to integrate seamlessly into GRF's two-stage framework. We refer to the complete algorithm for estimating $\theta^*(x)$ using fixed-point trees as GRF-$\FPT$. By preserving Stage II of GRF, the GRF-$\FPT$ estimator $\hat{\theta}(x)$ retains GRF's theoretical guarantees of consistency and asymptotic normality while offering a computationally efficient tree-building method. Pseudocode for Stage II of the GRF-$\FPT$ algorithm is provided in Algorithm~\ref{alg:grffpt-pseudocode-stage2}, located in Appendix~\ref{app:pseudocode}.

\section{Theoretical Analysis}\label{sec:theoretical-analysis}

In this section, we provide a theoretical foundation for the GRF-$\FPT$ estimator $\hat\theta(x)$. For Stage I, Proposition~\ref{prop:Delta-V-crit-asymptotic} establishes an asymptotic equivalence between the $\FPT$ criterion and a weighted oracle criterion $\Delta_V(C_1,C_2)$ in \eqref{eqn:Delta-V-defn}, while Lemma~\ref{lem:Delta-implies-Delta-V-spec} demonstrates that the Specifications~\ref{spec:all-specifications} are met by a forest based on the $\Delta_V$-criterion whenever they are met by a forest based on the $\Delta$-criterion. Assumptions~\ref{asm:all-assumptions} and Specifications~\ref{spec:all-specifications} are the sufficient conditions for the consistency and asymptotic normality of $\hat\theta(x)$ in \eqref{eqn:emp-weighted-est-eqn}, and thus are used to formally justify the $\FPT$ algorithm as a mechanism for specifying an estimator of $\theta^*(x)$.

\begin{proposition}\label{prop:Delta-V-crit-asymptotic}  Suppose Assumptions~\ref{asm:all-assumptions} hold, and assume moreover Neyman orthogonal moment conditions (defined in Appendix~\ref{app:neyman-example-vcm-hte}). Denote by $r \coloneqq \sup_{\{i:X_i\in P\}}\norm{X_i - x_P}$ the radius of the parent $P$, where $x_P$ denotes the center of mass over $X_i \in P$. Let $V_{\theta\theta}(x_P)$ denote the $\theta$-block of $V(x_P)$ in \eqref{eqn:V-block-form}. Denote by $\norm{\cdot}_V$ the weighted Euclidean norm $\norm{z}_V \coloneqq \norm{V_{\theta\theta}(x_P) z}_2 = \sqrt{z^\top V_{\theta\theta}^\top(x_P) V_{\theta\theta}(x_P) z}$. Define the weighted oracle criterion $\Delta_V(C_1,C_2)$:
\begin{equation}\label{eqn:Delta-V-defn}
    \Delta_V(C_1,C_2) \coloneqq \frac{n_{C_1}n_{C_2}}{n_P^2}\norm{\hat\theta_{C_1} - \hat\theta_{C_2}}_V^2.
\end{equation}
Then, treating the split as fixed with $r^{-2} \ll n_{C_1}, n_{C_2}$ and sufficiently small $r > 0$,
\begin{equation*}
    \widetilde\Delta^\FPT(C_1,C_2) = \Delta_{V}(C_1,C_2) + o_P\left(r^2,~\frac{1}{{n_{C_1}}},~\frac{1}{{n_{C_2}}}\right).
\end{equation*}
\end{proposition}

\begin{lemma}\label{lem:Delta-implies-Delta-V-spec} Let $\cT(\Delta)$ denote a tree whose splitting mechanism seeks splits that maximize $\Delta(C_1,C_2)$ defined in \eqref{eqn:target-het-criterion}, and let $\cT(\Delta_V)$ denote a tree whose splitting mechanism seeks splits that maximize $\Delta_V(C_1,C_2)$ defined in \eqref{eqn:Delta-V-defn}. Suppose Assumptions~\ref{asm:all-assumptions} hold and assume moreover that $\cT(\Delta)$ is a tree that satisfies Specifications~\ref{spec:all-specifications}. Then, $\cT(\Delta_V)$ satisfies Specifications~\ref{spec:all-specifications}.
\end{lemma}

For Stage II, Theorem~\ref{thm:consistency} establishes the consistency of the GRF-$\FPT$ estimator $\hat\theta(x)$:
\begin{theorem}\label{thm:consistency} Suppose that Assumptions~\ref{asm:all-assumptions} hold, and let $(\hat\theta(x),\hat\nu(x))$ be estimates that solve \eqref{eqn:emp-weighted-est-eqn} based on weights induced by a forest of trees grown under the fixed-point tree algorithm satisfying Specifications~\ref{spec:all-specifications}. Then, $(\hat\theta(x),\hat\nu(x))$ converges in probability to $(\theta^*(x),\nu^*(x))$.
\end{theorem}
The proof of Theorem~\ref{thm:consistency} follows directly from Theorem 3 of \citet{athey2019generalized}, which, under Assumptions~\ref{asm:all-assumptions}, establishes consistency for estimates $(\hat\theta(x),\hat\nu(x))$ that solve \eqref{eqn:emp-weighted-est-eqn} with weights from a forest that satisfies Specifications~\ref{spec:symmetric}-\ref{spec:honest}. 
Thanks to Lemma~\ref{lem:Delta-implies-Delta-V-spec}, these forest specifications must also apply to a forest grown under the $\FPT$ mechanism. Specifications~\ref{spec:symmetric}-\ref{spec:randomized} collectively impose mild boundary conditions on the splitting procedure. Meanwhile, Specification~\ref{spec:subsampling} requires that trees are trained on subsamples drawn without replacement \citep{biau2008consistency, scornet2015consistency, wager2014confidence,  wager2018estimation}, and Specification~\ref{spec:honest} requires that trees must be grown using an additional subsample splitting mechanism known as honesty \citep{athey2016recursive, biau2012analysis, denil2014narrowing}. Appendix~\ref{app:honest-indep} provides a detailed explanation of the subsampling and honest sample splitting procedure.

Finally, Theorem~\ref{thm:normality} establishes the asymptotic normality of the GRF-$\FPT$ estimator $\hat\theta(x)$:
\begin{theorem}\label{thm:normality} Under the conditions of Theorem~\ref{thm:consistency}, suppose moreover that Regularity Condition~\ref{regcond:variance} holds, and that a forest is grown on subsamples of size $s$ scaling as $s = n^\beta$, where $\beta$ satisfies Regularity Condition~\ref{regcond:sampling-size}. Then, there exists a sequence $\sigma_n(x)$ such that $(\hat\theta_n(x) - \theta^*(x))/\sigma_n(x) \leadsto \cN(0,1)$ and $\sigma^2_n(x) = \textnormal{polylog}(n/s)^{-1}s/n$, where $\textnormal{polylog}(n/s)$ is a function that is bounded away from 0 and increases at most polynomially with the log of the inverse sampling ratio $\log(n/s)$.
\end{theorem}
The proof of Theorem~\ref{thm:normality} is an immediate consequence of Theorem 5 of \citet{athey2019generalized}. Theorems~\ref{thm:consistency} and~\ref{thm:normality} demonstrate that the GRF-$\FPT$ estimator is able to meet key statistical guarantees.


\section{Applications}\label{sec:applications}

In this section, we explore applications of GRF-$\FPT$ for two related models: varying coefficient models and heterogeneous treatment effects. We consider an outcome model of the form introduced in Section~\ref{sec:limitations}. For each observation, let $Y_i$ denote the observed outcome, $W_i = (W_{i,1},\ldots,W_{i,K})^\top$ a $K$-dimensional regressor, and $X_i$ a set of mediating auxiliary variables, such that 
\begin{equation}\label{eqn:applications-model}
    Y_i = \nu^*(X_i) + W_i^\top \theta^*(X_i) + \epsilon_i,
\end{equation}
where $\nu^*(\cdot)$ is a nuisance intercept function, $\theta^*(x) = (\theta^*_1(x),\ldots,\theta^*_K(x))^\top$ are the target effect functions local to $X_i = x$, under the assumptions $\bbE[\epsilon_i \mid X_i=x]=0$ and $\bbE[\epsilon_i W_i \mid X_i=x]={\bf 0}$.

\noindent\textbf{Varying coefficient models (VCM).} Given regressors $W_i \in \bbR^K$, models of the form \eqref{eqn:applications-model} can be characterized as varying coefficient models \citep{hastie1993varying}. As discussed in Section~\ref{sec:limitations}, we must also assume that the regressors $W_i$ are conditionally exogenous given $X_i = x$.

\noindent\textbf{Heterogeneous treatment effects (HTE).} A special case of \eqref{eqn:applications-model} arises within the Neyman-Rubin potential outcome framework, which models the causal effect of treatment on an outcome \citep{neyman1923applications, rubin1974estimating}. Here, $\theta^*(x) = (\theta^*_1(x), \ldots, \theta^*_K(x))^\top$ represents heterogeneous treatment effects associated with $K$ discrete treatment levels. Let $T_i \in \{1, \ldots, K\}$ denote the observed treatment level for the $i$-th observation, and $Y_i(k)$ the potential outcome that would have been observed if treatment level $k$ had been applied. The regressors $W_i \in \{0,1\}^K$ in \eqref{eqn:applications-model} are interpreted as a vector of dummy variables indicating the observed treatment level, $W_{i,k} \coloneqq \mathds{1}(T_i = k)$. The auxiliary variables $X_i$ account for potential confounding effects. The conditional average treatment effect of treatment level $k \in \{2, \ldots, K\}$ relative to the baseline level $k = 1$ is then defined as:
\begin{equation*}
    \theta^*_k(x) \coloneqq \mathbb{E}\left[Y_i(k) - Y_i(1) \;\middle|\; X_i = x\right],
\end{equation*}
where the baseline contrast is set to $\theta^*_1(x) \coloneqq 0$.

Under exogeneity of the regressors, the target effects $\theta^*(x)$ in models \eqref{eqn:applications-model} are identified by moment conditions \eqref{eqn:est-eqn} for scoring function \citep{angrist2009mostly, athey2019generalized}
\begin{equation*}
    \psi_{\theta,\nu}(Y_i,W_i) \coloneqq \begin{bmatrix}
        (Y_i - W_i^\top \theta - \nu) \cdot W_i \\
        Y_i - W_i^\top \theta - \nu
    \end{bmatrix}.
\end{equation*}
The gradient-based pseudo-outcomes \eqref{eqn:grad-pseudo-outcome} are computed as
\begin{equation}\label{eqn:grad-het-treatment-effect-pseudo-outcomes}
    \rho_i^\grad = -A_P^{-1}  (W_i - \overline W_P) \left(Y_i - \overline Y_P - (W_i - \overline W_P)^\top\hat\theta_P\right),
\end{equation}
where $\overline W_P$ and $\overline Y_P$ are the local means of $W_i$ and $Y_i$ over the observations in $P$. Centering $Y_i - \overline Y_P$ and $W_i - \overline W_P$ removes the baseline effect of the mean $\hat\nu_P$ on $\rho_i^\grad$, and where $A_P$ is given by \eqref{eqn:AP-matrix-hte} as:
\begin{equation}\label{eqn:AP-matrix-applications}
     A_P = -\frac{1}{n_P}\sum_{\{i:X_i\in P\}} (W_i - \overline W_P)(W_i - \overline W_P)^\top.
\end{equation}
Computing $\rho_i^\grad$ in \eqref{eqn:grad-het-treatment-effect-pseudo-outcomes} involves the OLS coefficients $\hat\theta_P$ from regressing $Y_i - \overline Y_P$ on $W_i - \overline W_P$, over the observations in $P$:
\begin{equation}\label{eqn:vcm-hte-thetaP}
    \hat\theta_P \coloneqq -A_P^{-1} \frac{1}{n_P} \sum_{\{i:X_i \in P\}} (W_i - \overline W_P) (Y_i - \overline Y_P).
\end{equation}
In comparison, $\rho_i^\FPT$ in \eqref{eqn:fp-pseudo-outcome} are computed as:
\begin{align}
    \rho^\FPT_i &\coloneqq -\xi^\top\psi_{\hat\theta_P,\hat\nu_P}(Y_i,W_i),\nonumber\\
    &= -(W_i - \overline W_P) \left(Y_i - \overline Y_P - (W_i - \overline W_P)^\top\hat\theta_P\right),\label{eqn:fp-het-treatment-effect-pseudo-outcomes}
\end{align}
The relationship $\rho_i^\grad =  A_P^{-1}\rho_i^\FPT$ reveals a significant benefit of $\FPT$ pseudo-outcomes. The form of $\rho_i^\FPT$ eliminates the computational cost associated with the multiplication of $A_P^{-1}$, leading to $\cO(K^3)$ computational savings. Furthermore, the computation of $\hat\theta_P$ in \eqref{eqn:fp-het-treatment-effect-pseudo-outcomes} no longer requires solving for $A_P^{-1}$. Therefore, we can further enhance computational efficiency by using an accelerated form of pseudo-outcome $\phi_i^\FPT$ instead of $\rho_i^\FPT$:
\begin{equation}\label{eqn:accelerated-fpt-pseudo-outcome}
    \phi_i^\FPT \coloneqq -(W_i - \overline W_P) \left(Y_i - \overline Y_P - (W_i - \overline W_P)^\top\tilde\theta_P\right),
\end{equation}
where $\hat\theta_P$ is replaced by $\tilde\theta_P$  in \eqref{eqn:fp-het-treatment-effect-pseudo-outcomes}, which is defined as a one-step gradient descent approximation of $\hat\theta_P$ taken from the origin:
\begin{equation}\label{eqn:one-step-theta-approx}
    \tilde\theta_P \coloneqq \gamma  \frac{1}{n_P}  \sum_{\{i:X_i \in P\}}(W_i - \overline W_P) (Y_i - \overline Y_P).
\end{equation}
Here, $\gamma$ denotes the exact line search step size for the regression of $Y_i - \overline Y_P$ on $W_i - \overline W_P$ over $P$:
\begin{equation}\label{eqn:step-size}
    \gamma \coloneqq \frac{\norm{(W - \overline W_P)^\top (Y - \overline Y_P)}^2_2}{\norm{(W - \overline W_P) (W - \overline W_P)^\top (Y - \overline Y_P) }_2^2},
\end{equation}
where $W = \renewcommand\arraystretch{0.75}\left[\begin{matrix} W_1~\cdots ~W_{n_P} \end{matrix}\right]^\top$ and $Y = \renewcommand\arraystretch{0.75}\left[\begin{matrix} Y_1~\cdots ~Y_{n_P} \end{matrix}\right]^\top$ with the notation $W - \overline W_P$ and $Y - \overline Y_P$ understood as row-wise centering.

The computational cost associated with $\tilde\theta_P$ is comparatively small because many of the products that appear in \eqref{eqn:one-step-theta-approx} and \eqref{eqn:step-size} are already computed as part of $\rho_i^\FPT$ in \eqref{eqn:fp-het-treatment-effect-pseudo-outcomes}. Meanwhile, we show in Appendix~\ref{app:asymp-equiv-approx-fp} that the approximation for the $\FPT$ child estimator:
\begin{equation*}\label{eqn:approx-fpt-child-estimator}
    \bar\theta_{C_j}^\FPT \coloneqq \hat\theta_P + \frac{1}{n_{C_j}}\sum_{\{i:X_i \in C_j\}}\phi_i^\FPT,
\end{equation*}
is consistent for the original $\FPT$ child estimator $\tilde\theta_{C_j}^\FPT$ as $\lVert\tilde\theta_{C_j}^\FPT - \bar\theta_{C_j}^\FPT\rVert = o_P(1)$, meaning that this approximation does not alter the asymptotic behavior of our estimator. These accelerations are particularly compelling when the dimension of $\theta^*(x)$ is large and computational efficiency is critical, as in large-scale A/B testing with multiple concurrent treatment arms or observational studies with numerous treatment levels \citep{kohavi2013online, bakshy2014designing}.

\section{Simulations}\label{sec:simulations}

In this section, we perform empirical evaluations of the computational efficiency and estimation accuracy of the GRF-$\FPT$ method. We let GRF-$\FPTone$ denote the $\FPT$ algorithm using the exact form of the $\FPT$ VCM/HTE pseudo-outcomes \eqref{eqn:fp-het-treatment-effect-pseudo-outcomes} and we let GRF-$\FPTtwo$ denote the accelerated $\FPT$ algorithm based on the form of the $\FPT$ pseudo-outcome approximation \eqref{eqn:accelerated-fpt-pseudo-outcome} in Section~\ref{sec:applications}. We compare both implementations relative to GRF-$\grad$ under VCM and HTE designs. Implementation details and links to the reproducible code are found in Appendix~\ref{app:simulation-details}. 

\noindent\textbf{Settings.} We follow the structural model in \eqref{eqn:applications-model}. The auxiliary variables $X_i$ are drawn from the Gaussian copula with latent covariance matrix $\Sigma$, where $[\Sigma]_{j,k} = (0.3)^{|j-k|}$. Supporting experiments for multicollinearity in $X_i$ can be found in Appendix~\ref{app:supporting-experiments-sims}. The outcomes $Y_i$ follow \eqref{eqn:applications-model} with Gaussian noise $\epsilon_i \sim \mathcal N(0,1)$. For VCM experiments, regressors $W_i \in \bbR^K$ are sampled from $\cN_K({\bm 0},\bbI)$. For HTE experiments, $W_i \in \{0,1\}^K$ follows a multinomial distribution, $W_i\mid X_i=x \sim \text{Multinomial}(1,(\pi_1(x),\ldots,\pi_K(x)))$, where $\pi_k(x)$ is the probability of treatment level $k \in \{1,\ldots, K\}$, characterizing a variety of different location-specific dependence structures through the setting of $\pi_k(\cdot)$.  We set $\nu^*(x) \coloneqq 0$ and vary the target effect functions $\theta^*_k(x)$ and treatment probabilities $\pi_k(x)$ across different settings, fully detailed in Appendix~\ref{app:simulation-details}. Throughout our experiments we use subsampling ratio $s/n = 0.5$. Supporting experiments under different subsample ratios are found in Appendix~\ref{app:supporting-experiments-sims}.

\noindent\textbf{Results.} The relative computational advantage of forests trained under GRF-$\FPT$ is displayed in Figure~\ref{fig:timing-ratio-vcm}, while Figure~\ref{fig:timing-vcm} (in Appendix~\ref{app:simulation-figures}) summarizes the absolute fit times across the three methods. These data show that the $\FPT$ mechanism is able to consistently offer a relative advantage, observing speedups of up to 3.5$\times$ faster than the gradient-based approach at the largest dimension $K = 256$. Figure~\ref{fig:timing-ratio-vcm} also shows increasing gains with increasing $K$ and provides an empirical measurement of the theoretical scaling benefits discussed in Section~\ref{sec:applications}. Moreover, the absolute fit times in Figure~\ref{fig:timing-vcm} (in Appendix~\ref{app:simulation-figures}) illustrate that our method consistently remains faster than GRF-$\grad$, with no clear computational or algorithmic bottleneck as a function of either $n$ or $K$. Supporting experiments exploring the effects of sample sizes up to $n = 500,000$ are presented in Appendix~\ref{app:supporting-experiments-sims}, while Figures~\ref{fig:timing-ratio-vcm-small} and~\ref{fig:timing-vcm-small} (in Appendix~\ref{app:simulation-figures}) show that even when $n$ is small, GRF-$\FPT$ still observes a noticeable gain relative to GRF-$\grad$. Additional timing benchmarks for VCM experiments and all HTE experiments are discussed in Appendix~\ref{app:simulation-figures}.

To assess estimation accuracy, we evaluate the mean squared error (MSE) of $\hat\theta(x)$ across 50 replications of the model and testing on a separate set of $5,000$ observations. Figure~\ref{fig:forest-mse-vcm} in Appendix~\ref{app:simulation-figures} confirms that GRF-$\FPT$ matches the accuracy of GRF-$\grad$, while significantly reducing computation time. Further comparisons for both VCM and HTE settings are provided in Appendix~\ref{app:simulation-figures}.

\begin{figure}[t]
    \centering
    \includegraphics[width=1\linewidth]{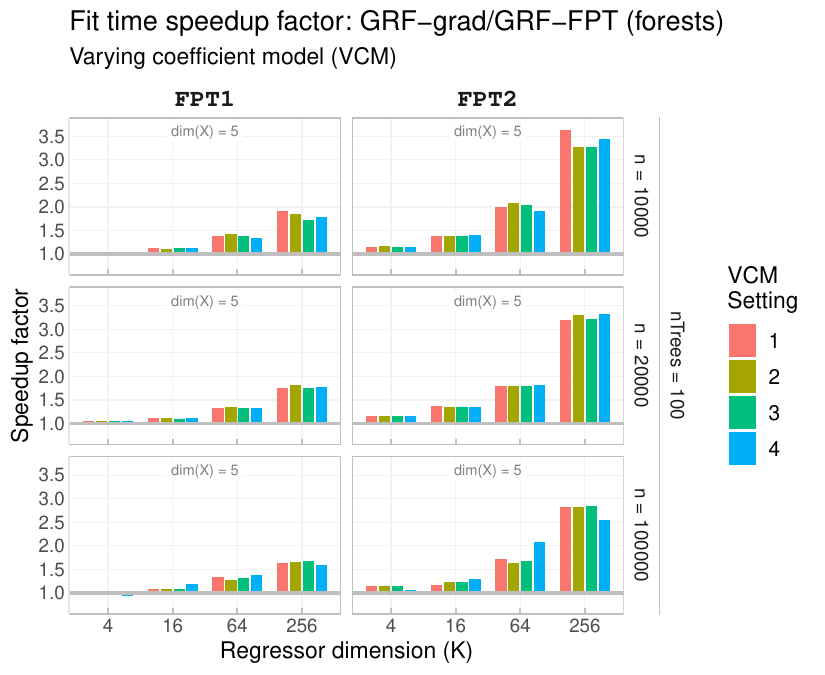}
    \caption{Speedup factor for GRF-$\FPT$ in comparison to GRF-$\grad$ for VCM timing experiments.}\label{fig:timing-ratio-vcm}
\end{figure}

\section{Real Data Application}\label{sec:real-data}

\noindent\textbf{Data.} In this section we apply GRF-$\FPT$ to the analysis of geographically-varying effects $\theta^*(x)$ on housing prices. The data, first appearing in \citet{kelley1997sparse}, contains 20,640 observations of housing prices taken from the 1990 California census. Each observation corresponds to measurements aggregated over a small geographical census block, and contains measurements of 9 variables: median housing value, longitude, latitude, median housing age, total rooms, total bedrooms, population, households, and median income. We employ a VCM design of the form \eqref{eqn:applications-model} where $Y_i$ denotes the housing value, $X_i$ denote the spatial coordinates, and $W_i = (W_{i,1},\ldots, W_{i,6})^\top$ are the remaining six regressors. Details of the model and data transformations used for the California housing analysis is found in Appendix~\ref{app:california-data}. 

\noindent\textbf{Results.} Table~\ref{tbl:california-housing-times} summarizes the computational benefit of GRF-$\FPT$ applied to the California housing data. Figure~\ref{fig:california-housing-fpt2} illustrates the six geographically-varying effect estimates under GRF-$\FPTtwo$, with qualitatively similar results shown in Figure~\ref{fig:california-housing-grad-fpt1} for GRF-$\FPTone$ and GRF-$\grad$ in Appendix~\ref{app:california-data}. Figure~\ref{fig:california-housing-fpt2} shows clearly the geographically-dependent relationship between different housing features and housing prices. In major urban centers such as LA, San Francisco, and Sacramento, housing prices tend to decrease with an increasing number of households, and may reflect overcrowding in densely populated areas. In contrast, rural regions show the opposite trend: prices rise slightly when rural areas have a larger number of housing units. This suggests that, in sparsely populated rural areas, a modest increase in households makes these places more attractive and livable. Median income, however, consistently shows a positive effect on prices across nearly all of California, while population size tends to show a negative effect, highlighting broader state-wide pressures on housing affordability.

\begin{figure}
    \centering
     \includegraphics[width=1\linewidth]{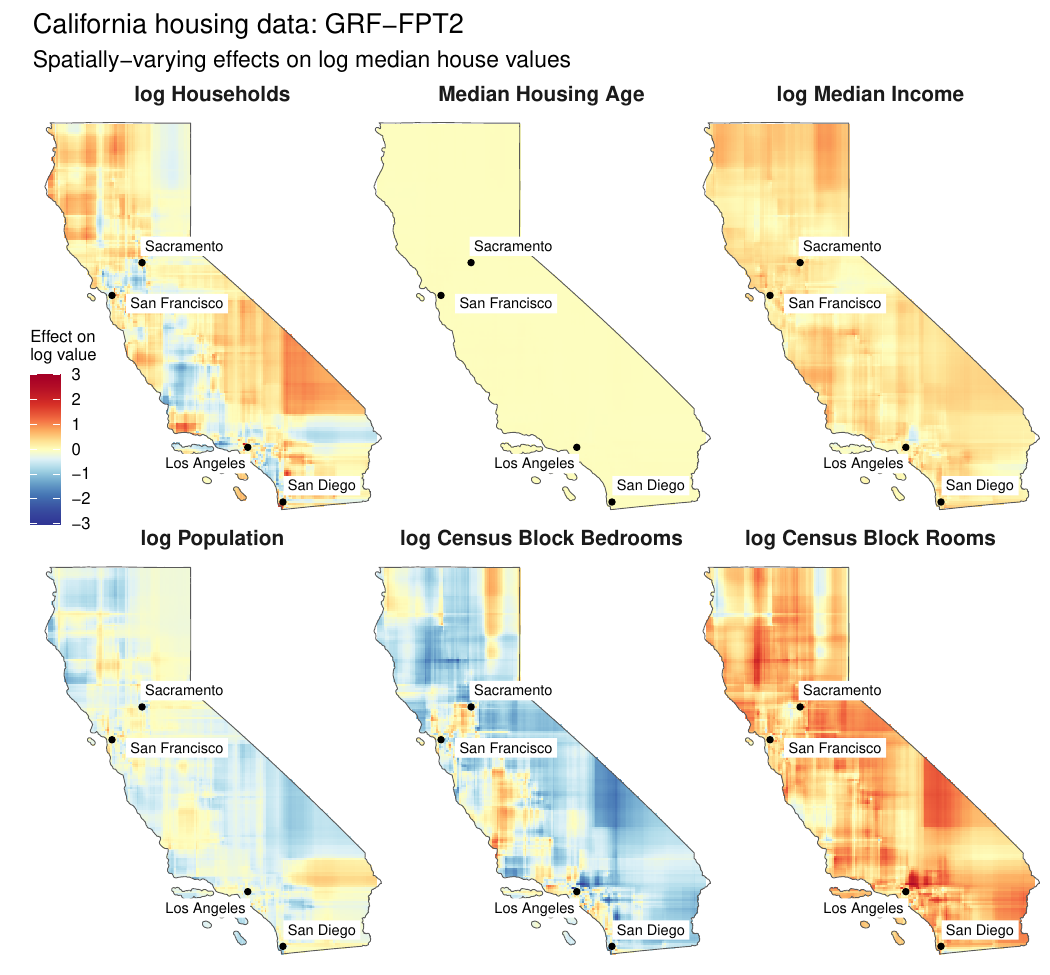}
    \caption{Geographically-varying GRF-$\FPTtwo$ estimates $\hat\theta(x)$.}
    \label{fig:california-housing-fpt2}
\end{figure}

\section{Conclusion}

Our results demonstrate that the $\FPT$ algorithm offers a substantial computational advantage over GRF-$\grad$ with comparable statistical accuracy, and highlights GRF-$\FPT$ as a powerful method for multi-dimensional estimation, particularly when estimates of the target function must be learned from the data rather than observed directly. Future work may explore extensions to larger-scale problems and alternative estimation tasks, as in unsupervised learning and structured prediction. Our findings position GRF-$\FPT$ as a scalable and robust alternative for practitioners seeking efficient localized estimation.

\clearpage
\section*{Impact Statement}

This paper presents work whose goal is to advance the field of machine learning. There are many potential societal consequences of our work, none which we feel must be specifically highlighted here.

\section*{Acknowledgments}
This work was supported by Natural Sciences and Engineering Research Council (NSERC) Discovery Grant (RGPIN-2024-06780) and FRQNT Team Research Project Grant (FRQ-NT 327788).

\bibliography{ref}
\bibliographystyle{icml2025}

\newpage
\appendix
\onecolumn

\section{Technical Preliminaries}\label{app:prelim}

\subsection{Assumptions}\label{app:assumptions}

We follow the key assumptions of \citet{athey2019generalized} made for the theoretical analyses of GRF. The predictor and parameter spaces are both subsets of Euclidean space such that $x \in \cX = [0,1]^p$ and $(\theta, \nu) \in \cB\subset\bbR^K$, where $\cB$ is a compact subset of $\bbR^K$. Under the analyses of \citet{wager2015adaptive}, we suppose that the features of the auxiliary covariates $X_i = (X_{i,1},\ldots,X_{i,p})^\top$ have density $f_X$ that is bounded away from 0 and $\infty$, i.e. $c \leq f_X(x) \leq C < \infty$, for some constants $c > 0$ and $C < \infty$. GRF does not require that the score function $\psi$ is continuous in $(\theta,\nu)$, as is the case for quantile estimation, one does require that the expected score/moment function
\begin{equation}\label{eqn:m-function}
    M_{\theta,\nu}(x) \coloneqq \bbE_{O|X}\left[\psi_{\theta,\nu}(O) \mid X =x\right],
\end{equation}
is smoothly varying in its parameters $(\theta,\nu)$.

\begin{enumerate}[
    label={\textsc{Assumption} \arabic{enumi}.},
    wide,
    labelindent=0pt,
    ref={\arabic{enumi}}
]\label{asm:all-assumptions}

    \item\label{asm:lipschitz-M} For fixed $(\theta,\nu)$, the $M$-function \eqref{eqn:m-function} is Lipschitz continuous in $x$.
    
    \item\label{asm:differentiable-M} For fixed $x$, the $M$-function is twice-differentiable in $(\theta,\nu)$ with uniformly bounded second derivative, 
    \begin{equation*}
        \left\VERT \nabla^2_{(\theta,\nu)} M_{\theta,\nu}(x) \right\VERT < \infty,
    \end{equation*}
    where $\left\VERT\cdot\right\VERT$ denotes the appropriate tensor norm for the second derivative of $M_{\theta,\nu}$ taken with respect to $(\theta,\nu)$. Let $V(x) \coloneqq \nabla_{(\theta,\nu)} M_{\theta,\nu}(x)\big|_{\theta=\theta^*(x),\nu=\nu^*(x)}$ denote the population Jacobian at the true $(\theta^*(x),\nu^*(x))$, and assume that $V(x)$ is invertible for all $x \in \cX$. We write $V(x)$ in block form as
    \begin{equation}\label{eqn:V-block-form}
        V(x) = \begin{bmatrix}
            V_{\theta\theta}(x) & V_{\theta\nu}(x) \\
            V_{\nu\theta}(x) & V_{\nu\nu}(x)
        \end{bmatrix}.
    \end{equation}
    
    \item\label{asm:continuous-psi} The score functions $\psi_{\theta,\nu}(O_i)$ have a continuous covariance structure in the following sense: Let $\gamma(\cdot,\cdot)$ denote the worst-case variogram:
    \begin{equation*}
        \gamma\left( \begin{bmatrix}
            \theta_1\\ \nu_1
        \end{bmatrix},\begin{bmatrix}
            \theta_2 \\ \nu_2
        \end{bmatrix} \right) \coloneqq \sup_{x\in\cX} \left\{ \norm{ \var_{O|X}\left( \psi_{\theta_1,\nu_1}(O_i) - \psi_{\theta_2,\nu_2}(O_i) \;\middle|\; X_i = x\right)}_F \right\},
    \end{equation*}
    then, for some $L > 0$,
    \begin{equation*}
        \gamma\left( \begin{bmatrix}
            \theta_1\\ \nu_1
        \end{bmatrix},\begin{bmatrix}
            \theta_2 \\ \nu_2
        \end{bmatrix} \right) \leq L \left\lVert \begin{bmatrix}
            \theta_1\\ \nu_1
        \end{bmatrix} - \begin{bmatrix}
            \theta_2 \\ \nu_2
        \end{bmatrix}  \right\rVert_2, \quad \text{for all } (\theta_1,\nu_1),~(\theta_2,\nu_2).
    \end{equation*}
    
    \item\label{asm:lipschitz-bounded-psi} The score function $\psi_{\theta,\nu}(O_i)$ can be written as
    \begin{equation*}
        \psi_{\theta,\nu}(O_i) = \lambda(\theta,\nu;O_i) + \zeta_{\theta,\nu}(g(O_i)),
    \end{equation*}
    where $\lambda$ is Lipschitz-continuous in $(\theta,\nu)$, $g:\{O_i\}\to\bbR$ a univariate summary of the observables $O_i$, and $\zeta_\theta:\bbR\to\bbR$ any family of monotone and bounded functions.
    
    \item\label{asm:weights} For any weights $\alpha_i$ with $\sum \alpha_i = 1$, the minimizer $(\hat\theta,\hat\nu)$ of the weighted empirical estimation problem \eqref{eqn:emp-weighted-est-eqn} satisfies:
    \begin{equation*}
        \left\lVert \sum^n_{i =1}\alpha_i \psi_{\hat\theta,\hat\nu}(O_i) \right\rVert_2 \leq C \max_{1\leq i\leq n}\{\alpha_i\}, \quad \text{for $C \geq 0$.}
    \end{equation*}
    
    \item\label{asm:convexity} The score function $\psi_{\theta,\nu}(O_i)$ is a negative subgradient of a convex function, and the moment function $M_{\theta,\nu}(X_i)$ is the negative gradient of a strongly convex function.
\end{enumerate}

\subsection{Forest specifications}

The consistency and asymptotic normality results, Theorems~\ref{thm:consistency} and~\ref{thm:normality}, require that the forest trained following Algorithm~\ref{alg:grffpt-pseudocode-stage1-forest} consists of trees that satisfy a certain set of specifications. These forest specifications are precisely those imposed by \citet{athey2019generalized} for forests of gradient-based trees, and collectively, these specifications describe fairly mild conditions on the tree splitting mechanism, as well as specific requirements for the sampling procedure. 

\begin{enumerate}[
    label={\textsc{Specification} \arabic{enumi}.},
    wide,
    labelindent=0pt,
    ref={\arabic{enumi}}
]\label{spec:all-specifications}

    \item\label{spec:symmetric} ({\em Symmetric}) Tree estimates are invariant to permutations of the training indices. In other words, the output of a tree does not depend on the order in which the training samples are indexed.

    \item\label{spec:balanced} ({\em Balanced/$\omega$-regular}) The proportion of parent observations assigned into either child is bound below by some $\omega > 0$, i.e. $n_{C_j} \geq \omega n_P$. 

    \item\label{spec:randomized} ({\em Randomized/random-split}) The probability of splitting along any feature/dimension of the input space is bound below by some $\pi > 0$.

    \item\label{spec:subsampling} ({\em Subsampling}) Trees are trained on subsample of size $s$, drawn without replacement from $n$ training samples, where $s/n\to 0$ as $s\to\infty$.

    \item\label{spec:honest} ({\em Honesty}) Trees are trained using the sample splitting procedure described in Appendix~\ref{app:honest-indep}.
\end{enumerate}

\subsection{Regularity conditions}

\begin{enumerate}[
    label={\textsc{Regularity condition} \arabic{enumi}.},
    wide,
    labelindent=0pt,
    ref={\arabic{enumi}}
]\label{regcond:all-regularity-conditions}

    \item\label{regcond:variance}  Let $V(x)$ be as defined in Assumption~\ref{asm:differentiable-M} and let $\rho_i^*(x)$ denote the influence function of the $i$-th observation with respect to the target $\theta^*(x)$:
    \begin{equation*}
        \rho_i^*(x) := -\xi^\top V(x)^{-1}\psi_{\theta^*(x),\nu^*(x)}(O_i).
    \end{equation*}
    Then,
    \begin{equation*}
        \var(\rho_i^*(x)\mid X_i=x) > 0, \qquad \text{for all $x \in \cX$.}
    \end{equation*}

    \item\label{regcond:sampling-size} Trees are grown on subsamples of size $s$ scaling as $s = n^\beta$, for some subsample scaling exponent $\beta$ bound according to $\beta_\text{min} < \beta < 1$, such that
    \begin{equation*}
        \beta_\textnormal{min} := 1 - \left(1 + \frac{1}{\pi} \cdot \frac{\log \left(\omega^{-1} \right)}{\log \left((1 - \omega)^{-1} \right)} \right)^{-1} < \beta < 1,
    \end{equation*}
    where $0 < \pi,\omega < 1$ are constants defined in forest Specifications~\ref{spec:balanced} and~\ref{spec:randomized}.
\end{enumerate}

\subsection{Neyman orthogonality}\label{app:neyman-example-vcm-hte}

To identify the underlying local parameters $(\theta^*(x),\nu^*(x)) \in \bbR^K$ one must have a score $\psi_{\theta,\nu}(O)$ with at least $K = K_\theta + K_\nu$ components, where here we use $K_\theta$ and $K_\nu$ to denote the dimensions of the component subvectors $\theta^*(x) \in \bbR^{K_\theta}$ and $\nu^*(x) \in \bbR^{K_\nu}$. Conceptually, a score $\psi_{\theta,\nu}(O)$ can be partitioned into the components that identify the $\theta$-coordinates, denoted by $\psi_1$, and those that identify the $\nu$-coordinates, denoted by $\psi_2$, and thus the moment functions $M_{\theta,\nu}(x)$ in \eqref{eqn:m-function} can also be partitioned the same way:
\begin{equation*}
    \psi_{\theta,\nu}(O) = \begin{bmatrix}
        \psi_1(\theta,\nu;O)  \\
        \psi_2(\theta,\nu;O)
    \end{bmatrix}, \qquad M_{\theta,\nu}(x) = \begin{bmatrix}
        M_1(\theta,\nu;x) \\
        M_2(\theta,\nu;x)
    \end{bmatrix} = \begin{bmatrix}
        \mathbb E[\psi_1(\theta,\nu;O)\mid X = x] \\
        \mathbb E[\psi_2(\theta,\nu;O)\mid X = x] \\
    \end{bmatrix}.
\end{equation*}
The corresponding Jacobian matrix of $M_{\theta,\nu}(x)$ taken with respect to $(\theta,\nu)$ and evaluated at the truth $(\theta^*(x),\nu^*(x))$ is
\begin{equation*}
    V(x) = \nabla_{(\theta,\nu)} \left.M(\theta,\nu;x)\right|_{\theta=\theta^*(x),\nu=\nu^*(x)} = \begin{bmatrix}
        V_{\theta\theta}(x) & V_{\theta\nu}(x) \\
        V_{\nu\theta}(x) & V_{\nu\nu}(x)
    \end{bmatrix},
\end{equation*}
where here the subscripts in the block expressions of $V(x)$ indicate the coordinates with which the gradient is taken, and in all cases are evaluated at the truth $(\theta^*(x),\nu^*(x))$:
\begin{align*}
    V_{\theta\theta}(x) &= \nabla_{\theta} \left.M_1(\theta,\nu;x)\right|_{\theta=\theta^*(x),\nu=\nu^*(x)}, \\
    V_{\theta\nu}(x) &= \nabla_{\nu} \left.M_1(\theta,\nu;x)\right|_{\theta=\theta^*(x),\nu=\nu^*(x)}, \\
    V_{\nu\theta}(x) &= \nabla_{\theta} \left.M_2(\theta,\nu;x)\right|_{\theta=\theta^*(x),\nu=\nu^*(x)}, \\
    V_{\nu\nu}(x) &= \nabla_{\nu} \left.M_2(\theta,\nu;x)\right|_{\theta=\theta^*(x),\nu=\nu^*(x)}.
\end{align*}
In this context, the assumption of Neyman orthogonal moment conditions is more completely labeled as Neyman orthogonality for the estimation of $\theta^*(x)$ with respect to the nuisance $\nu^*(x)$, and can be summarized as an assumption that the moment conditions for $\theta^*(x)$ are insensitive to first-order changes in $\nu$ around the truth $\nu^*(x)$ whenever $\theta = \theta^*(x)$. For GRF, this means that one assumes \eqref{eqn:est-eqn} satisfies $M_1(\theta^*(x),\nu^*(x);x) = {\bf 0}$, and in other words, the partial derivatives of the moment functions for $\theta^*(x)$ with respect to $\nu$ are zero at $(\theta^*(x),\nu^*(x))$:
\begin{equation*}
    V_{\theta\nu}(x) = {\bf 0}.
\end{equation*}

\subsection{Example: Neyman orthogonality for VCM and HTE}\label{app:example-neyman-orthog}

Consider the VCM/HTE model with data $(Y_i, W_i, X_i)$ related according to
\begin{equation*}
    \mathbb E[Y_i\mid X_i = x] = \nu^*(x) + W_i^\top \theta^*(x),
\end{equation*}
such that, as discussed in Section~\ref{sec:limitations}, the score function $\psi_{\theta,\nu}$ that identifies the underlying $(\theta^*(x),\nu^*(x))$ is
\begin{equation*}
    \psi_{\theta,\nu}(Y_i, W_i) \coloneqq \begin{bmatrix}
        (Y_i - W_i^\top \theta - \nu) W_i\\
        Y_i - W_i^\top \theta - \nu 
    \end{bmatrix},
\end{equation*}
and the corresponding local Jacobian $V(x)$ has block form
\begin{equation*}
    V(x) = -\mathbb E\left[\begin{bmatrix}
        W_iW_i^\top & W_i^\top \\
        W_i & 1
    \end{bmatrix}\;\middle|\;X_i = x\right] = -\begin{bmatrix}
        \mathbb E[W_iW_i^\top \mid X_i = x] &  \mathbb E[W_i^\top \mid X_i = x] \\
        \mathbb E[W_i \mid X_i = x] & 1
    \end{bmatrix}.
\end{equation*}
Therefore, for Neyman orthogonality to hold one requires that $\mathbb E[W_i \mid X_i = x] = {\bf 0}$.

\section{Derivations and Proofs}

\subsection{Proofs for Section~\ref{sec:pseudo-outcomes}}\label{app:proofs-pseudo-outcomes}

\subsubsection{Multivariate CART criteria}\label{app:multivariate-cart-criteria}

Let $\rho_i \in \bbR^K$ be vector-valued responses associated with covariates $X_i \in P$. A standard CART split $(C_1,C_2)$ of $P$ minimizes the conventional least-squares criterion:
\begin{equation}\label{eqn:cart-criterion-appendix}
    \sum_{\{i\,:\,X_i \in C_1\}} \norm{\rho_i - \bar \rho_{C_1}}^2 + \sum_{\{i\,:\,X_i \in C_2\}} \norm{\rho_i - \bar \rho_{C_2}}^2,
\end{equation}
where $\bar\rho_{C_j} \coloneqq n_{C_j}^{-1}\sum_{\{i:X_i \in C_j\}} \rho_i$ is the local prediction over child node $C_j$. We verify that a split $(C_1,C_2)$ minimizes \eqref{eqn:cart-criterion-appendix} if and only if it maximizes
\begin{equation}
\label{eqn:cart-criterion_eq}
n_{C_1}\norm{\bar\rho_{C_1}}^2 + n_{C_2}\norm{\bar\rho_{C_2}}^2.
\end{equation}

\begin{proof} Each sum in \eqref{eqn:cart-criterion-appendix} can be expanded as
\begin{align*}
    \sum_{\{i\,:\,X_i \in C_j\}} \norm{\rho_i - \bar \rho_{C_j}}^2 &= \sum_{\{i:X_i\in P\}} \norm{\rho_i - \bar\rho_{C_j}}^2\cdot\mathds 1(X_i \in C_j), \\
    &= \sum_{\{i:X_i\in P\}} \left(\norm{\rho_i}^2 - 2\rho_i^\top \bar\rho_{C_j} + \norm{\bar\rho_{C_j}}^2\right) \cdot\mathds 1(X_i \in C_j), \\
    &= \sum_{\{i:X_i\in P\}} \norm{\rho_i}^2 \cdot \mathds 1(X_i \in C_j) - n_{C_j} \norm{\bar\rho_{C_j}}^2.
\end{align*}
Therefore, the least-squares criterion CART \eqref{eqn:cart-criterion-appendix} is equivalently written as
\begin{align*}
    \sum_{j=1,2} \sum_{\{i\,:\,X_i \in C_j\}} \norm{\rho_i - \bar \rho_{C_j}}^2 &= \sum_{j=1,2} \left( \sum_{\{i:X_i\in P\}} \norm{\rho_i}^2 \cdot \mathds 1(X_i \in C_j) - n_{C_j} \norm{\bar\rho_{C_j}}^2 \right), \\
    &= \sum_{j=1,2}  \left(\sum_{\{i:X_i\in P\}} \norm{\rho_i}^2 \cdot \mathds 1(X_i \in C_j) \right) - \left(n_{C_1}  \norm{\bar\rho_{C_1}}^2  + n_{C_2}  \norm{\bar\rho_{C_2}}^2 \right), \\
    &= \sum_{\{i:X_i\in P\}} \norm{\rho_i}^2 - \left( n_{C_1}\norm{\bar\rho_{C_1}}^2 + n_{C_2} \norm{\bar\rho_{C_2}}^2\right).
\end{align*}
The first term does not depend on the choice of split, and therefore the split that minimizes \eqref{eqn:cart-criterion-appendix} is equivalent to the split that maximizes \eqref{eqn:cart-criterion_eq}.
\end{proof}

\subsubsection{Splits via CART on pseudo-outcomes}\label{app:splits-via-CART-pseudo-outcomes}

The following result is a generalization to the claim made in Section~\ref{sec:pseudo-outcomes} that a CART split on pseudo-outcomes $\rho_i^\FPT$ will produce a split that maximizes the $\widetilde\Delta^\FPT$-criterion, and is sufficiently general to cover gradient-based pseudo-outcomes $\rho_i^\grad$ and the corresponding $\widetilde\Delta^\grad$-criterion.

\begin{lemma}\label{lem:affine-equivalence-pseudo-outcome} Suppose we can write
\begin{equation}\label{eqn:affine-pseudo-outcome-theta-estimator}
    \tilde\theta_{C_j} =  a + \frac{1}{n_{C_j}} \sum_{\{i:X_i\in C_j\}} \rho_i, \qquad \rho_i = -B \psi_{\hat\theta_P,\hat\nu_P}(O_i),
\end{equation}
where $a$ and $B$ denote appropriately sized vectors and matrices whose values do not depend on the candidate child node $C_j$. Under Assumptions~\ref{asm:all-assumptions}, the split $(C_1,C_2)$ that maximizes
\begin{equation*}
\widetilde\Delta(C_1,C_2) = \frac{n_{C_1}n_{C_2}}{n_P^2} \norm{\tilde\theta_{C_1} - \tilde\theta_{C_2}}^2,
\end{equation*}
is exactly the split chosen by CART for vector-valued responses $\rho_i$ fit over covariates $X_i \in P$.
\end{lemma}

\begin{proof}[{\bf Proof of Lemma~\ref{lem:affine-equivalence-pseudo-outcome}}] The scores $\psi_{\theta,\nu}(O_i)$ satisfy subgradient conditions by Assumption~\ref{asm:convexity}, and therefore the parent solutions $(\hat\theta_P,\hat\nu_P)$ satisfy the first-order conditions
\begin{equation*}
    \sum_{\{i:X_i \in P\}} \psi_{\hat\theta_P,\hat\nu_P}(O_i) = {\bf 0}.
\end{equation*}
Hence,
\begin{align*}
    {\bf 0} =  \sum_{\{i:X_i\in P\}} \psi_{\hat\theta_P,\hat\nu_P}(O_i) &= \sum_{\{i:X_i \in C_1\}} \psi_{\hat\theta_P,\hat\nu_P}(O_i) + \sum_{\{i :X_i\in C_2\}} \psi_{\hat\theta_P,\hat\nu_P}(O_i), \\
    &= -B\left( \sum_{\{i:X_i \in C_1\}}  \psi_{\hat\theta_P,\hat\nu_P}(O_i) + \sum_{\{i:X_i \in C_2\}} \psi_{\hat\theta_P,\hat\nu_P}(O_i)\right), \\
    &= \sum_{\{i:X_i \in C_1\}} \rho_i + \sum_{\{i:X_i \in C_2\}} \rho_i.
\end{align*}
Each sum in the previous expression is equivalently written as $\sum \rho_i = n_{C_j} (\tilde\theta_{C_j} - a)$. Hence,
\begin{align*}
    {\bf 0} &= \sum_{\{i:X_i \in C_1\}}\rho_i + \sum_{\{i:X_i \in C_2\}}\rho_i, \\
    &= n_{C_1} (\tilde\theta_{C_1} - a) + n_{C_2} (\tilde\theta_{C_2} - a), \\
    \iff a &= \frac{n_{C_1}}{n_P} \tilde\theta_{C_1} + \frac{n_{C_2}}{n_P}\tilde\theta_{C_2}.
\end{align*}
Writing $\bar\rho_{C_j} \coloneqq \frac{1}{n_{C_j}}\sum_{\{i:X_i \in C_j\}} \rho_i$, one has:
\begin{align*}
    \bar\rho_{C_1} &=  \tilde\theta_{C_1} - a, \\
    &= \tilde\theta_{C_1} -  \frac{n_{C_1}}{n_P} \tilde\theta_{C_1} - \frac{n_{C_2}}{n_P}\tilde\theta_{C_2}, \\
    &= \frac{n_{C_2}}{n_P} \left(\tilde\theta_{C_1} - \tilde\theta_{C_2}\right),
\end{align*}
and
\begin{equation*}
    \frac{n_{C_1}}{n_P}\norm{\bar\rho_{C_1}}^2 = \frac{n_{C_1} n_{C_2}^2}{n_P^3} \norm{\tilde\theta_{C_1} - \tilde\theta_{C_2}}^2.
\end{equation*}
Applying analogous arguments with respect to $C_2$, one has the symmetric result:
\begin{equation*}
    \frac{n_{C_2}}{n_P}\norm{\bar\rho_{C_2}}^2 = \frac{n_{C_2} n_{C_1}^2}{n_P^3} \norm{\tilde\theta_{C_1} - \tilde\theta_{C_2}}^2.
\end{equation*}
Therefore,
\begin{align*}
    \frac{1}{n_P} \left( n_{C_1} \norm{\bar\rho_{C_1}}^2 + n_{C_2}\norm{\bar\rho_{C_2}}^2 \right) &= \frac{n_{C_1} n_{C_2}^2}{n_P^3} \norm{\tilde\theta_{C_1} - \tilde\theta_{C_2}}^2 + \frac{n_{C_2} n_{C_1}^2}{n_P^3} \norm{\tilde\theta_{C_1} - \tilde\theta_{C_2}}^2, \\
    &= \frac{n_{C_1}n_{C_1}}{n_P^2}\norm{\tilde\theta_{C_1} - \tilde\theta_{C_2}}^2, \\
    &= \widetilde\Delta(C_1,C_2).
\end{align*}
Based on the arguments in Appendix~\ref{app:multivariate-cart-criteria}, a split $(C_1,C_2)$ maximizes $n_{C_1}\norm{\bar\rho_{C_1}}^2 + n_{C_2}\norm{\bar\rho_{C_2}}^2$ if and only if it is a CART split performed on the $\rho_i$ over $P$. That is, $\widetilde\Delta(C_1,C_2)$ is precisely maximized by a single CART split on $\rho_i = -B\psi_{\hat\theta_P,\hat\nu_P}(O_i)$ fit over covariates $X_i \in P$, as desired.
\end{proof}

\subsubsection{Scale invariance of CART splits}\label{app:scale-invariance-eta}

 \begin{lemma}[Argmax equivalence of $\FPT$ criteria]\label{lem:scale-invariance-cart} The optimal split identified by CART on pseudo-outcomes $\rho_i^\FPT$ of the form \eqref{eqn:fp-pseudo-outcome-eta} does not depend on the scale factor $\eta$, for any $\eta \neq 0$.
\end{lemma}

\begin{proof}[{\bf Proof of Lemma~\ref{lem:scale-invariance-cart}}] Denote by $\rho_i^{(\eta)}$ $\FPT$ pseudo-outcomes based on an arbitrary scale factor $\eta \neq 0$ of the form \eqref{eqn:fp-pseudo-outcome-eta}:
\begin{equation}\label{eqn:fp-pseudo-outcomes-eta-notation}
    \rho_i^{(\eta)} \coloneqq -\eta \xi^\top \psi_{\hat\theta_P,\hat\nu_P}(O_i),
\end{equation}
and let $\overline\psi_{C_j}$ denote the child-leaf average score evaluated at the parent solution $(\hat\theta_P,\hat\nu_P)$:
\begin{equation*}
    \overline\psi_{C_j} \coloneqq \frac{1}{n_{C_j}}\sum_{\{i:X_i \in C_j\}}\psi_{\hat\theta_P,\hat\nu_P}(O_i),
\end{equation*}
such that the corresponding child-leaf pseudo-outcome averages $\bar\rho_{C_j}^{~(\eta)} \coloneqq \frac{1}{n_{C_j}}\sum_{\{i:X_i \in C_j\}}\rho_i^{(\eta)}$ are equivalently written as
\begin{equation*}
    \bar\rho_{C_j}^{~(\eta)} = -\eta\xi^\top~\overline\psi_{C_j}
\end{equation*}
Let $\widetilde\Delta_\eta^\FPT(C_1,C_2)$ denote the $\FPT$ criterion of the form \eqref{eqn:fp-criterion-pseudo-outcome} based on pseudo-outcomes \eqref{eqn:fp-pseudo-outcomes-eta-notation}:
\begin{equation*}
    \widetilde\Delta_\eta^\FPT(C_1,C_2) = \frac{n_{C_1}n_{C_2}}{n_P^2} \norm{ \bar\rho_{C_1}^{~(\eta)} -  \bar\rho_{C_j}^{~(\eta)}}^2 = \frac{n_{C_1}n_{C_2}}{n_P^2} \norm{\eta\xi^\top (\overline\psi_{C_1} - \overline\psi_{C_2})}^2.
\end{equation*}
One has:
\begin{equation*}
    \norm{\eta\xi^\top (\overline\psi_{C_1} - \overline\psi_{C_2})}^2 = \eta^2 \norm{\xi^\top (\overline\psi_{C_1} - \overline\psi_{C_2})}^2,
\end{equation*}
and hence the $\widetilde\Delta_\eta^\FPT$-criteria obey the scaling relation:
\begin{equation}\label{eqn:scaling-relationship-criteria-TEST}
    \widetilde\Delta_\eta^\FPT(C_1,C_2) = \eta^2 \cdot \widetilde\Delta_1^\FPT(C_1,C_2),
\end{equation}
where $\widetilde\Delta_1^\FPT$ denotes the $\FPT$ criterion induced by pseudo-outcomes $\rho_i^{(1)}$ based on unit scale factor $\eta = 1$. The relation \eqref{eqn:scaling-relationship-criteria-TEST} implies that  any nonzero split-independent rescaling $\rho_i^{(\eta)} = \eta \rho_i^{(1)}$ will induce a splitting criterion $\widetilde\Delta_\eta^\FPT(C_1,C_2)$ with the same maximizer as $\widetilde\Delta_1^\FPT(C_1,C_2)$:
\begin{equation*}
    \argmax_{(C_1,C_2)}~\left\{\widetilde\Delta_\eta^\FPT(C_1,C_2)\right\} = \argmax_{(C_1,C_2)}~\left\{\eta^2 \cdot \widetilde\Delta_1^\FPT(C_1,C_2) \right\} = \argmax_{(C_1,C_2)}~\left\{\widetilde\Delta_1^\FPT(C_1,C_2)\right\}.
\end{equation*}
Intuitively, a CART split is chosen by ranking the criterion values among the candidate splits and selecting the maximizing split $(C_1,C_2)$. Therefore, the $\FPT$ splitting mechanism is unaffected by the scale factor $\eta$ used to specify fixed-point pseudo-outcomes \eqref{eqn:fp-pseudo-outcome-eta}. The absolute scale of the $\widetilde\Delta^\FPT$-criterion does not matter when searching for the optimal split, and only the criterion rankings across the candidate splits determine the final partition.
\end{proof}

\subsection{Proofs for Section~\ref{sec:theoretical-analysis}}\label{app:proofs}

\paragraph{Notation and definitions.} 

\begin{itemize}
    \item Let $o_P(a,b,c) \coloneqq o_P(\max\{a,b,c\})$, with an analogous abbreviation for $O_P(\cdot)$.

    \item For a fixed parent node $P$, denote by $x_P$ the center of mass of the $X_i \in P$, and let $r \coloneqq \sup_{\{i:X_i\in P\}}\norm{X_i - x_P}$ denote the radius of the parent $P$. Throughout, we consider an asymptotic regime where $n_{C_j}\to\infty$ and $r\to 0$, corresponding to leaves over $\cX$ of vanishing radius. Further, $r$ and $n_{C_j}$ are related under the conditions of GRF Proposition 1, namely, $r^{-2} \ll n_{C_j}$ and hence $n_{C_j}r^2 \to \infty$ and $1/\sqrt{n_{C_j}} = o(r)$.

    \item Let $\theta_{C_j}^*$ denote the true parameter expectation over the child node:
    \begin{equation}\label{eqn:theta-star-Cj}
        \theta_{C_j}^* \coloneqq \mathbb E[\theta^*(X) \mid X \in C_j], \qquad j = 1,2,
    \end{equation}
    and let $\tilde\theta^*_{C_j}(x_P)$ denote an oracle version of the gradient-based leaf statistic:
    \begin{equation*}
        \tilde\theta^*_{C_j}(x_P) \coloneqq \theta^*(x_P) - \frac{1}{n_{C_j}} \sum_{\{i:X_i \in C_j\}} \xi^\top V(x_P)^{-1} \psi_{\theta^*(x_P),\nu^*(x_P)}(O_i),
    \end{equation*}
    where $V(x)$ is the underlying local Jacobian in Assumption~\ref{asm:differentiable-M}. Equivalently, in terms of the oracle pseudo-outcome/influence function $\rho_i^*(\cdot)$ defined in Regularity Condition~\ref{regcond:variance},
    \begin{equation*}
        \tilde\theta^*_{C_j}(x_P) \coloneqq \theta^*(x_P) + \frac{1}{n_{C_j}} \sum_{\{i:X_i \in C_j\}} \rho_i^*(x_P).
    \end{equation*}
\end{itemize}

The following are technical lemmas used for the proof of Proposition~\ref{prop:Delta-V-crit-asymptotic}.

\begin{lemma}\label{lem:Delta-crit-asymptotic} Suppose Assumptions~\ref{asm:all-assumptions} and Specifications~\ref{spec:all-specifications} hold. Then,
\begin{equation*}
    \Delta(C_1,C_2) = \frac{n_{C_1}n_{C_2}}{n_P^2} \norm{\theta_{C_1}^* - \theta_{C_2}^*}^2 + o_P\left(r^2,\frac{1}{n_{C_1}},~\frac{1}{n_{C_2}}\right).
\end{equation*}
\end{lemma}

\begin{proof}[{\bf Proof of Lemma~\ref{lem:Delta-crit-asymptotic}}] Write the difference $\hat\theta_{C_j} - \theta_{C_j}^*$ as
\begin{equation*}
    \hat\theta_{C_j} - \theta_{C_j}^* = \underbrace{\left(\hat\theta_{C_j} - \tilde\theta_{C_j}^*(x_P)\right)}_{T_1} + \underbrace{\left( \tilde\theta_{C_j}^*(x_P) - \bbE[\tilde\theta_{C_j}^*(x_P)\mid X \in C_j] \right)}_{T_2} + \underbrace{\left( \bbE[\tilde\theta_{C_j}^*(x_P)\mid X \in C_j] - \theta_{C_j}^* \right)}_{T_3}.
\end{equation*}
Under standard LLN arguments, the second term satisfies $T_2 = O_P(1/\sqrt{n_{C_j}})$, and in an asymptotic regime with $r^{-2}\ll n_{C_j}$ one has $T_2 = o_P(r)$. Meanwhile, the first and third terms appear in the proofs of Propositions 2 and 1 of \citet{athey2019generalized}, respectively, and satisfy $T_1 = o_P(r,~1/\sqrt{n_{C_j}})$ and $T_3 = \cO(r^2) \implies T_3 = o(r)$. It follows
\begin{equation*}
     \hat\theta_{C_j} - \theta_{C_j}^* = o_P\left(r,~1/\sqrt{n_{C_j}}\right),
\end{equation*}
and in particular
\begin{equation*}
    \hat\theta_{C_1} - \hat\theta_{C_2} = \theta_{C_1}^* - \theta_{C_2}^* + o_P\left(r, ~\frac{1}{\sqrt{n_{C_1}}},~\frac{1}{\sqrt{n_{C_2}}}\right).
\end{equation*}
Write $A = \theta_{C_1}^* - \theta_{C_2}^*$ and let $E$ be any term satisfying $E = o_P(r,~1/\sqrt{n_{C_1}},~1/\sqrt{n_{C_2}})$ such that $\Delta(C_1,C_2)$ is equivalently written $\Delta(C_1,C_2) = (n_{C_1}n_{C_2}/n_P^2) \cdot \norm{A + E}^2$. Consider the difference
\begin{align*}
    \Delta(C_1,C_2) - \frac{n_{C_1}n_{C_2}}{n_P^2} \norm{\theta_{C_1}^* - \theta_{C_2}^*}^2 &= \frac{n_{C_1}n_{C_2}}{n_P^2} \left( \norm{A + E}^2 - \norm{A}^2 \right),\\
    &= \frac{n_{C_1}n_{C_2}}{n_P^2} \left(2\langle A, E\rangle + \norm{E}^2\right).
\end{align*}
Under Specification~\ref{spec:balanced} there exists a fixed proportion $\omega > 0$ such that $n_{C_1},n_{C_2} \geq \omega n_P$, and hence $n_{C_1}n_{C_2}/n_P^2 \geq \omega(1 - \omega)$ and also $n_{C_1}n_{C_2}/n_P^2 \leq 1/4$ for all $n_{C_1} + n_{C_2} = n_P$. Therefore $n_{C_1}n_{C_2}/n_P^2 = \cO(1)$. Meanwhile, $\norm{E}^2 = o_P(r^2,~1/n_{C_1},~1/n_{C_2})$ is true by definition of $E$, and under our assumptions one may follow the arguments of \citet{athey2019generalized} Proposition 1 to see that $A = \theta_{C_1}^* - \theta_{C_2}^* = \cO(r)$. Thus,
\begin{equation*}
    \langle A, E\rangle =\cO(r) \cdot o_P\left(r, ~\frac{1}{\sqrt{n_{C_1}}},~\frac{1}{\sqrt{n_{C_2}}}\right) = o_P\left(r^2, ~\frac{r}{\sqrt{n_{C_1}}},~\frac{r}{\sqrt{n_{C_2}}}\right),
\end{equation*}
and therefore
\begin{equation*}
    \Delta(C_1,C_2) - \frac{n_{C_1}n_{C_2}}{n_P^2} \norm{\theta_{C_1}^* - \theta_{C_2}^*}^2 = o_P\left(r^2, ~\frac{1}{{n_{C_1}}},~\frac{1}{{n_{C_2}}}\right),
\end{equation*}
as desired.
\end{proof}

\begin{lemma}\label{lem:Delta-FPT-crit-asymptotic} Suppose the conditions of Lemma~\ref{lem:Delta-crit-asymptotic} hold, and assume moreover Neyman orthogonal moment conditions such that the underlying Jacobian $V(x)$ defined in Assumption~\ref{asm:differentiable-M} with block form \eqref{eqn:V-block-form}. Then, 
\begin{equation*}
    \widetilde\Delta_\eta^\FPT(C_1,C_2) = \frac{n_{C_1}n_{C_2}}{n_P^2} \eta^2 \norm{ V_{\theta\theta}(x_P)(\theta_{C_1}^* - \theta_{C_2}^*)}^2 + o_P\left(r^2, ~\frac{1}{{n_{C_1}}},~\frac{1}{{n_{C_2}}}\right),
\end{equation*}
where $\Delta_\eta^\FPT$ defined in Lemma~\ref{lem:scale-invariance-cart} denotes the $\FPT$ criterion with arbitrary scale factor $\eta \neq 0$.
\end{lemma}

\begin{proof}[{\bf Proof of Lemma~\ref{lem:Delta-FPT-crit-asymptotic}}] From the proof of Lemma~\ref{lem:scale-invariance-cart} one finds that $\Delta_\eta^\FPT(C_1,C_2)$ is equivalently written
\begin{equation*}
    \Delta_\eta^\FPT(C_1,C_2) \coloneqq \frac{n_{C_1}n_{C_2}}{n_P^2} \eta^2 \norm{\xi^\top (\overline\psi_{C_1} - \overline\psi_{C_2})}^2, \qquad \overline\psi_{C_j} \coloneqq \frac{1}{n_{C_1}} \sum_{\{i:X_i \in C_j\}}\psi_{\hat\theta_P,\hat\nu_P}(O_i).
\end{equation*}
Under standard LLN arguments the average scores $\overline\psi_{C_j}$ satisfy
\begin{equation}
    \overline\psi_{C_j} = \bbE[\psi_{\hat\theta_P,\hat\nu_P}(O) \mid X \in C_j] + O_P(1/\sqrt{n_{C_j}}).
\end{equation}
One applies iterated expectation to see
\begin{equation*}
    \bbE[\psi_{\hat\theta_P,\hat\nu_P}(O) \mid X \in C_j] = \bbE\left[ \bbE\left[ \psi_{\hat\theta_P,\hat\nu_P}(O) \mid X \right]\mid X \in C_j \right] = \bbE[M_{\hat\theta_P,\hat\nu_P}(X) \mid X \in C_j],
\end{equation*}
and hence
\begin{equation}\label{eqn:average-score-asymp} 
    \overline\psi_{C_j} = \bbE[M_{\hat\theta_P,\hat\nu_P}(X) \mid X \in C_j] + O_P(1/\sqrt{n_{C_j}}).
\end{equation}

\paragraph{Expansion of $\bm{M_{\hat\theta_P,\hat\nu_P}(X)}$.}

Under Assumption~\ref{asm:differentiable-M} one considers the Taylor expansion of $M_{\hat\theta_P,\hat\nu_P}(X)$ about $(\theta,\nu) = (\theta^*(x_P),\nu^*(x_P))$:
\begin{align*}
    M_{\hat\theta_P,\hat\nu_P}(X) &= M_{\theta^*(x_P),\nu^*(x_P)}(X) \\
    &\hphantom{...} + \left[\nabla_{(\theta,\nu)}M_{\theta^*(x_P),\nu^*(x_P)}(X)\right] \begin{bmatrix}
        \hat\theta_P - \theta^*(x_P) \\
        \hat\nu_P - \nu^*(x_P)
    \end{bmatrix} + O_P\left( \norm{ \begin{bmatrix}
        \hat\theta_P - \theta^*(x_P) \\
        \hat\nu_P - \nu^*(x_P)
    \end{bmatrix} }^2 \right).
\end{align*}
The consistency of the parent solutions $(\hat\theta_P,\hat\nu_P)$ for $(\theta^*(x_P),\nu^*(x_P))$ is established by \citet{athey2019generalized}, and in particular $(\hat\theta_P,\hat\nu_P) - (\theta^*(x_P),\nu^*(x_P)) = O_P(r,~1/\sqrt{n_P})$. The asymptotic regime $r^{-2}\ll n_P$ implies $1/\sqrt{n_P} = o(r)$ and therefore the higher order quadratic term is equivalently expressed:
\begin{equation*}
    M_{\hat\theta_P,\hat\nu_P}(X) = M_{\theta^*(x_P),\nu^*(x_P)}(X) + \left[\nabla_{(\theta,\nu)} M_{\theta^*(x_P),\nu^*(x_P)}(X)\right] \begin{bmatrix}
        \hat\theta_P - \theta^*(x_P) \\
        \hat\nu_P - \nu^*(x_P)
    \end{bmatrix} + O_P(r^2),
\end{equation*}
and therefore
\begin{align*}
    \mathbb E\left[M_{\hat\theta_P,\hat\nu_P}(X)\mid X \in C_j \right] &= \mathbb E\left[ M_{\theta^*(x_P),\nu^*(x_P)}(X) \mid X \in C_j \right] \\
    &\hphantom{.....} + \mathbb E\left[  \nabla_{(\theta,\nu)} M_{\theta^*(x_P),\nu^*(x_P)}(X) 
    \,\middle|\, X \in C_j \right] \begin{bmatrix}
        \hat\theta_P - \theta^*(x_P) \\
        \hat\nu_P - \nu^*(x_P)
    \end{bmatrix} + O_P(r^2).
\end{align*}
One has $\nabla_{(\theta,\nu)} M_{\theta^*(x_P),\nu^*(x_P)}(X) = V(x_P) + O_P(r)$ because $M_{\theta,\nu}(x)$ is Lipschitz in $x$, and the expansion in the previous display becomes:
\begin{equation}\label{eqn:M-theta-hat-P-taylor-expansion}
    \mathbb E\left[M_{\hat\theta_P,\hat\nu_P}(X)\mid X \in C_j \right] = \mathbb E\left[ M_{\theta^*(x_P),\nu^*(x_P)}(X) \mid X \in C_j \right] + V(x_P) \begin{bmatrix}
        \hat\theta_P - \theta^*(x_P) \\
        \hat\nu_P - \nu^*(x_P)
    \end{bmatrix} + O_P(r^2).
\end{equation}

\paragraph{Expansion of $\bm{M_{\theta^*(x_P),\nu^*(x_P)}(X)}$.}

Following similar arguments, the term $M_{\theta^*(x_P),\nu^*(x_P)}(X)$ is expanded about $(\theta,\nu) = (\theta^*(X),\nu^*(X))$ as:
\begin{align*}
    M_{\theta^*(x_P),\nu^*(x_P)}(X) &= M_{\theta^*(X),\nu^*(X)}(X) + V(X) \begin{bmatrix}
        \theta^*(x_P) - \theta^*(X) \\
        \nu^*(x_P) - \nu^*(X)
    \end{bmatrix} + O_P(r^2), \\
    &= V(X) \begin{bmatrix}
        \theta^*(x_P) - \theta^*(X) \\
        \nu^*(x_P) - \nu^*(X)
    \end{bmatrix} + O_P(r^2),
\end{align*}
where $M_{\theta^*(X),\nu^*(X)}(X) = {\bf 0}$ holds because $(\theta^*(X),\nu^*(X))$ are defined as satisfying the GRF moment conditions \eqref{eqn:est-eqn} local to $X$. One takes the conditional expectation of the previous display:
\begin{equation*}
    \mathbb E\left[M_{\theta^*(x_P),\nu^*(x_P)}(X)\mid X \in C_j\right] = \mathbb E\left[V(X)  \begin{bmatrix}
        \theta^*(x_P) - \theta^*(X) \\
        \nu^*(x_P) - \nu^*(X)
    \end{bmatrix}\,\middle|\, X \in C_j\right] + O_P(r^2).
\end{equation*}
Whenever $X \in C_j$ one has $\norm{X - x_P} = \cO(r)$, and the same Lipschitz arguments can be applied to see $V(X) = V(x_P) + O_P(r)$ conditional on $X \in C_j$, and the previous display simplifies:
\begin{equation}\label{eqn:M-theta-star-xp-expansion}
    \mathbb E\left[M_{\theta^*(x_P),\nu^*(x_P)}(X)\mid X \in C_j\right] = V(x_P) \begin{bmatrix}
        \theta^*(x_P) - \theta_{C_j}^* \\
        \nu^*(x_P) - \nu_{C_j}^*
    \end{bmatrix} + O_P(r^2),
\end{equation}
where $\theta_{C_j}^*\coloneqq \mathbb E[\theta^*(X)\mid X \in C_j]$ and $\nu_{C_j}^*\coloneqq \mathbb E[\nu^*(X) \mid X \in C_j]$. Substitute \eqref{eqn:M-theta-star-xp-expansion} into the conditional expectation \eqref{eqn:M-theta-hat-P-taylor-expansion}:
\begin{align*}
    \mathbb E\left[M_{\hat\theta_P,\hat\nu_P}(X)\mid X \in C_j \right] &= V(x_P) \begin{bmatrix}
        \theta^*(x_P) - \theta_{C_j}^* \\
        \nu^*(x_P) - \nu_{C_j}^*
    \end{bmatrix} + V(x_P) \begin{bmatrix}
        \hat\theta_P - \theta^*(x_P) \\
        \hat\nu_P - \nu^*(x_P)
    \end{bmatrix} + O_P(r^2),\\
    &= V(x_P) \begin{bmatrix}
        \hat\theta_P - \theta_{C_j}^* \\
        \hat\nu_P - \nu_{C_j}^*
    \end{bmatrix} + O_P(r^2).
\end{align*}
Therefore, the child node score averages $\overline\psi_{C_j}$ in \eqref{eqn:average-score-asymp} satisfy
\begin{equation*}
    \overline\psi_{C_j} = V(x_P) \begin{bmatrix}
        \hat\theta_P - \theta_{C_j}^* \\
        \hat\nu_P - \nu_{C_j}^*
    \end{bmatrix} + O_P\left(r^2,~1/\sqrt{n_{C_j}}\right),
\end{equation*}
and the difference $\overline\psi_{C_1} - \overline\psi_{C_2}$ satisfies
\begin{align*}
    \overline\psi_{C_1} - \overline\psi_{C_2} &= V(x_P) \begin{bmatrix}
        \hat\theta_P - \theta_{C_1}^* \\
        \hat\nu_P - \nu_{C_1}^*
    \end{bmatrix} - V(x_P)\begin{bmatrix}
        \hat\theta_P - \theta_{C_2}^* \\
        \hat\nu_P - \nu_{C_2}^*
    \end{bmatrix}  +  O_P\left(r^2,~\frac{1}{\sqrt{n_{C_1}}},~\frac{1}{\sqrt{n_{C_2}}}\right),\\
    &= -V(x_P) \begin{bmatrix}
        \theta_{C_1}^* - \theta_{C_2}^* \\
        \nu_{C_1}^* - \nu_{C_2}^*
    \end{bmatrix}  + O_P\left(r^2,~\frac{1}{\sqrt{n_{C_1}}},~\frac{1}{\sqrt{n_{C_2}}}\right).
\end{align*}
We assume $\eta$ is a fixed scalar and $\xi^\top$ a fixed matrix, it follows:
\begin{equation}\label{eqn:initial-consistency-diff-psiC1-psiC2}
    \eta\xi^\top(\overline\psi_{C_1} - \overline\psi_{C_2}) = -\eta\xi^\top V(x_P) \begin{bmatrix}
        \theta_{C_1}^* - \theta_{C_2}^* \\
        \nu_{C_1}^* - \nu_{C_2}^*
    \end{bmatrix} +  O_P\left(r^2,~\frac{1}{\sqrt{n_{C_1}}},~\frac{1}{\sqrt{n_{C_2}}}\right).
\end{equation}
The fixed matrix $\xi^\top$ selects the coordinates of the target effect as $\xi^\top(\theta,\nu)^\top = \theta$, and hence the product $\xi^\top V(x_P)$ simplifies:
\begin{equation*}
    \xi^\top V(x_P) = \xi^\top \begin{bmatrix}
        V_{\theta\theta}(x_P) & V_{\theta\nu}(x_P) \\
        V_{\nu\theta}(x_P) & V_{\nu\nu}(x_P)
    \end{bmatrix} = \begin{bmatrix}
        V_{\theta\theta}(x_P) & V_{\theta\nu}(x_P)
    \end{bmatrix}.
\end{equation*}
Under Neyman orthogonality one has $V_{\theta\nu}(x_P) = {\bf 0}$, implying that $\xi^\top V(x_P) = [V_{\theta\theta}(x_P)~~{\bf 0}]$, and \eqref{eqn:initial-consistency-diff-psiC1-psiC2} becomes
\begin{equation}\label{eqn:asymptotic-FPT-vector}
    \eta\xi^\top(\overline\psi_{C_1} - \overline\psi_{C_2}) = -\eta  V_{\theta\theta}(x_P) (\theta_{C_1}^* - \theta_{C_2}^*) +  O_P\left(r^2,~\frac{1}{\sqrt{n_{C_1}}},~\frac{1}{\sqrt{n_{C_2}}}\right).
\end{equation}

\paragraph{Asymptotic analysis.}  

Let $E$ be any term satisfying $E = O_P(r^2,~1/\sqrt{n_{C_1}},~1/\sqrt{n_{C_2}})$. In our asymptotic regime with $r^{-2}\ll n_{C_j} \implies 1/\sqrt{n_{C_j}} = o(r)$, one has
\begin{equation*}
    E = O_P\left(r^2,~\frac{1}{\sqrt{n_{C_1}}},~\frac{1}{\sqrt{n_{C_2}}}\right) \implies E = o_P\left(r,~\frac{1}{\sqrt{n_{C_1}}},~\frac{1}{\sqrt{n_{C_2}}}\right)
\end{equation*}
and therefore \eqref{eqn:asymptotic-FPT-vector} satisfies
\begin{equation}\label{eqn:Delta-FPT-norm-term-asympt}
    \eta\xi^\top(\overline\psi_{C_1} - \overline\psi_{C_2}) = -\eta  V_{\theta\theta}(x_P) (\theta_{C_1}^* - \theta_{C_2}^*) +  o_P\left(r,~\frac{1}{\sqrt{n_{C_1}}},~\frac{1}{\sqrt{n_{C_2}}}\right).
\end{equation}
Write $A = V_{\theta\theta}(\theta_{C_1}^* - \theta_{C_2}^*)$ such that $\widetilde\Delta_\eta^\FPT(C_1,C_2)$ is equivalently written $\widetilde\Delta_\eta^\FPT(C_1,C_2) = (n_{C_1}n_{C_2}/n_P^2) \cdot \eta^2 \norm{A + E}^2$. Consider the difference
\begin{align*}
    \widetilde\Delta_\eta^\FPT(C_1,C_2) - \frac{n_{C_1}n_{C_2}}{n_P^2} \eta^2 \norm{V_{\theta\theta}(\theta_{C_1}^* - \theta_{C_2}^*)}^2 &= \frac{n_{C_1}n_{C_2}}{n_P^2} \eta^2 \left(\norm{A + E}^2 - \norm{A}^2\right), \\
    &= \frac{n_{C_1}n_{C_2}}{n_P^2} \eta^2 \left( 2\langle A, E\rangle + \norm{E}^2 \right).
\end{align*}
One repeats the same arguments used in the final asymptotic analysis of Lemma~\ref{lem:Delta-crit-asymptotic} to show
\begin{equation*}
    2\langle A, E\rangle +  \norm{E}^2 =  o_P\left(r^2,~\frac{1}{{n_{C_1}}},~\frac{1}{{n_{C_2}}}\right),
\end{equation*}
and thus
\begin{equation*}
    \widetilde\Delta_\eta^\FPT(C_1,C_2) - \frac{n_{C_1}n_{C_2}}{n_P^2} \eta^2 \norm{V_{\theta\theta}(x_P) (\theta_{C_1}^* - \theta_{C_2}^*)}^2 = o_P\left(r^2,~\frac{1}{{n_{C_1}}},~\frac{1}{{n_{C_2}}}\right),
\end{equation*}
as desired.
\end{proof}

\begin{proof}[{\bf Proof of Proposition~\ref{prop:Delta-V-crit-asymptotic}}] First, under Assumptions~\ref{asm:all-assumptions} the $\theta$-block $V_{\theta\theta}(x_P)$ of the local Jacobian $V(x_P)$ is strictly positive definite and thus $\norm{\cdot}_V$ defines a true norm. From the proof of Lemma~\ref{lem:Delta-crit-asymptotic}:
\begin{equation*}
    \hat\theta_{C_1} - \hat\theta_{C_2} = \theta_{C_1}^* - \theta_{C_2}^* + o_P\left(r, ~\frac{1}{\sqrt{n_{C_1}}},~\frac{1}{\sqrt{n_{C_2}}}\right).
\end{equation*}
The matrix $V(x_P)$ is non-random and fixed given $P$ and $\eta$ is a fixed scalar. It follows:
\begin{equation*}
    \eta V_{\theta\theta}(x_P) (\hat\theta_{C_1} - \hat\theta_{C_2}) = \eta  V_{\theta\theta}(x_P) (\theta_{C_1}^* - \theta_{C_2}^*) + o_P\left(r, ~\frac{1}{\sqrt{n_{C_1}}},~\frac{1}{\sqrt{n_{C_2}}}\right).
\end{equation*}
Up to a negative factor, the expression on the right is precisely the same as \eqref{eqn:Delta-FPT-norm-term-asympt} in the proof of Lemma~\ref{lem:Delta-FPT-crit-asymptotic}, and thus one repeats the arguments to arrive at
\begin{equation*}
    \norm{\eta V_{\theta\theta}(x_P) (\hat\theta_{C_1} - \hat\theta_{C_2})}_2^2 = \norm{ \eta  V_{\theta\theta}(x_P) (\theta_{C_1}^* - \theta_{C_2}^*)}_2^2 + o_P\left(r^2, ~\frac{1}{{n_{C_1}}},~\frac{1}{{n_{C_2}}}\right),
\end{equation*}
and hence
\begin{equation*}
    \Delta_{\eta V}(C_1,C_2) = \frac{n_{C_1}n_{C_2}}{n_P^2}  \eta^2 \norm{  V_{\theta\theta}(x_P) (\theta_{C_1}^* - \theta_{C_2}^*)}_2^2 + o_P\left(r^2, ~\frac{1}{{n_{C_1}}},~\frac{1}{{n_{C_2}}}\right).
\end{equation*}
The right hand side is precisely the same as in the statement of Lemma~\ref{lem:Delta-FPT-crit-asymptotic} established for $\widetilde\Delta_\eta^\FPT(C_1,C_2)$, and thus
\begin{equation}
    \widetilde\Delta_\eta^\FPT(C_1,C_2) -  \Delta_{\eta V}(C_1,C_2) =  o_P\left(r^2, ~\frac{1}{{n_{C_1}}},~\frac{1}{{n_{C_2}}}\right),
\end{equation}
as desired.
\end{proof}

\begin{proof}[{\bf Proof of Lemma~\ref{lem:Delta-implies-Delta-V-spec}}] Firstly, Specifications~\ref{spec:subsampling} (subsampling) and~\ref{spec:honest} (honesty) describe conditions imposed on the sampling mechanism and are not affected by the form of the splitting criterion. It remains to verify whether Specification~\ref{spec:symmetric} (symmetry), Specification~\ref{spec:balanced} (balanced/$\omega$-regular), and Specification~\ref{spec:randomized} (randomized/random-split) are satisfied by $\cT(\Delta_V)$.

\begin{enumerate}
    \item {\bf Specification~\ref{spec:symmetric}: Symmetry.} A tree is said to be symmetric if its estimates are invariant under permutations of the tree's training samples. Conditional on a sequence of criterion values computed over splits of $P$, the CART mechanism of selecting the best split by scanning over the collection of candidates does not depend on the parent samples at all. This means that asymmetry could only enter through the criterion values. Therefore, a sufficient condition for symmetry in the tree estimates with respect to permutations of the tree samples is whether the criterion $\Delta_V(C_1,C_2)$ is symmetric. Conditional on the child solutions $\hat\theta_{C_1},\hat\theta_{C_2}$, the map $(\hat\theta_{C_1},\hat\theta_{C_2})\mapsto \Delta_V(C_1,C_2)$ does not depend on the parent samples at all, and therefore asymmetry could only enter through child solutions $\hat\theta_{C_j}$. However, both criteria use precisely the same child solutions $\hat\theta_{C_j}$ in \eqref{eqn:local-est-Cj}, and therefore $\Delta_V(C_1,C_2)$ will be symmetric whenever $\Delta(C_1,C_2)$ is symmetric (specifically, whenever $\psi_{\theta,\nu}(O_i)$ is symmetric with respect to permutations of $i$).    

    \item {\bf Specification~\ref{spec:balanced}: Balanced/$\omega$-regular.} This condition is enforced by by GRF by adding an additional stopping condition to the gradient-based version of Algorithm~\ref{alg:grffpt-pseudocode-stage1-tree}. Specifically, GRF stops a recursive splitting path if a proposed $\Delta$-optimal split were to send fewer than $\omega n_P$ of the parent samples into either child. Simply stated, GRF enforces balanced splits by defining the set of valid candidate splits as those with at least $\omega n_P$ parent samples in each child. This is left unchanged by our method.

    \item {\bf Specification~\ref{spec:randomized}: Randomized/random-split.} The asymptotic theory of GRF requires that, at each node, each variable is selected for a split with some lower bound probability $\pi > 0$. In order to satisfy the minimum split probability GRF uses the feature sampling device of \citet{denil2014narrowing} which, at each step, considers only $\min\{\max\{\textrm{Poisson}(m), 1\}, p\}$ randomly selected features as candidate variables, where $p = \dim(\cX)$ and $m$ is a GRF tuning parameter. In other words, GRF defines the set of valid candidate splits such that the set of valid splitting dimensions is itself a random variable. This mechanism is left unchanged under our method. 
    
    No column of $V(x_P)$ is all-zero $V_{\cdot,k}(x_P) \neq {\bf 0}$ because $V(x_P)$ is strictly positive definite symmetric, and therefore $\Delta_V(C_1,C_2)$ will not systematically ignore signals along parameter dimensions $\theta_k$ that can be detected by the $\Delta$-criterion. Finally,
    \begin{equation*}
        \hat\theta_{C_1} - \hat\theta_{C_1} \neq {\bf 0} \implies V_{\theta\theta}(x_P)(\hat\theta_{C_1} - \hat\theta_{C_2}) \neq {\bf 0},
    \end{equation*}
    because $V(x)$ is strictly positive definite symmetric by Assumption~\ref{asm:differentiable-M}. Therefore $\Delta(C_1,C_2) > 0 \implies \Delta_V(C_1,C_2) > 0$ meaning that the $\Delta_V$-criterion is non-degenerate and will always be able to select at least one feature whenever the $\Delta$-criterion can select a feature.
\end{enumerate}

Therefore, all five specifications are met, and one concludes that $\cT(\Delta_V)$ must satisfy the forest Specifications~\ref{spec:all-specifications} whenever they are satisfied by $\cT(\Delta)$.
\end{proof}

\subsection{Asymptotic equivalence of the pseudo-outcome approximation for VCM/HTE models}\label{app:asymp-equiv-approx-fp}

In this section we establish the asymptotic equivalence of the further acceleration of the fixed-point method proposed in Section~\ref{sec:applications} for VCM/HTE models. The accelerated algorithm is based on $\FPT$ pseudo-outcomes that use an approximation $\tilde\theta_P$ for the actual parent solutions $\hat\theta_P$ in \eqref{eqn:vcm-hte-thetaP}. Specifically, the parent-leaf approximations $\tilde\theta_P$ are found by a single gradient descent step towards $\hat\theta_P$ taken from the origin \eqref{eqn:one-step-theta-approx-appendix}. Let $\rho_i^\FPT$ denote the original $\FPT$ pseudo-outcomes for VCM/HTE models:
\begin{equation}\label{eqn:fp-het-treatment-effect-pseudo-outcomes-appendix}
    \rho_i^\FPT \coloneqq -(W_i - \overline W_P) \left(Y_i - \overline Y_P - (W_i - \overline W_P)^\top \hat\theta_P\right),
\end{equation}
where the solution $\hat\theta_P$ for the local model over the parent $P$ are precisely the OLS coefficients from the regression the centered outcomes $Y_i - \overline Y_P \in \bbR$ on the centered regressors $W_i - \overline W_P \in \bbR^K$. In contrast, let $\phi_i^\FPT$ denote approximations of $\rho_i^\FPT$ pseudo-outcomes that are of the form
\begin{equation}\label{eqn:approx-fpt-pseudo-outcomes}
    \phi_i^\FPT \coloneqq -(W_i - \overline W_P) \left( Y_i - \overline Y_P - (W_i - \overline W_P)^\top \tilde\theta_P\right),
\end{equation}
where $\tilde\theta_P$ approximates $\hat\theta_P$ as:
\begin{equation}\label{eqn:one-step-theta-approx-appendix}
    \tilde\theta_P \coloneqq \gamma \cdot \frac{1}{n_P} \sum_{\{i:X_i \in P\}} (W_i - \overline W_P) (Y_i - \overline Y_P) = \gamma \cdot \frac{1}{n_P} W_P^\top Y_P.
\end{equation}
Here, $W_P \in \bbR^{n_P\times K}$ and $Y_P \in \bbR^{n_P}$ denote the centered data matrices, $W_P \coloneqq [W_i - \overline W_P]_{i:X_i \in P}$ and $Y_P \coloneqq [Y_i - \overline Y_P]_{i:X_i \in P}$, and the scalar $\gamma > 0$ denotes the exact line search step size corresponding to the regression of the centered outcomes on the centered regressors:
\begin{equation}
    \gamma \coloneqq \frac{ \norm{W_P^\top Y_P}_2^2 }{ \norm{W_P W_P^\top Y_P}_2^2 }.
\end{equation}

\begin{lemma}\label{lem:approx-fp-method}

Let $\tilde\theta_{C_j}$ denote the $\FPT$ estimator of the form \eqref{eqn:fp-estimator-eta-pseudo-outcome} for the child solution $\hat\theta_{C_j}$ for VCM/HTE models. One can express $\tilde\theta_{C_j}$ in terms of the corresponding fixed-point pseudo-outcomes:
\begin{equation*}
    \tilde\theta_{C_j} \coloneqq \hat\theta_P + \frac{1}{n_{C_j}}\sum_{i:X_i \in C_j} \rho_i^\FPT.
\end{equation*}
Similarly, denote by $\bar\theta_{C_j}$ the $\FPT$ estimator of $\hat\theta_{C_j}$ induced by pseudo-outcomes approximations $\phi_i^\FPT$:
\begin{equation*}
    \bar\theta_{C_j} \coloneqq \hat\theta_P + \frac{1}{n_{C_j}} \sum_{\{i:X_i \in C_j\}} \phi_i^\FPT.
\end{equation*}
Then, under the assumptions of Proposition~\ref{prop:Delta-V-crit-asymptotic}, $\bar\theta_{C_j}$ is consistent for $\tilde\theta_{C_j}$ as:
\begin{equation*}
    \norm{\tilde\theta_{C_j} - \bar\theta_{C_j}} = o_P(1).
\end{equation*}
\end{lemma}

\begin{proof} 
 A direct calculation reveals that the difference between the original $\FPT$ pseudo-outcomes $\rho_i^\FPT$ in \eqref{eqn:fp-het-treatment-effect-pseudo-outcomes-appendix} and the approximations $\phi_i^\FPT$ in \eqref{eqn:approx-fpt-pseudo-outcomes} satisfy
\begin{align}
    \rho_i^\FPT - \phi_i^\FPT &= -(W_i - \overline W_P) \left(\left[Y_i - \overline Y_P - (W_i - \overline W_P)^\top \hat\theta_P\right] - \left[Y_i - \overline Y_P - (W_i - \overline W_P)^\top \tilde\theta_P\right] \right), \nonumber \\
    &= (W_i - \overline W_P) (W_i - \overline W_P)^\top (\hat\theta_P - \tilde\theta_P). \label{eqn:fpt-pseudo-outcome-difference}
\end{align}

Therefore, the difference between the original $\FPT$ child estimator $\tilde\theta_{C_j}$ and the approximation $\bar\theta_{C_j}$ satisfies
\begin{align*}
    \tilde\theta_{C_j} - \bar\theta_{C_j} &= \frac{1}{n_{C_j}} \sum_{\{i:X_i \in C_j\}} (\rho_i^\FPT - \phi_i^\FPT), \\
    &= \frac{1}{n_{C_j}} \sum_{\{i:X_i \in C_j\}} (W_i - \overline W_P) (W_i - \overline W_P)^\top (\hat\theta_P - \tilde\theta_P), \\
    &= S_{C_j} (\hat\theta_P - \tilde\theta_P),
\end{align*}
where we denote $S_{C_j} \coloneqq \frac{1}{n_{C_j}}\sum_{\{i:X_i \in C_j\}} (W_i - \overline W_P) (W_i - \overline W_P)^\top$. Therefore,
\begin{equation}\label{eqn:tildethetaCj-barthetaCj-bound}
    \norm{\tilde\theta_{C_j} - \bar\theta_{C_j}} = \norm{S_{C_j} (\hat\theta_P - \tilde\theta_P)} \leq \norm{S_{C_j}}_F \norm{\hat\theta_P - \tilde\theta_P}.
\end{equation}
Under GRF's regularity conditions, in a limit where $n_{C_j}\to\infty$ and the parent radius $r \coloneqq \sup_{\{i:X_i\in P\}}\norm{X_i - \overline X_P}$ goes to zero $r\to 0$, we have $S_{C_j} \stackrel{p}{\to} Q$ for some positive semidefinite symmetric matrix $Q$, and hence $\norm{S_{C_j}}_F = O_P(1)$. Meanwhile, by definition \eqref{eqn:vcm-hte-thetaP} for the OLS coefficients $\hat\theta_P$ and definition \eqref{eqn:one-step-theta-approx-appendix} for the one-step approximations $\bar\theta_P$, we have
\begin{align}
    \norm{\hat\theta_P - \tilde\theta_P} &= \norm{ \left( -A_P^{-1} \cdot n_P^{-1} W_P^\top Y_P \right) - \left( \gamma \cdot n_P^{-1} W_P^\top Y_P \right) },\nonumber \\
    &= \norm{\left(  -A_P^{-1} - \gamma \bbI\right) \cdot n_P^{-1} W_P^\top Y_P },\nonumber \\
    &\leq \norm{-A_P^{-1} - \gamma \bbI}_F \norm{ n_P^{-1} W_P^\top Y_P},\nonumber \\
    &= \norm{\left[n_P^{-1} W_P^\top W_P\right]^{-1} - \gamma \bbI}_F \norm{ n_P^{-1} W_P^\top Y_P},\label{eqn:hatthetaP-tildethetaP-bound}
\end{align}
where $-A_P = n_P^{-1} W_P^\top W_P$ follows from the definition of $A_P$ as an estimator of the Jacobian $\nabla \psi$, e.g. \eqref{eqn:AP-matrix-hte} in the context of VCM/HTE models. Under the Lipschitz continuity Assumptions~\ref{asm:lipschitz-M} \&~\ref{asm:continuous-psi}, one has the standard stochastic bound for the cross term $n_P^{-1}W_P^\top Y_P$:
\begin{equation}
     \norm{ n_P^{-1} W_P^\top Y_P} = O_P\left(r,~\frac{1}{\sqrt{n_P}}\right),
\end{equation}
while the difference $[n_P^{-1} W_P^\top W_P]^{-1} - \gamma \bbI$ is stochastically bound as
\begin{equation*}
    \norm{\left[n_P^{-1} W_P^\top W_P\right]^{-1} - \gamma \bbI}_F = O_P(1),
\end{equation*}
because $n_P^{-1} W_P^\top W_P \stackrel{p}{\to}\cov(W_i \mid X_i \in P)$ is non-singular under Assumption~\ref{asm:differentiable-M}. Coupling these stochastic bounds together according to \eqref{eqn:hatthetaP-tildethetaP-bound} gives
\begin{equation*}
    \norm{\hat\theta_P - \tilde\theta_P} = O_P\left(r,~\frac{1}{\sqrt{n_P}}\right),
\end{equation*}
and trivially, because $n_{C_j} < n_P$, 
\begin{equation}\label{eqn:big-OP-hatP-tildeP}
    \norm{\hat\theta_P - \tilde\theta_P} = O_P\left(r,~\frac{1}{\sqrt{n_{C_j}}}\right).
\end{equation}
Under Proposition 1 of GRF one assumes $r^{-2} \ll n_{C_1}, n_{C_2}$ and thus, in an asymptotic regime where $n_{C_j}\to \infty$ and $r\to 0$, one has $1/\sqrt{n_{C_j}} = o(r)$, and hence:
\begin{equation}\label{eqn:small-oP-hatP-tildeP}
    \norm{\hat\theta_P - \tilde\theta_P} = o_P(1).
\end{equation}
Returning to \eqref{eqn:tildethetaCj-barthetaCj-bound}, the consistency of the parent approximation $\tilde\theta_P$ as \eqref{eqn:small-oP-hatP-tildeP} implies that the approximation $\bar\theta_{C_j}$ is itself consistent for the original $\FPT$ child estimator $\tilde\theta_{C_j}$:
\begin{equation}\label{eqn:approx-fpt-method-rate}
    \norm{\tilde\theta_{C_j} - \bar\theta_{C_j}} = o_P(1),
\end{equation}
as desired.
\end{proof}

\section{Implementation Details}\label{app:implementation-details}

\subsection{Honest subsampling}\label{app:honest-indep}

In this section we present the honest subsampling mechanism. Trees are used to form partitions of the input space such as to to specify weight functions $\alpha_i(x)$, defined as
\begin{equation}\label{eqn:appendix-grf-weight}
    \alpha_i(x) \coloneqq \frac{1}{B}\sum^B_{b=1}\alpha_{bi}(x), \quad \text{for} \quad  \alpha_{bi}(x) \coloneqq \frac{\mathds 1(X_i \in L_b(x))}{|L_b(x)|}, \qquad i =1,\ldots,n,
\end{equation}
where $L_b(x)$ denotes a subset training samples that fall alongside $x$ according to the partition of tree $b$. The honesty mechanism ensures that no observation in leaf $L_b(x)$ was used to build the partition rules of tree $b$. This is achieved by separating an initial subsample into two subsets: One for building the partition rules, and the other allocated as samples to the local leaves $L_b(x)$ according to the trained rules. Below, we give a detailed outline of how subsampling and honest sample splitting is used to train a forest of trees, then show that weight function $\alpha_i(x)$ given by honest trees according to \eqref{eqn:appendix-grf-weight} is conditionally independent of $O_i$ given $X_i$.

\begin{mdframed}[frametitle={Honest subsampling for GRF}]
For tree $b \in \{1, \ldots, B\}$,

\begin{enumerate}
    \item ({\em Subsampling}). Draw an initial subsample $\cI^{(b)}$ of size $s\coloneqq |\cI^{(b)}|$ from the training set (without replacement).
    
    \item ({\em Honest splitting}). Split $\cI^{(b)}$ into disjoint sets $\cJ_1^{(b)}$ and $\cJ_2^{(b)}$ of size $|\cJ_1^{(b)}| = \lfloor s/2\rfloor$ and $|\cJ_2^{(b)}| = \lceil s/2\rceil$. 
    \begin{enumerate}
        \item Train tree $T(\cJ^{(b)}_1)$ based on the first subsample $\{(X_i, O_i) : i \in \cJ^{(b)}_1\}$. Let $\cR^{(b)}_1,\ldots,\cR^{(b)}_M$ denote the partition of $\cX$ induced by $T(\cJ^{(b)}_1)$ such that
        \begin{equation*}
            \cR^{(b)}_m \coloneqq \left\{ x \in \cX : \text{$x$ satisfies the partition rules for leaf $m$ of $T(\cJ^{(b)}_1)$} \right\}.
        \end{equation*}

        \item Subset the samples from the second subsample $\{X_i : i\in \cJ^{(b)}_2\}$ according to the trained rules of $T(\cJ^{(b)}_1)$, i.e. the samples of $\cJ^{(b)}_2$ in the leaves are determined by the rules of $T(\cJ^{(b)}_1)$.
    \end{enumerate}
\end{enumerate} 
\end{mdframed}

For any $x\in \cX$, the local leaf $L_b(x)$ that appears in \eqref{eqn:appendix-grf-weight} is defined as the specific subset of $\cJ^{(b)}_2$ samples belonging to the same leaf of tree $T(\cJ^{(b)}_1)$ as $x$,
\begin{equation*}
    L_b(x) = \{X_i \in \cR^{(b)}_m : i\in\cJ^{(b)}_2 \text{ and } x\in \cR^{(b)}_m\},
\end{equation*}

\paragraph{Conditional independence of $\alpha_i(x)$ and $O_i$ given $X_i$.}

By definition, the partition rules of tree $T(\cJ^{(b)}_1)$ depend only on the $\cJ^{(b)}_1$ subsample. The rules of a tree operate only on covariate values, and therefore the task of subsetting $\{X_i : i\in\cJ^{(b)}_2\}$ into leaves according to the rules of $T(\cJ^{(b)}_1)$ requires knowledge of the $X_i$ values from the $\cJ^{(b)}_2$ subsample but not necessarily the $O_i$. Based on this understanding, we will show that $\alpha_{i}(x)$ is conditionally independent of $O_i$ given $X_i$. Based on \eqref{eqn:appendix-grf-weight}, it is sufficient to show
\begin{equation*}
    \bbE[\alpha_{bi}(x)\mid O_i,X_i] = \bbE[\alpha_{bi}(x)\mid  X_i].
\end{equation*}

\begin{enumerate}[label=Case \arabic*.]
    \item Suppose $i \notin \cJ^{(b)}_2$. By definition $L_b(x) \subset \{X_j : j \in \cJ^{(b)}_2\}$. It is immediate that $\mathds 1(\{X_i \in L_b(x)\}) = 0$, and therefore $\alpha_{bi}(x) = 0$, and trivially
    \begin{equation*}
        \bbE[\alpha_{bi}(x)\mid O_i,X_i] = \bbE[\alpha_{bi}(x)\mid  X_i]=0, \qquad \text{for all} \quad i \notin \cJ^{(b)}_2.
    \end{equation*}

    \item Suppose $i \in \cJ^{(b)}_2$. We show that each component used to specify $\alpha_{bi}(x)$ in  \eqref{eqn:appendix-grf-weight} is conditionally independent of $O_i$ given $X_i$:
    \begin{itemize}
        \item Tree $T(\cJ^{(b)}_1)$ is trained using only the $\cJ^{(b)}_1$ subsample. This does not depend on $O_i$, for all $i \in \cJ^{(b)}_2$, conditionally on $X_i$ or otherwise.
        
        \item The rules of tree $T(\cJ^{(b)}_1)$ operate only on input values. Therefore, conditionally on $X_i$ for all $i \in \cJ^{(b)}_2$, the leaves of the $\cJ^{(b)}_2$ subsample specified by tree $T(\cJ^{(b)}_1)$ do not depend on the value of $O_i$.
        
        \item Leaf $L_b(x)$ is the specific subset of the $\cJ^{(b)}_2$ samples satisfying the same partition rules of $T(\cJ^{(b)}_1)$ as $x$. Given the leaves have been specified by the previous step, this depends only on $x$.
     \end{itemize}
     
     Therefore, the individual component functions $\alpha_{bi}(x) = \mathds 1(\{X_i \in L_b(x)\})/|L_b(x)|$ are conditionally independent of $O_i$ given $X_i$,
     \begin{equation*}
        \bbE[\alpha_{bi}(x)\mid O_i,X_i] = \bbE[\alpha_{bi}(x)\mid  X_i], \qquad \text{for all} \quad i \in \cJ^{(b)}_2.
    \end{equation*}    
\end{enumerate}

\paragraph{Demonstration of honest subsampling.}

Let $\{(X_i, O_i)\}_{i=1}^n$ denote a training set of $n = 20$ observations, where each $X_i = (X_{i,1}, X_{i,2})$ is over $\mathcal{X} \equiv \mathbb{R}^2$. We will use a forest of a single tree ($B = 1$) to specify the functional form of weights $\alpha_i(x)$.

\begin{enumerate}[label=\arabic*.]
    \item (Subsampling). Draw an initial subsample $\cI$ of size $s = 10$. 
    
    \item (Honest splitting). Split $\cI$ into two disjoint sets \magentabg{$\cJ_1$} and \cyanbg{$\cJ_2$}, each with $s/2 = 5$ samples.

\begin{minipage}{.65\textwidth}
\begin{center}
{\footnotesize
\begin{tabular}{r|c|c}
    $i$ & $X_{i,1}~X_{i,2}$ & $O_i$ \\
    \hline
    1 & & \\
    \vdots & & \\
    20 & & 
\end{tabular}
\begin{tikzcd}[labels={font=\normalsize},semithick]
{} \arrow[r,"\cI"] & {}
\end{tikzcd}
\begin{tabular}{r|c|c}
    $i \in \cI$ & $X_{i,1}~X_{i,2}$ & $O_i$ \\
    \hline
    \magentabg{2} & & \\
    \magentabg{3} & & \\
    \cyanbg{5} & & \\
    \magentabg{8} & & \\
    \cyanbg{10} & & \\
    \cyanbg{11} & & \\
    \cyanbg{14} & & \\
    \magentabg{15} & & \\
    \cyanbg{16} & & \\
    \magentabg{20} & & 
\end{tabular}}
\begin{tikzcd}[labels={font=\normalsize},semithick]
    & {} \\
    {} \arrow[ur,"\magentabg{$\cJ_1$}"] & \\
    {} \arrow[dr,"\cyanbg{$\cJ_2$}"'] & \\
    & {}
\end{tikzcd}
\end{center}
\end{minipage}
\begin{minipage}{.25\textwidth}
\begin{center}
{\footnotesize
\begin{tabular}{r|cc|c}
    $i \in$ \magentabg{$\cJ_1$} & $X_{i,1}$ & $X_{i,2}$ & $O_i$ \\
    \hline
    \magentabg{2} & & & \\
    \magentabg{3} & & & \\
    \magentabg{8} & & & \\
    \magentabg{15} & & & \\
    \magentabg{20} & & &  
    \vspace*{20pt}
\end{tabular}

\begin{tabular}{r|cc|c}
    $i \in$ \cyanbg{$\cJ_2$} & $X_{i,1}$ & $X_{i,2}$ & $O_i$ \\
    \hline 
    \cyanbg{5} & 1 & 0 & \\
    \cyanbg{10} & 2 & -2 & \\
    \cyanbg{11} & 0 & 1 & \\
    \cyanbg{14} & 1 & -2 & \\
    \cyanbg{16} & 2 & 2 &    
\end{tabular}}
\end{center}
\end{minipage}

    \begin{enumerate}[label=(\alph*)]
        \item Train a tree using the data from the first subsample \magentabg{$\cJ_1$}, inducing a partition of $\cX \equiv \bbR^2$. Suppose the fitted tree has the following structure:

\begin{center}
{\small
\begin{tikzpicture}[node/.style]
    \node [node,draw] (A) {$X_{i,1} < 3$};
    \path (A) ++(-125:\nodeDist) node [node,draw] (B) {$X_{i,2} < -1$};
    \path (A) ++(-55:\nodeDist) node [node] (C) {$R_3$};
    \path (B) ++(-125:\nodeDist) node [node] (D) {$R_1$};
    \path (B) ++(-55:\nodeDist) node [node] (E) {$R_2$};

    \draw (A) -- (B) node [left,pos=0.5]{yes} (A);
    \draw (A) -- (C) node [right,pos=0.5]{no} (A);
    \draw (B) -- (D) node [left,pos=0.5]{yes} (A);
    \draw (B) -- (E) node [right,pos=0.5]{no} (A);
\end{tikzpicture}}
\begin{tikzpicture}
\begin{axis}[
  axis lines=middle,
  axis line style={Stealth-Stealth,thick},
  xmin=-1.2,xmax=3.7,ymin=-2.5,ymax=2.5,
  xtick distance=1,
  ytick distance=1,
  tick label style={font=\footnotesize, anchor=north west},
  xlabel=$X_{i,1}$,
  ylabel=$X_{i,2}$,
  label style={anchor=south},
  grid=none,
  ]
    \draw[magenta,ultra thick](axis cs:3, -8) -- (axis cs:3, 8);
    \draw[magenta,ultra thick](axis cs:-8, -1) -- (axis cs:3, -1);
    \node[] at (axis cs: 3.4,-0.85) {{$\bm{\cR_3}$}};
    \node[] at (axis cs: 2.5,0.75) {{$\bm{\cR_2}$}};
    \node[] at (axis cs: 2.5,-1.5) {{$\bm{\cR_1}$}};
\end{axis}
\end{tikzpicture}
\end{center}

        \item Use the trained partition rules to subset the \cyanbg{$\cJ_2$} subsample into separate leaves.

\begin{center}
\begin{tikzpicture}
\begin{axis}[
  axis lines=middle,
  axis line style={Stealth-Stealth,thick},
  xmin=-1.2,xmax=3.7,ymin=-2.5,ymax=2.5,
  xtick distance=1,
  ytick distance=1,
  xticklabel=\empty,
  yticklabel=\empty,
  xlabel=$X_{i,1}$,
  ylabel=$X_{i,2}$,
  label style={anchor=south},
  grid=none,
  ]
    \draw[magenta,ultra thick](axis cs:3, -8) -- (axis cs:3, 8);
    \draw[magenta,ultra thick](axis cs:-8, -1) -- (axis cs:3, -1);
    \node[label={120:$X_5$},cyan,circle,fill,inner sep=2pt] at (axis cs:1,0) {};
    \node[label={120:$X_{10}$},cyan,circle,fill,inner sep=2pt] at (axis cs:2,-2) {};
    \node[label={120:$X_{11}$},cyan,circle,fill,inner sep=2pt] at (axis cs:0,1) {};
    \node[label={120:$X_{14}$},cyan,circle,fill,inner sep=2pt] at (axis cs:1,-2) {};
    \node[label={120:$X_{16}$},cyan,circle,fill,inner sep=2pt] at (axis cs:2,2) {};
    \node[] at (axis cs: 3.4,-0.85) {{$\bm{\cR_3}$}};
    \node[] at (axis cs: 2.5,0.75) {{$\bm{\cR_2}$}};
    \node[] at (axis cs: 2.5,-1.5) {{$\bm{\cR_1}$}};
\end{axis}
\end{tikzpicture}
\end{center}
    
    The tree trained on the \magentabg{$\cJ_1$} subsample will subset the \cyanbg{$\cJ_2$} subsample as
    \begin{equation*}
        \{X_i\,:\,i\in \cJ_2\} = \{X_{10},X_{14}\} \cup \{X_5, X_{11},X_{16}\} \cup \emptyset,
    \end{equation*}
    where we include the trivial union with $\emptyset$ to note that the tree assigns none of the \cyanbg{$\cJ_2$} samples to the partition of $\bbR^2$ where $X_{i,1} \geq 3$.
    \end{enumerate}
\end{enumerate}

The leaf $L_b(x)$ is the specific subset of the \cyanbg{$\cJ_2$} subsample such that $X_i \in  \cyanbg{$\cJ_2$}$ satisfy the same partition rules as $x$. Given a test point $x = x_0$, there are three possible scenarios for $L_b(x_0)$ that correspond to the three regions $R_1,R_2,R_3 \subset \bbR^2$ in which the test point $x_0$ can appear.

\paragraph{Region 1.} If $x_0 \in \cR_1$ then $L_b(x_0) = \{X_{10},X_{14}\}$ and
    \begin{equation*}
        \alpha_{bi}(x_0) = \frac{\mathds 1(\{X_i \in L_b(x_0)\})}{|L_b(x_0)|} = \begin{cases}
            \frac{1}{2} & \text{if } i \in \{10, 14\}, \\
            0 & \text{otherwise.}
        \end{cases}
    \end{equation*}
    Therefore, $\alpha_i(x_0) = \frac{1}{2}$ for $i = 10, 14$ and zero for $i \in \{1,\ldots,20\}\setminus\{10,14\}$.

\paragraph{Region 2.}  If $x_0 \in \cR_2$ then $L_b(x_0) = \{X_{5},X_{11},X_{16}\}$ and
\begin{equation*}
    \alpha_{bi}(x_0) = \frac{\mathds 1(\{X_i \in L_b(x_0)\})}{|L_b(x_0)|} = \begin{cases}
        \frac{1}{3} & \text{if } i \in \{5, 11, 16\}, \\
        0 & \text{otherwise.}
        \end{cases}
\end{equation*}
Therefore, $\alpha_i(x_0) = \frac{1}{3}$ for $i = 5, 11, 16$ and zero $i \in \{1,\ldots,20\}\setminus\{5,11,16\}$.

\paragraph{Region 3.} If $x_0 \in \cR_3$ then $L_b(x_0) = \emptyset$. This is a degenerate case such that
\begin{equation*}
    \alpha_{bi}(x_0) =  \frac{\mathds 1(\{X_i \in L_b(x_0)\})}{|L_b(x_0)|}
\end{equation*}
is undefined, leading to a non-identifiability problem whenever $x_0 \in \cR_3$. When this occurs, \citet{grfpackage} recommends calculating $\alpha_i(x_0)$ based on only the trees with non-empty $L_b(x_0)$. Let $\cB \coloneqq \{b \in \{1,\ldots,B\}\,:\,|L_b(x_0)| > 0\}$ denote the indices of non-empty leaves associated with $x_0$. Then, the GRF weight functions based on this recommendation can be written as
\begin{equation*}
    \alpha_i(x_0) = \frac{1}{|\cB|}\sum_{b\in\cB} \alpha_{bi}(x_0).
\end{equation*}

\subsection{In-sample predictions}\label{app:insample-predictions}

There is additional bias associated with making predictions based on in-sample observations $X_i$ that may have been used either to train the tree structure or to populate the local leaves $L_b(x)$. The recommendation of \citet{grfpackage} is along the lines of the out-of-bag mechanism used by \citet{breiman2001random}. For an in-sample observation $x' \in \{X_i\}_{i=1}^n$, calculate weights $\alpha^\texttt{oob}_i(x')$ based only on those trees whose initial subsample $\cI^{(b)}$ does not contain $x'$. Then the out-of-bag weight is defined as:
\begin{equation*}
    \alpha_i^\texttt{oob}(x') \coloneqq \frac{1}{|\{b\,:\,x' \notin \cI^{(b)}\}|} \sum_{\{b:x'\notin \cI^{(b)}\}} \alpha^\texttt{oob}_{bi}(x') \qquad \text{for} \qquad \alpha_{bi}^\texttt{oob}(x') \coloneqq \frac{\mathds 1(X_i \in L_b(x'))}{|L_b(x')|}.
\end{equation*}
The in-sample prediction $\hat\theta^\texttt{oob}(x')$ for $x'$ is made by GRF by solving a version of the locally weighted estimating equation \eqref{eqn:emp-weighted-est-eqn} using out of bag weights $\alpha_i^\texttt{oob}(x')$ 
\begin{equation*}
    \left(\hat\theta^\texttt{oob}(x'),\hat\nu^\texttt{oob}(x')\right) \in \argmin_{\theta,\nu} \norm{\sum^n_{i=1}\alpha^\texttt{oob}_i(x') \psi_{\theta,\nu}(O_i) },
\end{equation*}
which preserve the consistency and asymptotic normality of the GRF estimator at in-sample observations.

\subsection{Algorithms and Pseudocode}\label{app:pseudocode}

\begin{algorithm}[tb]
\caption{The fixed-point tree algorithm} \label{alg:grffpt-pseudocode-stage1-tree}
\begin{algorithmic}
\FUNCTION{\textsc{\color{FPTcolor}TrainFixedPointTree}}
\STATE {\bfseries Input:} node $\cN$
\STATE node $P_0 \leftarrow$ \textsc{GetSamples}($\cN$)
\STATE queue $\cQ \leftarrow$ \textsc{InitializeQueue}($P_0$)
\WHILE{\textsc{NotNull}(node $P \leftarrow$ \textsc{Pop}($\cQ$))}
    \STATE \makebox[0.6\linewidth][l]{$(\hat\theta_P,\hat\nu_P) \leftarrow$ \textsc{SolveEstimatingEquation}($P$)} $\triangleright$ Computes \eqref{eqn:local-est-P}.
    \STATE \makebox[0.6\linewidth][l]{$\rho^\FPT \leftarrow$ \textsc{{\color{FPTcolor}FPTPseudoOutcomes}}($\hat\theta_P, \hat\nu_P$)} $\triangleright$ Applies \eqref{eqn:fp-pseudo-outcome} over $P$.
    \STATE \makebox[0.6\linewidth][l]{split $\Sigma \leftarrow$ \textsc{CARTSplit}($P$, $\rho^\FPT$)} $\triangleright$ Optimizes \eqref{eqn:fp-approx-criterion}.
    \IF{\textsc{SplitSucceeded}($\Sigma$)}
        \STATE \textsc{SetChildren}($P$, \textsc{GetLeftChild}($\Sigma$), \textsc{GetRightChild}($\Sigma$))
        \STATE \textsc{AddToQueue}($\cQ$, \textsc{GetLeftChild}($\Sigma$))
        \STATE \textsc{AddToQueue}($\cQ$, \textsc{GetRightChild}($\Sigma$))
    \ENDIF
\ENDWHILE
\STATE {\bfseries Output:} tree with root node $P_0$
\ENDFUNCTION
\end{algorithmic}
\justify
\textsc{Pop} returns and removes the oldest element of queue a $\cQ$, unless $\cQ$ is empty, in which case it returns \texttt{NULL}. \textsc{CARTSplit} runs a multivariate CART split on the pseudo-outcomes $\rho^\FPT \coloneqq \{\rho_i^\FPT\}_{i\in P}$, and either returns a pair of child nodes or indicates that no split of $P$ is possible.
\end{algorithm}

\begin{algorithm}[tb]
\caption{Stage I GRF-$\FPT$: Training a generalized random forest using fixed-point trees} \label{alg:grffpt-pseudocode-stage1-forest}
\begin{algorithmic}
\FUNCTION{\textsc{\color{FPTcolor}TrainGeneralizedRandomForestFPT}}
\STATE {\bfseries Input:} samples $\cS$, number of trees $B$
\FOR{$b = 1,\ldots, B$}
    \STATE set of samples $\cI \leftarrow$ \textsc{Subsample}($\cS$)
    \STATE \makebox[0.6\linewidth][l]{sets of samples $\cJ_\textsc{build},\cJ_\textsc{populate} \leftarrow$ \textsc{HonestSplit}($\cI$)} $\triangleright$ See honesty: Appendix~\ref{app:honest-indep}.
    \STATE \makebox[0.6\linewidth][l]{tree $\cT_b \leftarrow$ \textsc{{\color{FPTcolor}TrainFixedPointTree}}($\cJ_\textsc{build}$)}  $\triangleright$ See Algorithm~\ref{alg:grffpt-pseudocode-stage1-tree}.
    \STATE \makebox[0.6\linewidth][l]{leaves $\cL_b\leftarrow$ \textsc{PopulateLeaves}($\cT_b, \cJ_\textsc{populate}$)}  $\triangleright$ See honesty: Appendix~\ref{app:honest-indep}.
\ENDFOR
\STATE {\bfseries Output:} forest $\cF \leftarrow \{\cL_1, \ldots, \cL_B\}$
\ENDFUNCTION
\end{algorithmic}
\justify
\textsc{PopulateLeaves} creates a collection of subsets (leaves) of the $\cJ_\textsc{populate}$ samples based on the partition rules of tree $\cT_b$. For weight functions $\alpha_i(x)$, see \textsc{GetWeights} in Algorithm~\ref{alg:grffpt-pseudocode-stage2}. For Stage II, see \textsc{Estimate} in Algorithm~\ref{alg:grffpt-pseudocode-stage2}, where estimates $\hat\theta(x)$ are made given a forest $\cF$.
\end{algorithm}

\newcommand{\pluseq}{\mathrel{+}=}

\begin{algorithm}[tb]
\caption{GRF-$\FPT$: Estimates of $\theta^*(x)$} \label{alg:grffpt-pseudocode-stage2}
\begin{algorithmic}
\FUNCTION{\textsc{Estimate}}
\STATE {\bfseries Input:} forest $\cF$, test observation $x \in \cX$
\STATE weights $\alpha \leftarrow$ \textsc{GetWeights}($\cF,x$)
\STATE {\bfseries Output:} $\hat\theta(x)$, the solution to the weighted estimating equation \eqref{eqn:emp-weighted-est-eqn} using weights $\alpha$
\ENDFUNCTION
\vspace{0.5em}
\FUNCTION{\textsc{GetWeights}}
\STATE {\bfseries Input:} forest $\cF$, test observation $x$
\STATE \makebox[0.6\linewidth][l]{vector of weights $\alpha \gets$ \textsc{Zeros}($n$)} $\triangleright$ Initialize weights; $n = |\cS|$ used to train $\cF$.
\FOR{indices $i:X_i \in L_b(x)$}
    \STATE $\alpha[{i}] \pluseq 1/|L_b(x)|$
\ENDFOR
\STATE \makebox[0.6\linewidth][l]{{\bfseries Output:} local weights $\alpha\leftarrow \alpha/|\cF|$} $\triangleright$ Weights \eqref{eqn:grf-weights}.
\ENDFUNCTION
\end{algorithmic}
\justify
Stage II of the GRF-$\FPT$ algorithm. The procedure \textsc{Estimate} returns an estimate of $\theta^*(x)$ given a forest $\cF$ trained under Stage I and a test observation $x$; see Algorithm~\ref{alg:grffpt-pseudocode-stage1-forest}.
\end{algorithm}

\subsection{Simulation Details}\label{app:simulation-details}

\paragraph{Implementation details.} 

We implement the GRF-$\FPT$ algorithm in a fork of \texttt{grf} \citep{grfpackage} available at \url{https://github.com/dfleis/grf}. The functions \texttt{grf::lm\_forest} and \texttt{grf::multi\_arm\_causal\_forest} provide an easy to use interface for VCM and HTE estimation, respectively, and we allow the choice GRF-$\FPTone$, GRF-$\FPTtwo$, or GRF-$\grad$ to be controlled via the \texttt{method} argument. Code and data for reproducing all experiments and figures are available at \url{https://github.com/dfleis/grf-experiments}.

\paragraph{Data-generating settings.} 

The different setting for the target effects $\theta_k^*(x)$ include a sparse linear setting, a sparse logistic setting with interaction, a dense logistic setting, and a random function generator setting. Tables~\ref{tbl:vcm-sim-settings} and~\ref{tbl:hte-sim-settings} provide the details of each regime for VCM and HTE experiments, respectively, for the data-generating model \eqref{eqn:applications-model}. These tables also summarize the different settings used to generate the $K$-dimensional regressors $W_i = (W_{i,1},\ldots,W_{i,K})^\top$. For VCM experiments, $W_{i,k}\sim \cN(0,1)$ for all $k = 1,\ldots, K$. For HTE experiments, $W_i\mid X_i=x \sim \text{Multinomial}(1, (\pi_1(x),\ldots,\pi_K(x)))$, where $\pi_k(x)$ denotes the underlying probability the sample is observed as having treatment level $k \in \{1,\ldots, K\}$.

\begin{table}
    \centering
    \begin{tabular}{lc}
        \toprule
        Parameter & Values \\
        \hline
        $K$ & 4; 16; 64; 256 \\
        $n$ & 10,000; 20,000; 100,000 \\
        $\dim(\cX)$ & 5 \\
        nTrees & 100 \\
        \bottomrule
    \end{tabular}
    \begin{tabular}{lc}
        \toprule
        Parameter & Values \\
        \hline
        $K$ & 4; 16; \\
        $n$ & 1000; 4000 \\
        $\dim(\cX)$ & 2 \\
        nTrees & 100; 500 \\
        \bottomrule
    \end{tabular}
    \caption{Parameter values for VCM and HTE experiments in Section~\ref{sec:simulations}. Target/regressor dimension $K$, number of observations $n$, dimension of the auxiliary variables $\dim(\cX)$, and number of trees $\text{nTrees}$. Experiments include a large-$n$ setting (left table) and a small-$n$ setting (right table).}
    \label{tbl:sim_params}
\end{table}

\begin{table}[h!]
\centering
\begin{tabular}{cll}
\toprule
    VCM Setting & Effect function $\theta^*_k(x)$ & $W_{i,k}$ \\
\hline
    1 & $\theta^*_k(x) = \beta_{k1} x_1,~\beta_{k1} \sim \cN(0,1)$ &  $\cN(0,1)$ \\
    2 & $\theta^*_k(x) = \varsigma(\beta_{k1} x_1)\varsigma(\beta_{k2} x_2),~\beta_{k1},\beta_{k2} \sim \cN(0,1)$ & $\cN(0,1)$ \\
    3 & $\theta^*_k(x) = \varsigma({\beta}^\top_k x)$, for ${\beta}_k\sim\cN_p({\bf 0},\mathbb I)$ &  $\cN(0,1)$ \\
    4 & $\theta^*_k(x) = \text{RFG}(x)$ & $\cN(0,1)$ \\
\bottomrule
\end{tabular}
\caption{Settings for the true effects $\theta^*_k(\cdot)$ and the regressors $W_{i,k}$ for VCM experiments in Section~\ref{sec:simulations}. The function $\varsigma(u) \coloneqq 1 + (1 + e^{-20(u - 1/3)})^{-1}$ is a logistic-type function in \cite{athey2019generalized}. The random function generator $\text{RFG}(x)$ is described in Appendix~\ref{app:simulation-details}.}
\label{tbl:vcm-sim-settings}
\end{table}

\begin{table}[h!]
\centering
\begin{tabular}{cll}
\toprule
    HTE Setting & Treatment effect $\theta^*_k(x)$ & Treatment probability $\pi_k(x)$ for $W_{i,k}$ \\
\hline
    1 & $\theta^*_k(x) = \beta_{k1} x_1,~\beta_{k1} \sim \cN(0,1)$ & $\pi_k(x) = 1/K$ for all $k$. \\
    2 & $\theta^*_k(x) = \beta_{k1} x_1,~\beta_{k1} \sim \cN(0,1)$ & $\pi_k(x) = \begin{cases}
        x_1 & k = 1, \\
        \frac{1}{K-1}(1-x_1) & k = 2,\ldots,K
    \end{cases}$ \\
    & & \\
    3 & $\theta^*_k(x) = \varsigma(\beta_{k1} x_1)\varsigma(\beta_{k2} x_2)$ & $\pi_k(x) = 1/K$ for all $k$. \\
    & \hphantom{.....} for $\beta_{k1},\beta_{k2} \sim \cN(0,1)$ & \\
    4 & $\theta^*_k(x) = \varsigma({\beta}^\top_k x)$, for ${\beta}_k\sim\cN_p({\bf 0},\mathbb I)$ & $\pi_k(x) = \begin{cases}
        x_1 & k = 1, \\
        \frac{1}{K-1}(1-x_1) & k = 2,\ldots,K.
    \end{cases}$ \\
    5 & $\theta^*_k(x) = \text{RFG}(x)$ & $\pi_k(x) = \frac{\exp \left\{{\gamma}_k^\top x \right\}}{ \sum^K_{j=1} \exp \left\{{\gamma}_j^\top x \right\}}$  \\
    & & \hphantom{.....} for ${\gamma}_k\sim\cN_p({\bm 0},\mathbb I)$. \\
\bottomrule
\end{tabular}
\caption{Settings for the underlying treatment effects $\theta^*_k(\cdot)$ and treatment probabilities $\pi_k(x)$ for HTE experiments in Section~\ref{sec:simulations}. The function $\varsigma(u) \coloneqq 1 + (1 + e^{-20(u - 1/3)})^{-1}$ is a logistic-type function used in \cite{athey2019generalized}. The random function generator $\text{RFG}(x)$ is described in Appendix~\ref{app:simulation-details}.}
\label{tbl:hte-sim-settings}
\end{table}

\paragraph{Random function generator.}

The effect functions $\theta^*_k(x) = \text{RFG}(x)$ under VCM Setting 4 (in Table~\ref{tbl:vcm-sim-settings}) and HTE Setting 5 (in Table~\ref{tbl:hte-sim-settings}) follow the random function generator design of \citet{friedman2001greedy}. The idea is to measure the performance of the estimator under a variety of randomly generated targets. Each $\theta^*_k(\cdot)$ is randomly generated as a linear combination of functions $\{g_{\ell}(\cdot)\}_{\ell}^{20}$ of the form
    \begin{equation*}
        \theta^*_k(x) = \sum_{\ell=1}^{20}a_{\ell}g_{\ell}(z_{\ell}),
    \end{equation*}
    where the coefficients $\{a_{\ell}\}_{\ell=1}^{20}$ are randomly generated from a uniform distribution $a_{\ell}\sim\mathcal U([-1,1])$. Each $g_{l}(z_{l})$ is a function of a randomly selected $p_{\ell}$-size subset of the $p$-dimensional variable $x$, where the size of each subset $p_{\ell}$ is randomly chosen by $p_{\ell} = \min(\left\lfloor 1.5+r_\ell\right\rfloor ,p)$, and $r_\ell$ is generated from an exponential distribution with mean $2$, $r_\ell\sim\mathrm{Exp}(0.5)$. Each $g(z_\ell)$ uses a $p_\ell$-sized random subset $z_\ell \in \bbR^{p_\ell}$ of the $p$-dimensional input $x \in \bbR^p$:
    \begin{equation*}
        z_{\ell} \coloneqq \left( x_{\phi_\ell(1)}, \ldots, x_{\phi_\ell(p_\ell)} \right) \in \bbR^{p_\ell},
    \end{equation*}
    such that $\{\phi_\ell(1),\ldots,\phi_\ell(p_\ell)\}$ is a length-$p_\ell$ permutation of indices drawn from $\{1,\ldots, p\}$, without replacement. The functions $g_{\ell}(\cdot)$ are Gaussian functions of the $p_\ell$ sampled variables:
    \begin{equation*}
    g_{\ell}(z_{\ell}) \coloneqq \exp\left\{-\frac{1}{2}(z_{\ell}-\mu_{\ell})^\top \mathbf{V}_{\ell}(z_{\ell}-\mu_{\ell})\right\},
    \end{equation*}
    where the mean vector $\mu_\ell \in \bbR^{p_\ell}$ is randomly generated from a standard multivariate Gaussian, $\mu_\ell \sim \mathcal N_{p_\ell}({\bf 0}, \bbI)$. The $p_{l}\times p_{l}$ covariance matrix $\mathbf{V}_{l}$ are formed through the spectral decomposition:
    \begin{equation*}
        \mathbf{V}_{\ell} = \mathbf{U}_{\ell}\mathbf{D}_{\ell}\mathbf{U}_{\ell}^\top,
    \end{equation*}
    where $\mathbf{U}_{\ell}$ is a random $p_\ell\times p_\ell$ orthonormal matrix and $\mathbf{D}_{\ell} \coloneqq \text{diag}(d_{1,\ell}, \ldots ,d_{p_{\ell},\ell})$ with diagonal entries $d_{j,\ell}$ generated from a uniform distribution according to $\sqrt{d_{j,\ell}}\sim\mathcal U(0.1,2.0)$.

\section{Additional Simulations}

\subsection{Settings for the criterion value experiment in Section~\ref{sec:pseudo-outcomes}}\label{app:experiments-fp-tree-split-invariance}

The criterion value experiment in Section~\ref{sec:pseudo-outcomes} was run under a varying coefficient model of the form
\begin{equation}\label{eqn:criteria-comparison-model}
    Y_i \coloneqq W_i^\top \theta^*(X_i) + \epsilon_i, \qquad \epsilon_i \sim \mathcal N(0,0.5^2),
\end{equation}
where the regressors $W_i$ were generated as bivariate standard Gaussian samples $W_i \sim \mathcal N_2({\bf 0},\bbI)$ and the auxiliary covariates were generated as standard uniform samples $X_i \sim \cU(0,1)$. The data-generating coefficient functions were $\theta^*(x) \coloneqq (\sin(2\pi x),\; x)$ and the criterion values were computed based on $n = 1000$ samples following \eqref{eqn:criteria-comparison-model}.

\subsection{Supporting experiments for Section~\ref{sec:simulations}}\label{app:supporting-experiments-sims}

\paragraph{Multicollinearity in auxiliary covariates.} 

We conducted a VCM experiment with highly correlated auxiliary covariate features. We ran a modified version of VCM Setting 3 by generating auxiliary covariates as $X_i \sim \mathcal N({\bf 0}, \Sigma)$, where $[\Sigma]_{j,k} = \omega^{|j-k|}$ for $\omega \in \{0, 0.5, 0.9\}$. Table~\ref{tbl:multicollinearity-in-X} provides a clear summary of the computational performance of GRF-$\FPT$ relative to GRF-$\grad$ and statistical accuracy (MSE -- multiplied by 100 for readability). All experiments were run over a forest of 10 trees and MSE estimates were computed over 50 replications of the model and evaluated on a separate set of $n = 5,000$ samples, carried out using GRF-$\FPTtwo$ and GRF-$\grad$. These results demonstrate clearly that GRF-$\FPT$ remains robust, stable, and computationally efficient, even under high multicollinearity in $X_i$.

\begin{table}[h!]
\centering
\begin{tabular}{llllccc}
    \toprule
    $\dim(\cX)$ & $n$ & $K$ & $\omega$ & Speedup & $100\times \text{MSE}$ $\grad$ & $100\times \text{MSE}$ $\FPTtwo$ \\
    \hline
     5 & 10,000 & 64 & 0.00 & 2.55 & 16.60 & 16.83 \\
     5 & 10,000 & 64 & 0.50 & 2.53 & 15.48 & 15.47 \\
     5 & 10,000 & 64 & 0.90 & 2.35 & 10.95 & 11.09 \\
    \bottomrule
\end{tabular}
\caption{Effect of multicollinearity in the auxiliary covariates $X_i$ on the relative computational gain of GRF-$\FPTtwo$, as well as the statistical accuracy of both GRF-$\FPT$ and GRF-$\grad$ estimators.}
\label{tbl:multicollinearity-in-X}
\end{table}

\paragraph{Subsampling ratio.}

We carried out an experiment to show that the subsample proportion does not affect the computational advantage or statistical accuracy of GRF-$\FPT$ relative to GRF-$\grad$. We varied the subsampling ratio $s/n \in \{0.25, 0.50, 0.75\}$ under VCM Setting 3 over a forest of 10 trees carried out using GRF-$\FPTtwo$ and GRF-$\grad$. Table~\ref{tbl:subsample-ratio} summarizes our results, averaged over 50 replications of the model, with a test set of 5,000 samples. These results show clearly that the statistical accuracy (MSE) of GRF-$\FPTtwo$ relative to GRF-$\grad$ does not depend strongly on the subsample ratio.

\begin{table}[h!]
\centering
\begin{tabular}{llllccc}
    \toprule
    $\dim(\cX)$ & $n$ & $K$ & $s/n$ & Speedup & $100\times \text{MSE}$ $\grad$ & $100\times \text{MSE}$ $\FPTtwo$ \\
    \hline
     2 & 10,000 & 64 & 0.25 & 2.77 & 2.86 & 2.90 \\
     2 & 10,000 & 64 & 0.50 & 3.10 & 2.91 & 2.90 \\
     2 & 10,000 & 64 & 0.75 & 2.98 & 3.21 & 3.19 \\
    \bottomrule
\end{tabular}
\caption{Effect of the subsampling ratio $s/n$ on the relative computational gain of GRF-$\FPTtwo$, as well as the statistical accuracy of both GRF-$\FPT$ and GRF-$\grad$ estimators.}
\label{tbl:subsample-ratio}
\end{table}

\paragraph{Large sample size.}

We ran additional experiments to clearly show how our method scales for very large datasets. Using a forest of 10 trees, we tested our method on VCM Setting 3 with 
sample sizes up to $n = 500,000$, carried out using GRF-$\FPTtwo$ and GRF-$\grad$. The results are summarized in Table~\ref{tbl:large-sample} and demonstrate that, even as the dataset grows very large, our method consistently remains faster than GRF-$\grad$. While the relative speedup slightly decreases at first, it stabilizes towards a consistent advantage as grows $n$ grows sufficiently large, suggesting that the advantage is not bottlenecked by $n$ and maintains a robust advantage at scale.

\begin{table}[h!]
\centering
\begin{tabular}{lllc}
    \toprule
    $\dim(\cX)$ & $K$ & $n$ & Speedup \\
    \hline
     2 & 256 & 10,000 & 4.54 \\
     2 & 256 & 20,000 & 3.59 \\
     2 & 256 & 50,000 & 3.49 \\
     2 & 256 & 100,000 & 3.11 \\
     2 & 256 & 200,000 & 3.04 \\
     2 & 256 & 500,000 & 3.08 \\
    \bottomrule
\end{tabular}
\caption{Effect of increasing sample sizes $n$ on the relative computational gain of GRF-$\FPTtwo$.}
\label{tbl:large-sample}
\end{table}

\subsection{Supporting figures for Section~\ref{sec:simulations}}\label{app:simulation-figures}

\subsubsection{VCM experiments}

\paragraph{Large $n$ VCM simulations.}

Figure~\ref{fig:timing-vcm} illustrates the absolute fit times for the GRF-$\FPT$ algorithms under the four VCM settings for $\theta^*_k(x)$ described in Table~\ref{tbl:vcm-sim-settings} over the large-$n$ settings in Table~\ref{tbl:sim_params}. Across all settings and all dimensions, GRF-$\FPT$ is consistently several factors faster than GRF-$\grad$. The speedup factor is summarized in Figure~\ref{fig:timing-ratio-vcm}, which illustrates the relative speedup of GRF-$\FPT$, calculated as the ratio of GRF-$\grad$ fit times over GRF-$\FPT$ fit times. Consistent with the observations in Section~\ref{sec:applications}, we find that the speed advantage of GRF-$\FPT$ increases as the dimension of the target increases. 

Figure~\ref{fig:forest-mse-vcm} shows that this speed advantage comes while performing comparably to GRF-$\grad$ in terms of statistical accuracy. Across all settings for VCMs with $K = 4$ dimensional targets, the MSE estimates from GRF-$\FPT$ is highly similar to the MSE estimates of GRF-$\grad$, while for $K = 256$ dimensional targets one sees more variation in MSE estimates across the methods. This effect likely reflects the increased variance associated with high-dimensional estimation. In some cases we see GRF-$\FPTone$ slightly outperform both GRF-$\FPTtwo$ and GRF-$\grad$, in other cases we see GRF-$\grad$ slightly outperform both GRF-$\FPT$ methods, and in others GRF-$\FPTtwo$ yields the lowest MSE. One sees that these differences are typically small. The key benefit we emphasize is that GRF-$\FPT$ is able to achieve nearly identical statistical accuracy with a substantial improvement in computational speed.

\paragraph{Small $n$ VCM simulations.}

Figures~\ref{fig:timing-vcm-small} and~\ref{fig:timing-ratio-vcm-small} illustrate the absolute fit times and relative speed advantage, respectively, of GRF-$\FPT$ under the VCM design of $\theta_k^*(x)$ over the small-$n$ settings. One sees that even when $n$ is more modest, GRF-$\FPT$ consistently offers a computational advantage over GRF-$\grad$, with possible outliers under VCM Setting 2 at $K = 4$. We believe this negative relative advantage to be caused by random fluctuations in computation and are not representative of the $\FPT$ algorithm itself, particularly in light of the fact that the negative effect vanishes when the number of trees increases from 100 to 500. As one would expect based on the large-$n$ results, the relative advantage tends to increase with increasing $K$, and generally stabilizes with increasing $n$. Figure~\ref{fig:forest-mse-vcm-small} shows that the GRF-$\FPT$ speed advantage does not come at any material cost in statistical accuracy, with similar performance to GRF-$\grad$ across all settings.

\begin{figure}
    \centering
    \includegraphics[width=0.49\linewidth]{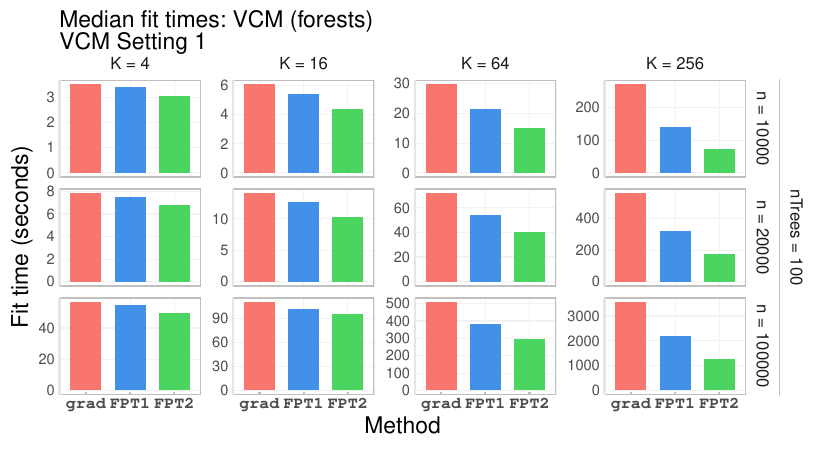}
    \includegraphics[width=0.49\linewidth]{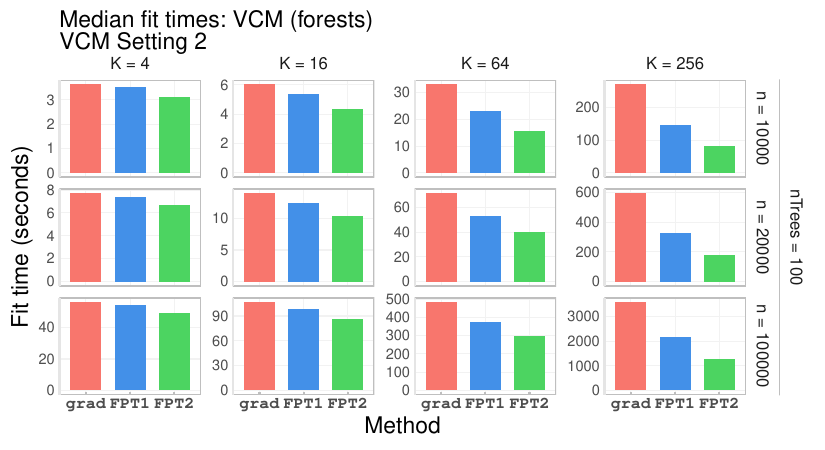}

    \includegraphics[width=0.49\linewidth]{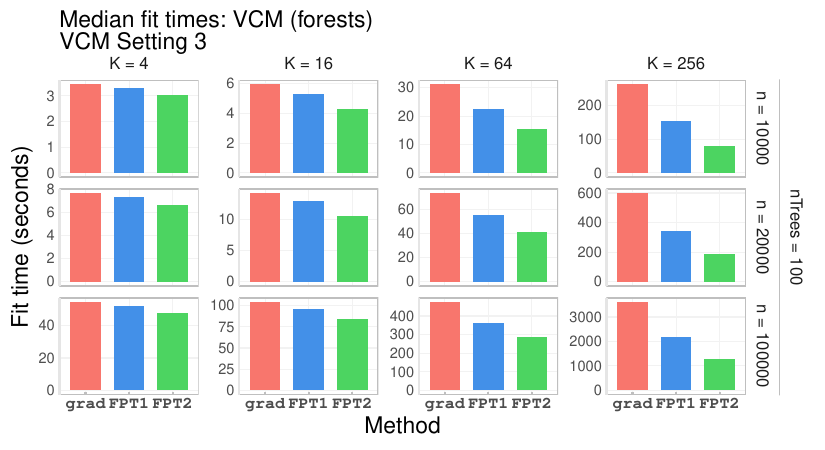}
    \includegraphics[width=0.49\linewidth]{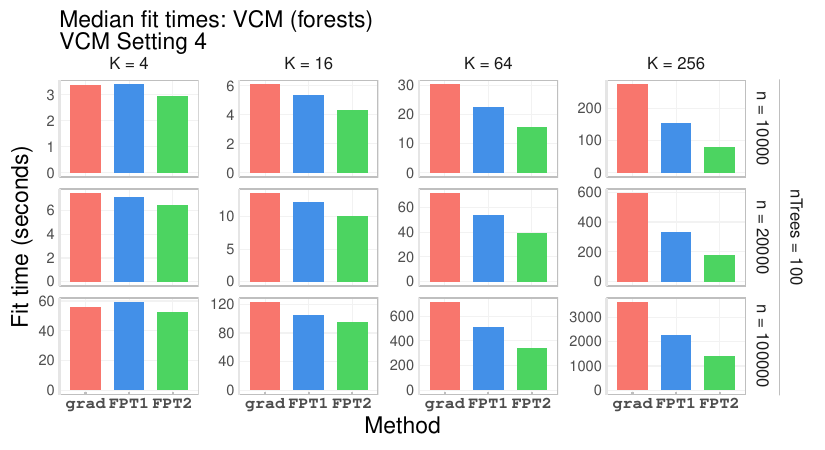}
    \caption{Absolute fit times for VCM timing experiments under the settings in  Table~\ref{tbl:vcm-sim-settings} and large-$n$ settings in Table~\ref{tbl:sim_params}.}\label{fig:timing-vcm}
\end{figure}

\begin{figure}
    \centering
    \includegraphics[width=0.95\linewidth]{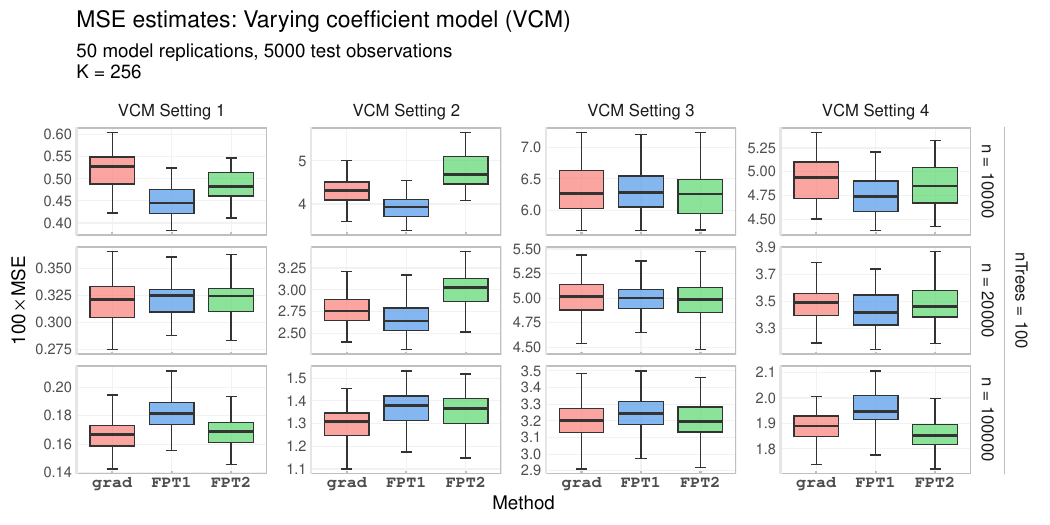}
    \includegraphics[width=0.95\linewidth]{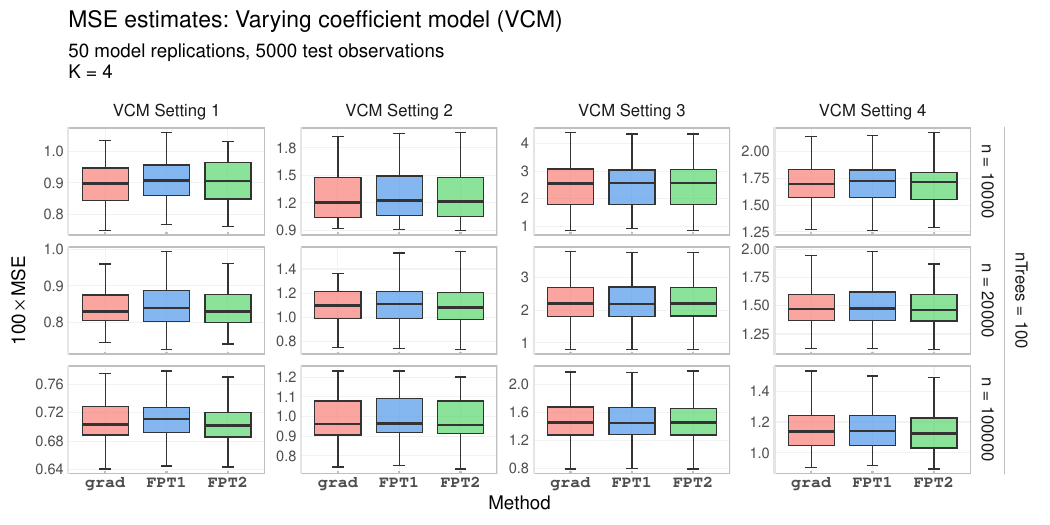}
    \caption{Estimates of MSE $\bbE[\|\theta^*(X) - \hat\theta(X)/K\|^2_2]$ for VCM for $K = 256$ dimensional (top) and $K = 4$ dimensional targets (bottom) under the large-$n$ settings in Table~\ref{tbl:sim_params}.}\label{fig:forest-mse-vcm}
\end{figure}

\begin{figure}
    \centering
    \includegraphics[width=0.7\linewidth]{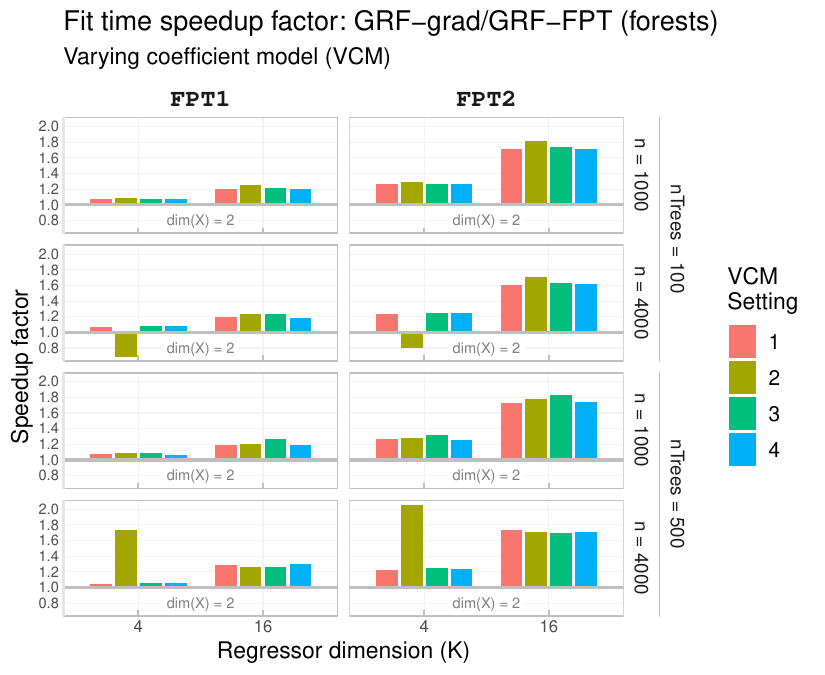}
    \caption{Speedup factor for GRF-$\FPT$ in comparison to GRF-$\grad$ for VCM timing experiments under the small-$n$ settings in Table~\ref{tbl:sim_params}.}\label{fig:timing-ratio-vcm-small}
\end{figure}

\begin{figure}
    \centering
    \includegraphics[width=0.4\linewidth]{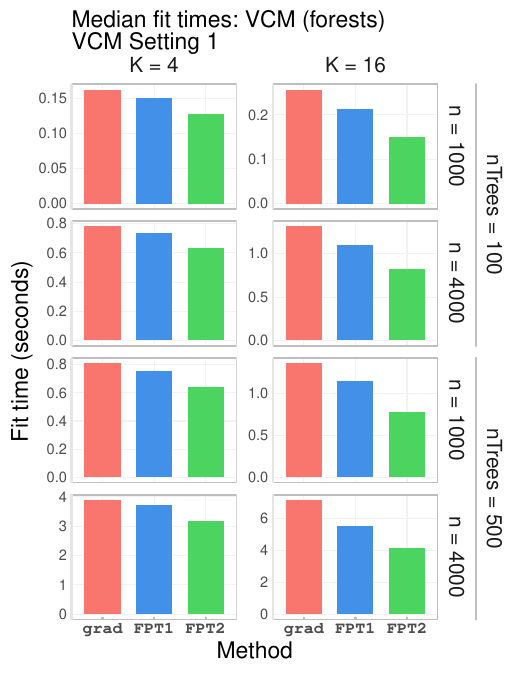}
    \includegraphics[width=0.4\linewidth]{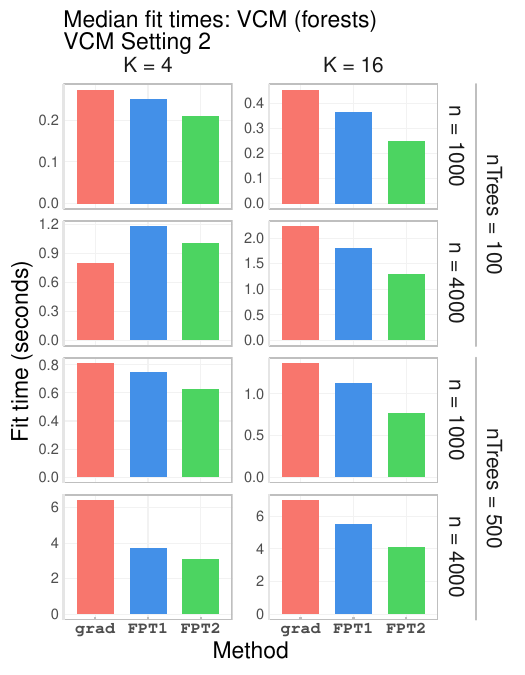}
    
    \includegraphics[width=0.4\linewidth]{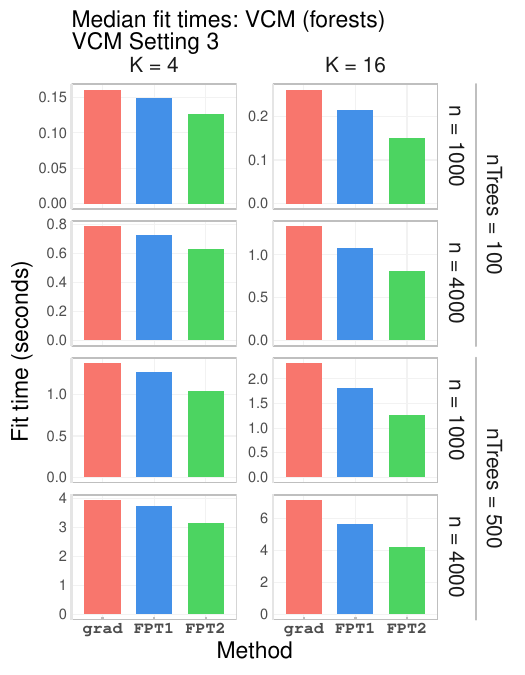}
    \includegraphics[width=0.4\linewidth]{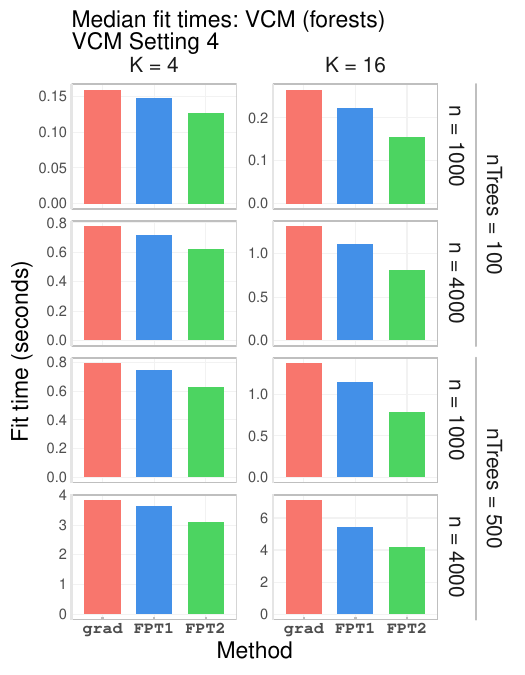}
    \caption{Absolute fit times for VCM timing experiments under the settings in  Table~\ref{tbl:vcm-sim-settings} and small-$n$ settings in Table~\ref{tbl:sim_params}.}\label{fig:timing-vcm-small}
\end{figure}

\begin{figure}
    \centering
    \includegraphics[width=0.9\linewidth]{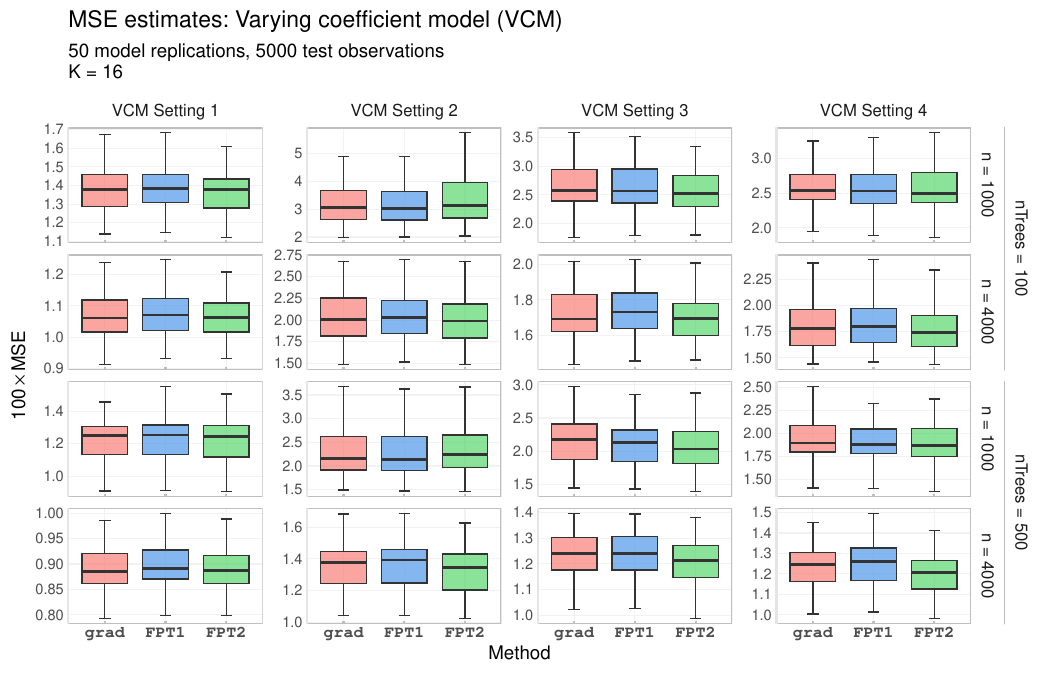}
    \includegraphics[width=0.9\linewidth]{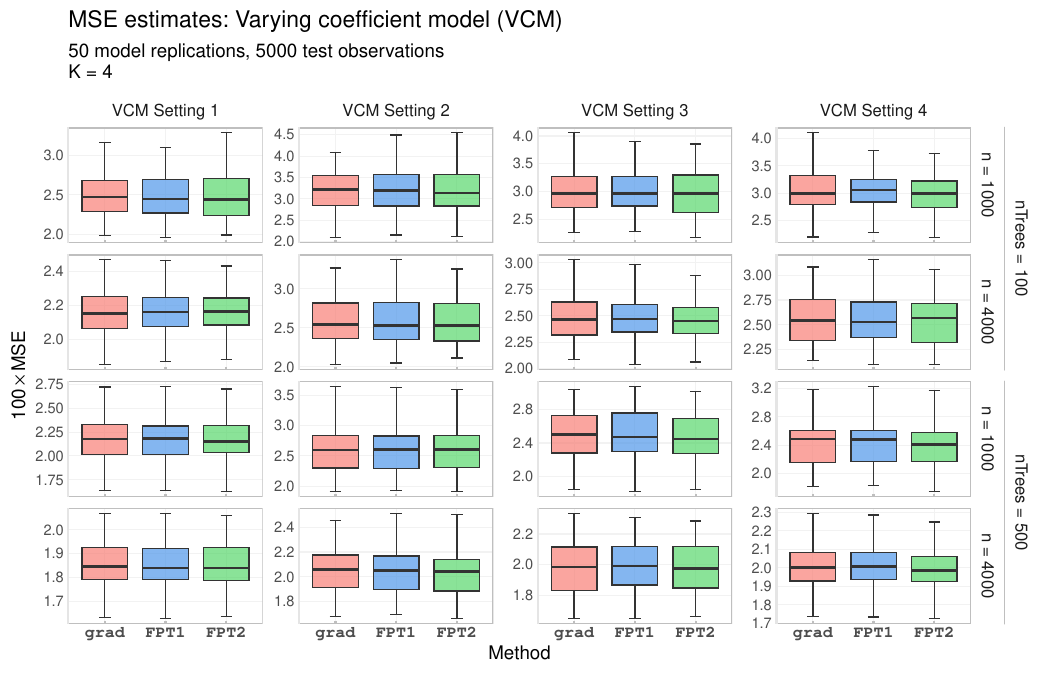}
    \caption{Estimates of MSE $\bbE[\|\theta^*(X) - \hat\theta(X)/K\|^2_2]$ for VCM for $K = 16$ dimensional (top) and $K = 4$ dimensional targets (bottom) under the small-$n$ settings in Table~\ref{tbl:sim_params}.}\label{fig:forest-mse-vcm-small}
\end{figure}

\subsection{HTE experiments}

\paragraph{Large $n$ HTE simulations.}

Figure~\ref{fig:timing-hte} illustrates the absolute fit times for the GRF-$\FPT$ algorithm under the five HTE settings of $\theta^*_k(x)$ and $\pi_k(x)$ described in Table~\ref{tbl:hte-sim-settings} over the large-$n$ settings in Table~\ref{tbl:sim_params}. We find that GRF-$\FPT$ is consistently faster than GRF-$\grad$. The speedup factor of GRF-$\FPT$ relative to GRF-$\grad$ is summarized in Figure~\ref{fig:timing-ratio-hte}, calculated as the ratio of GRF-$\grad$ fit times over GRF-$\FPT$ fit times. As was seen for VCM experiments, the speed advantage of GRF-$\FPT$ scales with the dimensionality $K$ of the target. One sees from both Figures~\ref{fig:timing-ratio-hte} and~\ref{fig:timing-hte} that GRF-$\FPT$'s computational advantage is less dramatic than under the VCM experiments. This can be understood based on the fact that the VCM regressors $W_i$ are continuous while the HTE regressors represent binary indicators. Continuous regressors provide more granularity when fitting the child statistics $\tilde\theta_{C_j}$, and as a result provide a larger set of candidate splits over the covariates. Nevertheless, one sees in Figure~\ref{fig:timing-ratio-hte} that the $\FPT$ splitting mechanism is still up to 1.5$\times$ faster under the largest regressor setting $K = 256$, with a more modest, but persistent savings across all settings. 

The statistical benchmarks for our HTE experiments are shown in Figure~\ref{fig:forest-mse-hte}. Consistent with the VCM experiments, one sees that the computational advantage of GRF-$\FPT$ does not come at the cost of in terms of its statistical accuracy.

\paragraph{Small $n$ HTE simulations.}

Figures~\ref{fig:timing-ratio-hte-small} and~\ref{fig:timing-hte-small} summarize the relative speed advantage and absolute fit times for the GRF-$\FPT$ algorithm under the small-$n$ HTE design. Consistent with the large-$n$ HTE experiments the $\FPTtwo$ mechanism sees a stable computational advantage across all settings, with an increasing effect in increasing $K$, while the $\FPTone$ mechanism displays a persistent advantage for $K = 16$ and comparable computational performance for $K = 4$. The more modest relative advantage for the small-$n$ experiments is itself consistent with the VCM small-$n$ experiments, owing in large part due to the smaller values of $K$. Figure~\ref{fig:forest-mse-hte-small} compares the statistical performance of GRF-$\FPT$ to GRF-$\grad$, with no material difference between either GRF-$\FPTone$, GRF-$\FPTtwo$, or GRF-$\grad$'s estimation accuracy.

\begin{figure}[h!]
    \centering
    \includegraphics[width=0.75\linewidth]{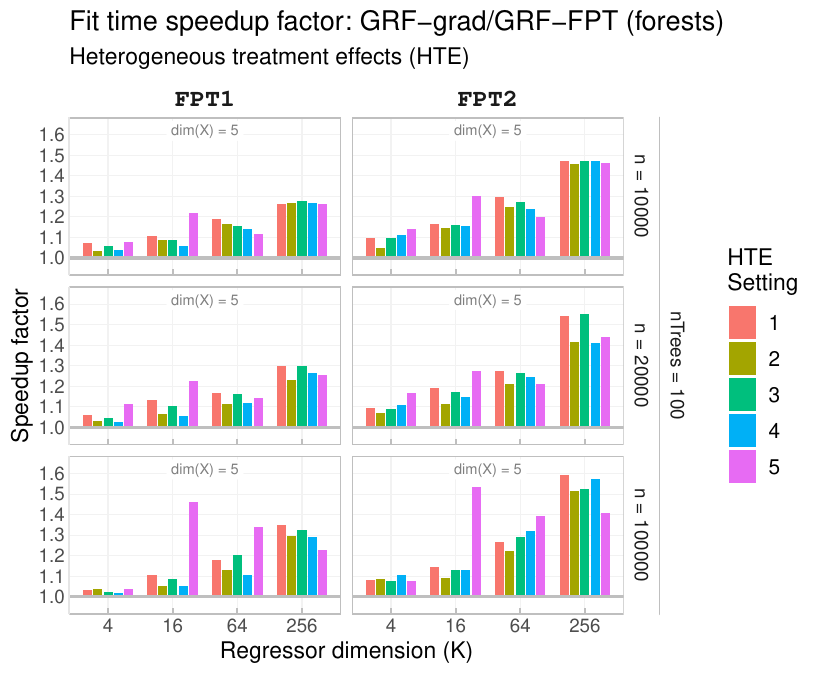}
    \caption{Speedup factor for GRF-$\FPT$ in comparison to GRF-$\grad$ for HTE timing experiments under the large-$n$ setting in Table~\ref{tbl:sim_params}.}\label{fig:timing-ratio-hte}
\end{figure}

\begin{figure}[h!]
    \centering
    \includegraphics[width=0.49\linewidth]{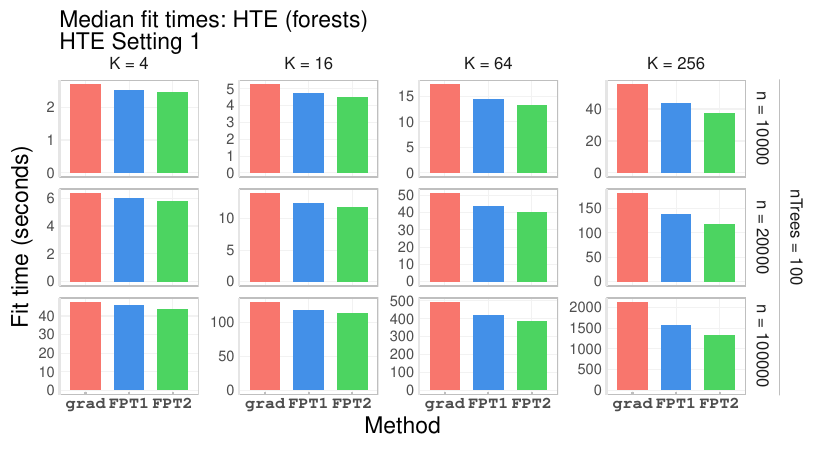}
    \includegraphics[width=0.49\linewidth]{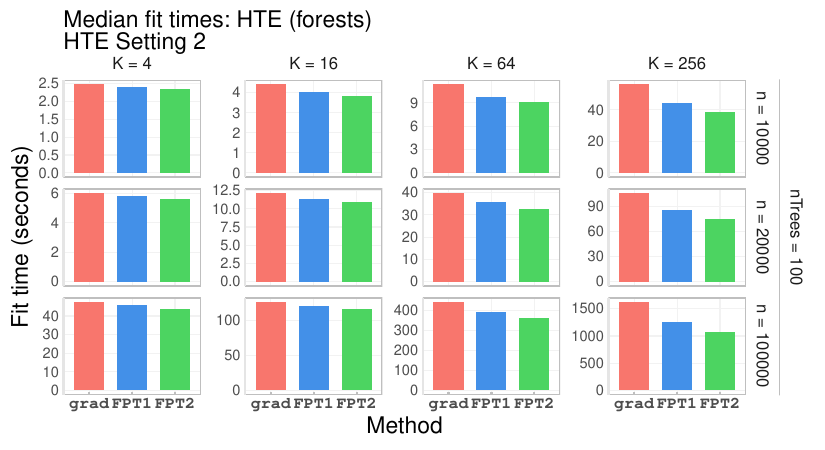}

    \includegraphics[width=0.49\linewidth]{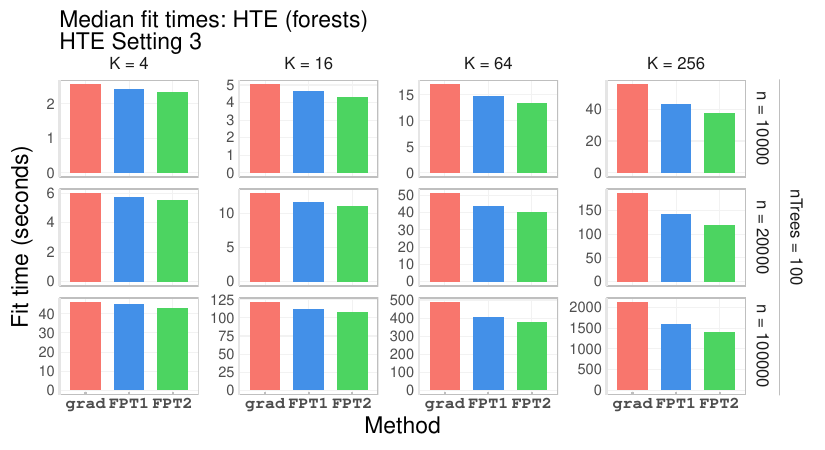}
    \includegraphics[width=0.49\linewidth]{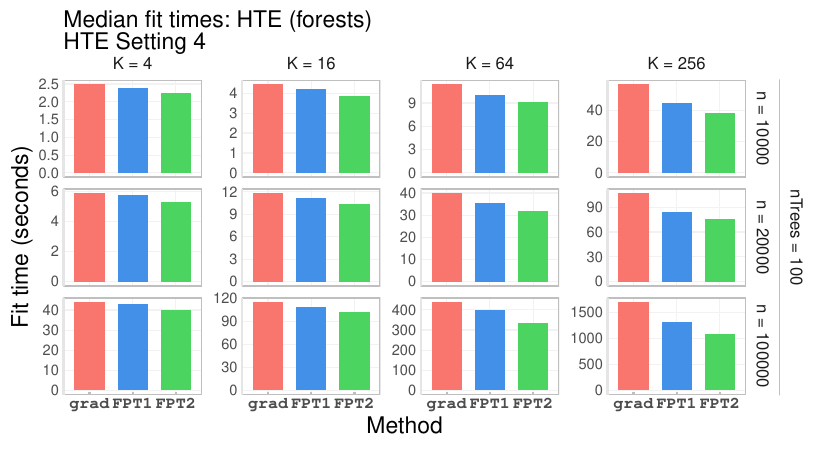}

    \includegraphics[width=0.49\linewidth]{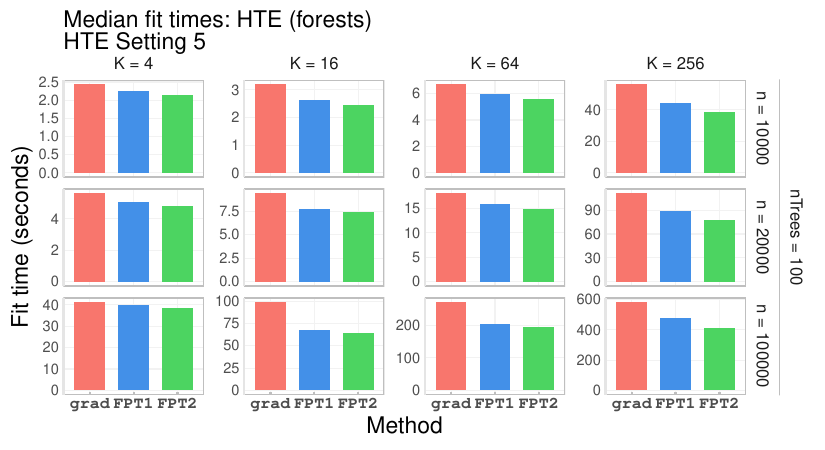}
    \caption{Absolute fit times for HTE timing experiments under the settings in  Table~\ref{tbl:hte-sim-settings} and large-$n$ settings in Table~\ref{tbl:sim_params}.}\label{fig:timing-hte}
\end{figure}

\begin{figure}[h!]
    \centering
    \includegraphics[width=0.9\linewidth]{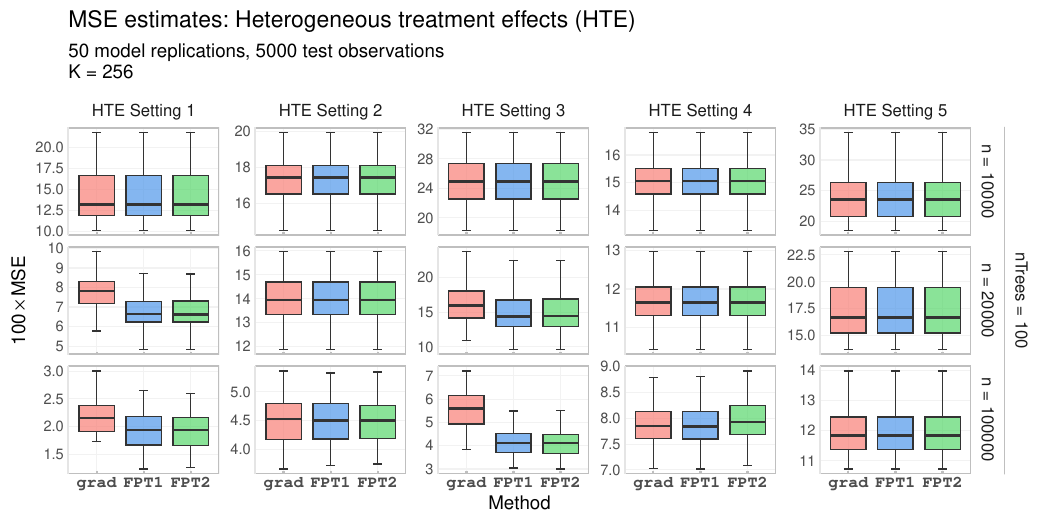}
    \includegraphics[width=0.9\linewidth]{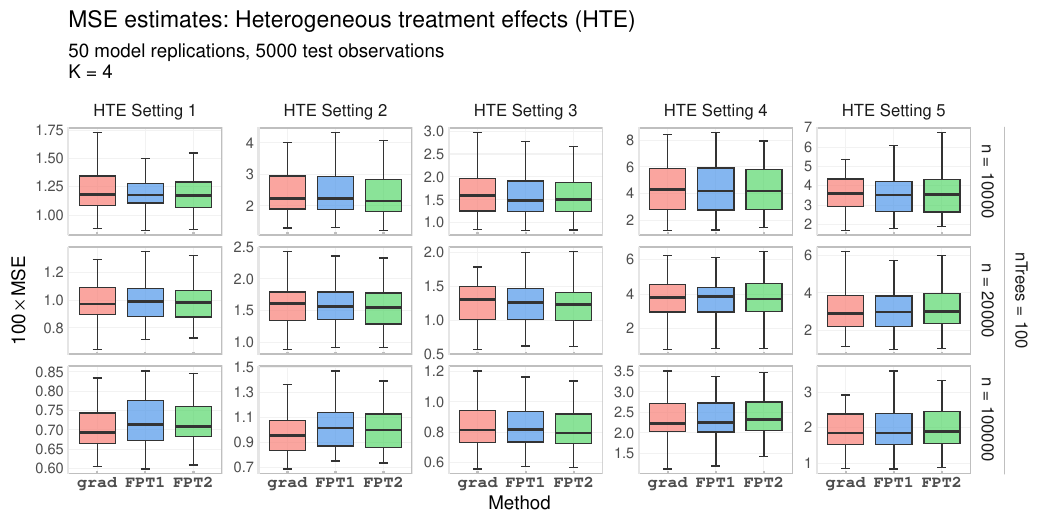}
    \caption{Estimates of MSE $\bbE[\|\theta^*(X) - \hat\theta(X)/K\|^2_2]$ for HTE for $K = 256$ dimensional (top) and $K = 4$ dimensional targets (bottom) under the large-$n$ settings in Table~\ref{tbl:sim_params}.}\label{fig:forest-mse-hte}
\end{figure}

\begin{figure}[h!]
    \centering
    \includegraphics[width=0.75\linewidth]{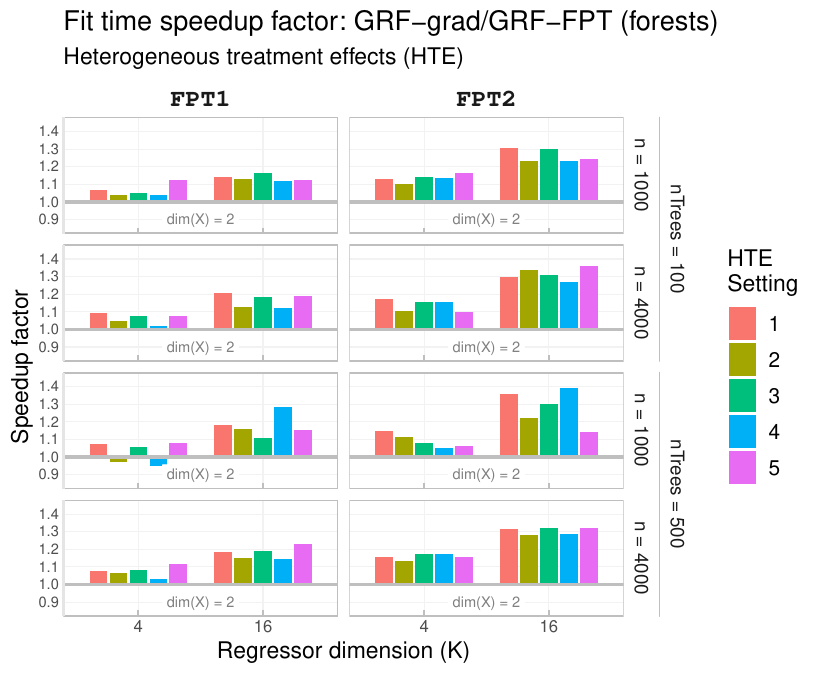}
    \caption{Speedup factor for GRF-$\FPT$ in comparison to GRF-$\grad$ for HTE timing experiments under the small-$n$ settings in Table~\ref{tbl:sim_params}.}\label{fig:timing-ratio-hte-small}
\end{figure}

\begin{figure}[h!]
    \centering
    \includegraphics[width=0.34\linewidth]{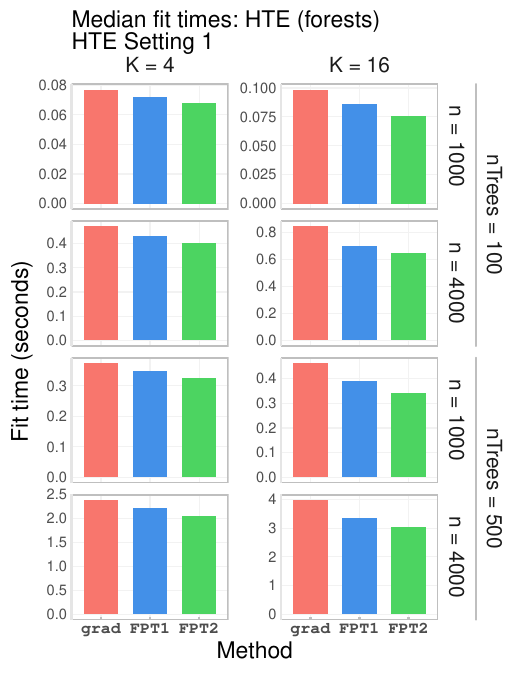}
    \includegraphics[width=0.34\linewidth]{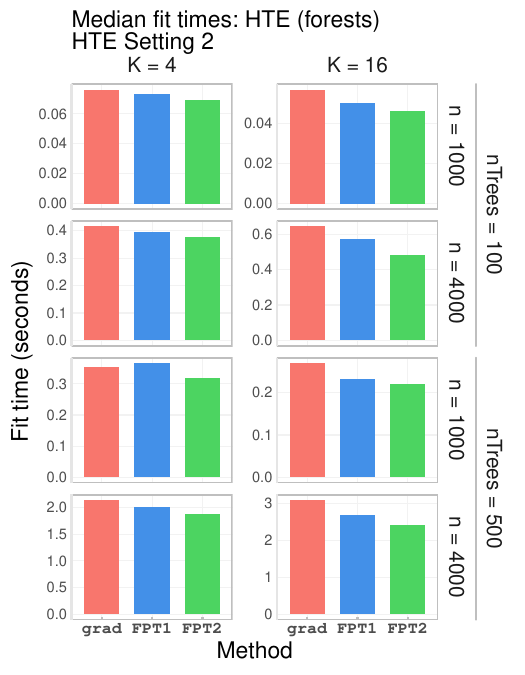}

    \includegraphics[width=0.34\linewidth]{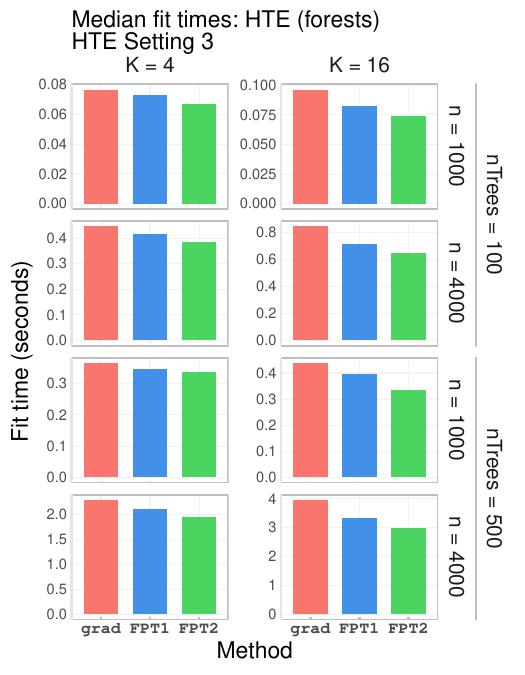}
    \includegraphics[width=0.34\linewidth]{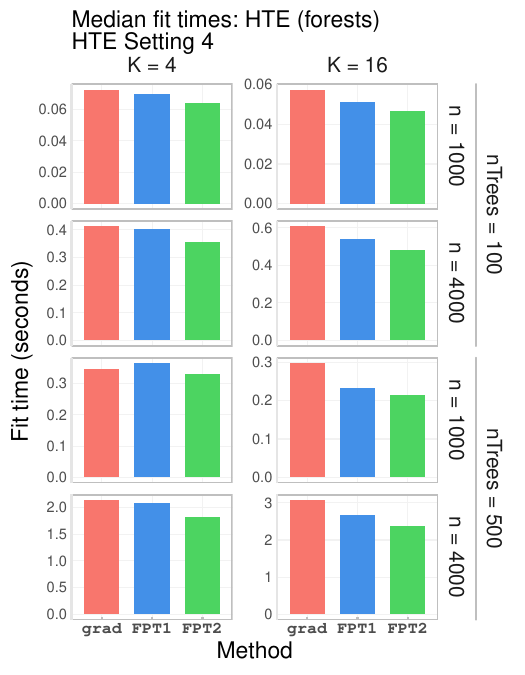}

    \includegraphics[width=0.34\linewidth]{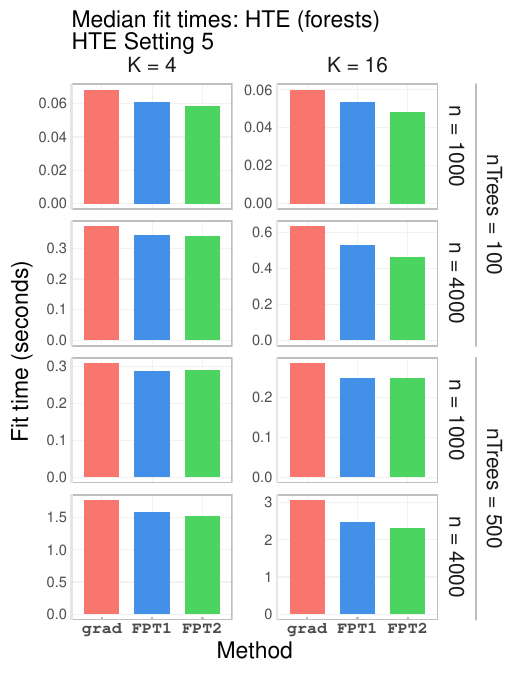}
    \caption{Absolute fit times for HTE timing experiments under the settings in  Table~\ref{tbl:hte-sim-settings} and small-$n$ settings in Table~\ref{tbl:sim_params}.}\label{fig:timing-hte-small}
\end{figure}

\begin{figure}[h!]
    \centering
    \includegraphics[width=0.9\linewidth]{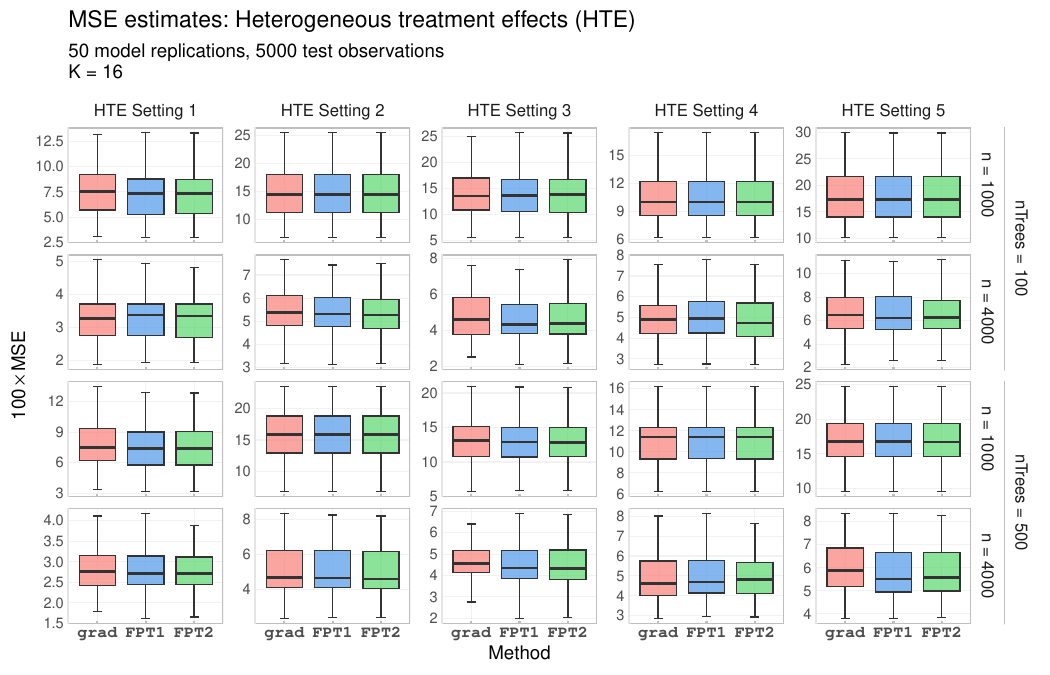}
    \includegraphics[width=0.9\linewidth]{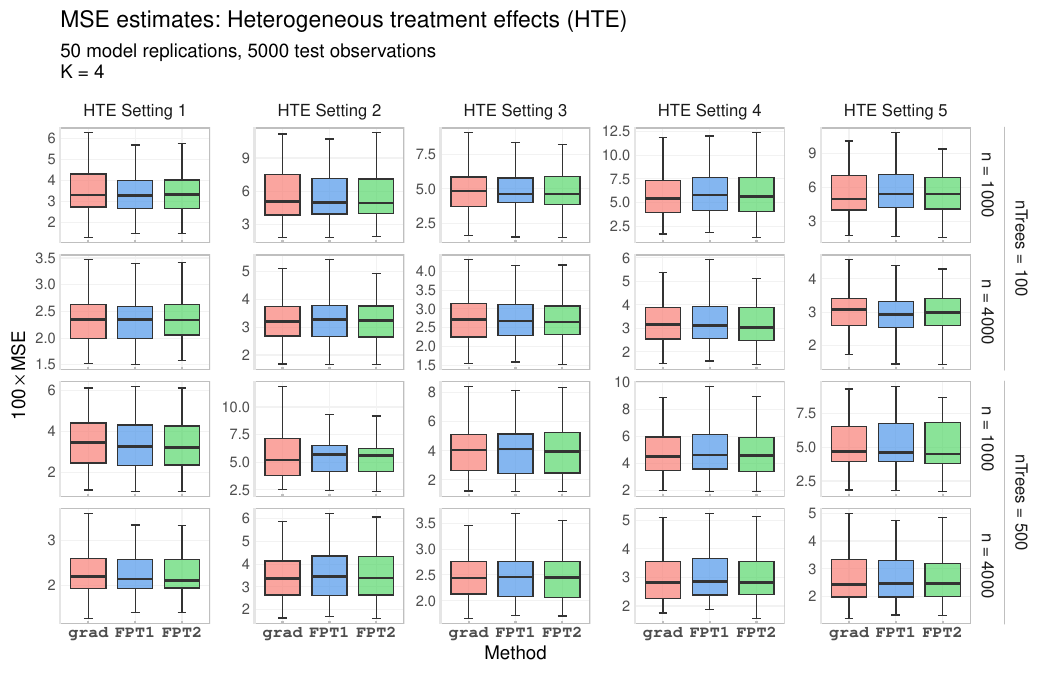}
    \caption{Estimates of MSE $\bbE[\|\theta^*(X) - \hat\theta(X)/K\|^2_2]$ for HTE for $K = 16$ dimensional (top) and $K = 4$ dimensional targets (bottom) under the small-$n$ settings in Table~\ref{tbl:sim_params}.}\label{fig:forest-mse-hte-small}
\end{figure}

\section{Additional Examples}\label{app:additional-examples}

\subsection{Pseudo-outcomes for nonparametric regression}\label{app:pseudo-outcomes-nonparam-regression}

Consider the task of estimating the conditional mean function $\theta^*(x) \coloneqq \bbE[Y|X=x]$. The target $\theta^*(x)$ is identified by a moment condition of the form \eqref{eqn:est-eqn} with scoring function $\psi_{\theta}(Y_i) \coloneqq Y_i - \theta$, the residual associated with using $\theta$ as the local estimate with respect to the $i$-th sample. The local solution $\hat\theta_P$ over $P$ is the mean observed response over the parent,
\begin{equation*}
    \hat\theta_P = \overline Y_P \coloneqq \frac{1}{n_P}\sum_{\{i:X_i\in P\}} Y_i.
\end{equation*}
The fixed-point pseudo-outcomes are simple the (negative) residuals that result from fitting \eqref{eqn:local-est-P} over $P$
\begin{equation*}
    \rho^\FPT_i = -(Y_i - \hat\theta_P) = -(Y_i - \overline Y_P).
\end{equation*}
The gradient of the score function is $\nabla_\theta\psi_{\theta}(y) = -1$, and hence $A_P = -1$. Therefore, up to a constant factor, the gradient-based pseudo-outcomes $\rho_i^\grad$ for conditional mean estimation reduce to their fixed-point counterparts $\rho_i^\FPT$,
\begin{equation*}
    \rho_i^\grad = -A_P^{-1}\psi_{\hat\theta_P}(Y_i) = Y_i - \overline Y_P = -\rho_i^\FPT.
\end{equation*}

In this special case, we recover the conventional splitting algorithm used for univariate responses \citep{breiman1984classification, breiman2001random} or for multivariate responses \citep{de2002multivariate, segal1992tree}. Trees grown using $\rho_i^\FPT$, $\rho_i^\grad$, or $Y_i$ will be identical to one another because CART splits are scale and translation invariant with respect to the response.

More generally, for targets $\theta^*(x)$ beyond the conditional mean, the form of $\rho_i^\grad$ will be equivalent to $\rho_i^\FPT$ whenever the target function $\theta^*:\cX\to\Theta$ is a map from the input space $\cX$ a one-dimensional parameter space $\Theta$.

\section{Real Data Comparison: California Housing}\label{app:california-data}

\paragraph{Data.} The California housing data appeared in \citet{kelley1997sparse} and can be directly obtained from the Carnegie Mellon StatLib repository (\url{https://lib.stat.cmu.edu/datasets/houses.zip}). 
The data includes 20640 observations, where each observation corresponds to measurements over an individual census block group in California taken from the 1990 census. A census block is the smallest geographical area for which the U.S. Census Bureau publishes sample data, typically with a population between 600 and 3000 people per block. Each observation from the California housing data set contains measurements of 9 variables: median housing value (dollars), longitude, latitude, median housing age (years), total rooms (count, aggregated over the census block), total bedrooms (count, aggregated over the census block), population (count), households (count), median income (dollars).

\paragraph{Model.} We consider a varying coefficient model of the form
\begin{equation}\label{eqn:california-housing-model}
    Y_i = \nu^*(X_i) + \theta^*_1(X_i) W_{i,1} + \cdots + \theta^*_6(X_i) W_{i,6} + \epsilon_i
\end{equation}
where we suppose that our effects are local to spatial coordinates $x := (\texttt{latitude}_i,\texttt{longitude}_i)$, $Y_i$ denotes the log median housing value of the census block, and the primary regressors $W_i = (W_{i,1},\ldots,W_{i,6})$ passed to the model were median housing age, log(total rooms), log(total bedrooms), log(population), log(households), and log(median income). Here, each $\theta^*_k(x)$ denotes the geographically-varying effect of the corresponding regressor $W_{i,k}$, for $k = 1,\ldots,6$. The empirical distribution of the transformed regressors passed to each of the GRF models is seen in Figure~\ref{fig:california-housing-regressors}.

\paragraph{Algorithms.} We target GRF estimates $\hat\theta(x) = (\hat\theta_1(x),\ldots,\hat\theta_6(x))^\top$ of $\theta^*(x) = (\theta^*_1(x),\ldots,\theta^*_6(x))^\top$ based on the GRF-$\FPTone$ and GRF-$\FPTtwo$ algorithms described in Section~\ref{sec:simulations}, and compare those to GRF-$\grad$. All forests were fit using the \texttt{grf::lm\_forest} function, which trains the Stage I forest and optionally solves for the Stage II estimates $\hat\theta(x)$ for varying coefficient models \eqref{eqn:california-housing-model}. All versions fit a forest of 2000 trees, the default settings of the original \texttt{R} implementation \citep{grfpackage}, a subsample ratio of $0.5$, and a target minimum node size of 5 observations.

\begin{table}[h!]
\centering
\begin{tabular}{lcc}
    \toprule
    Algorithm & Training time (sec.) & Speedup factor \\
    \hline
    GRF-$\grad$ & 19.1 & \\
    GRF-$\FPTone$ & 15.4 & 1.24\\
    GRF-$\FPTtwo$ & 12.6 & 1.52 \\
    \bottomrule
\end{tabular}
\caption{Fit times to train a forest of 2000 trees on the California housing data.}
\label{tbl:california-housing-times}
\end{table}

\paragraph{Results.} Table~\ref{tbl:california-housing-times} summarizes the computational benefit of GRF-$\FPT$ applied to the California housing data. Figures~\ref{fig:california-housing-fpt2} illustrates the local estimates $\hat\theta(x)$ made by GRF-$\FPTtwo$, while Figure~\ref{fig:california-housing-grad-fpt1} illustrates the fits under GRF-$\FPTone$ and GRF-$\grad$.

\begin{figure}
    \centering
    \makebox[\textwidth][c]{\includegraphics[width=0.65\textwidth]{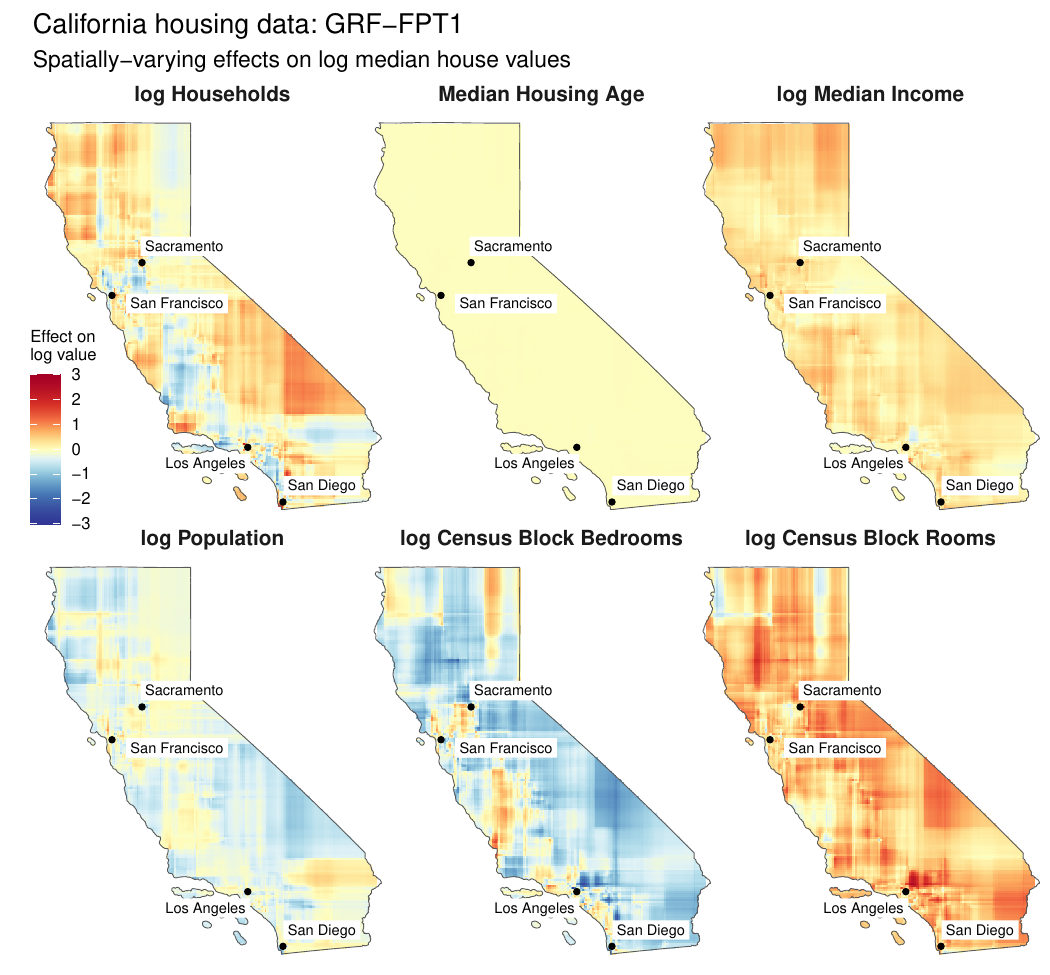}}
    \makebox[\textwidth][c]{\includegraphics[width=0.65\textwidth]{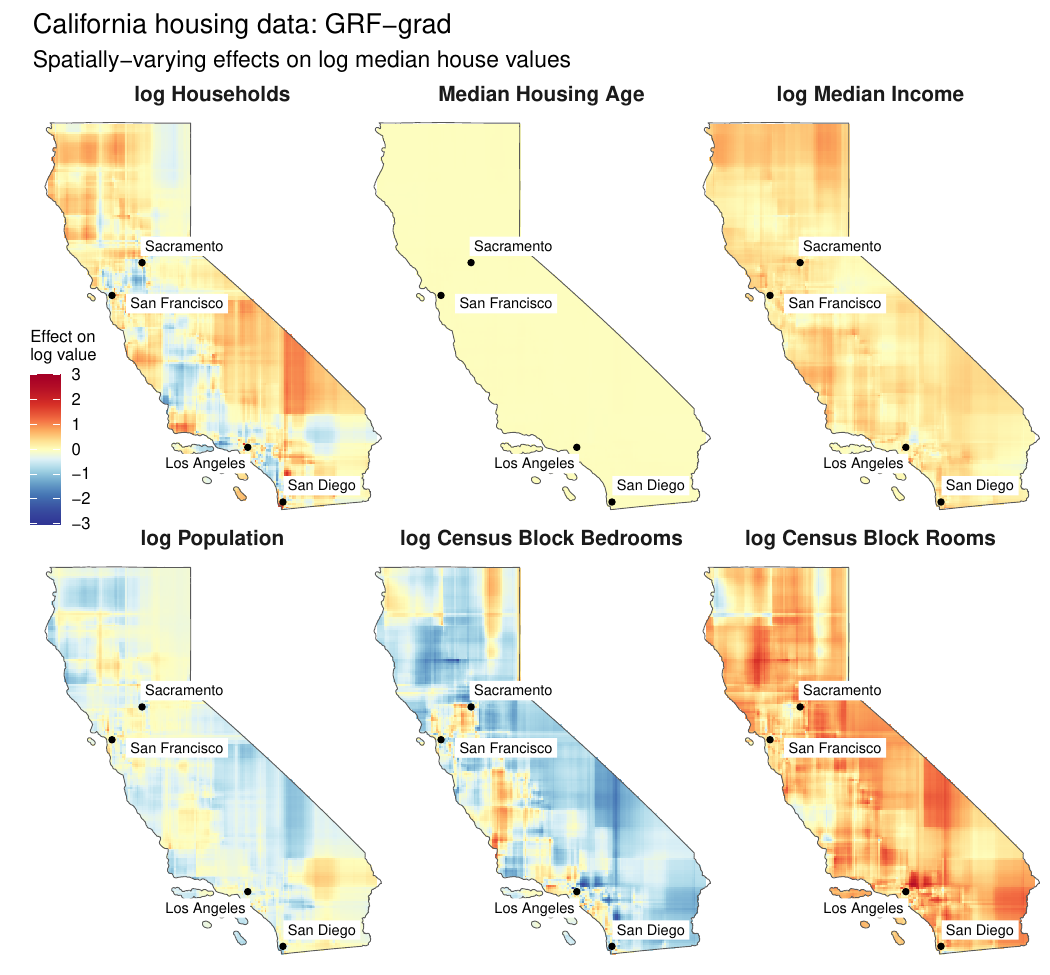}}
    \caption{Geographically-varying local estimates $\hat\theta(x) = (\hat\theta_1(x),\ldots,\hat\theta_6(x))$, fit under GRF-$\FPTone$ (top) and GRF-$\grad$ (bottom). Results for GRF-$\FPTtwo$ are presented in Figure~\ref{fig:california-housing-fpt2} found in Section~\ref{sec:real-data}.}
    \label{fig:california-housing-grad-fpt1}
\end{figure}

\begin{figure}
    \centering
    \makebox[\textwidth][c]{\includegraphics[width=0.9\textwidth]{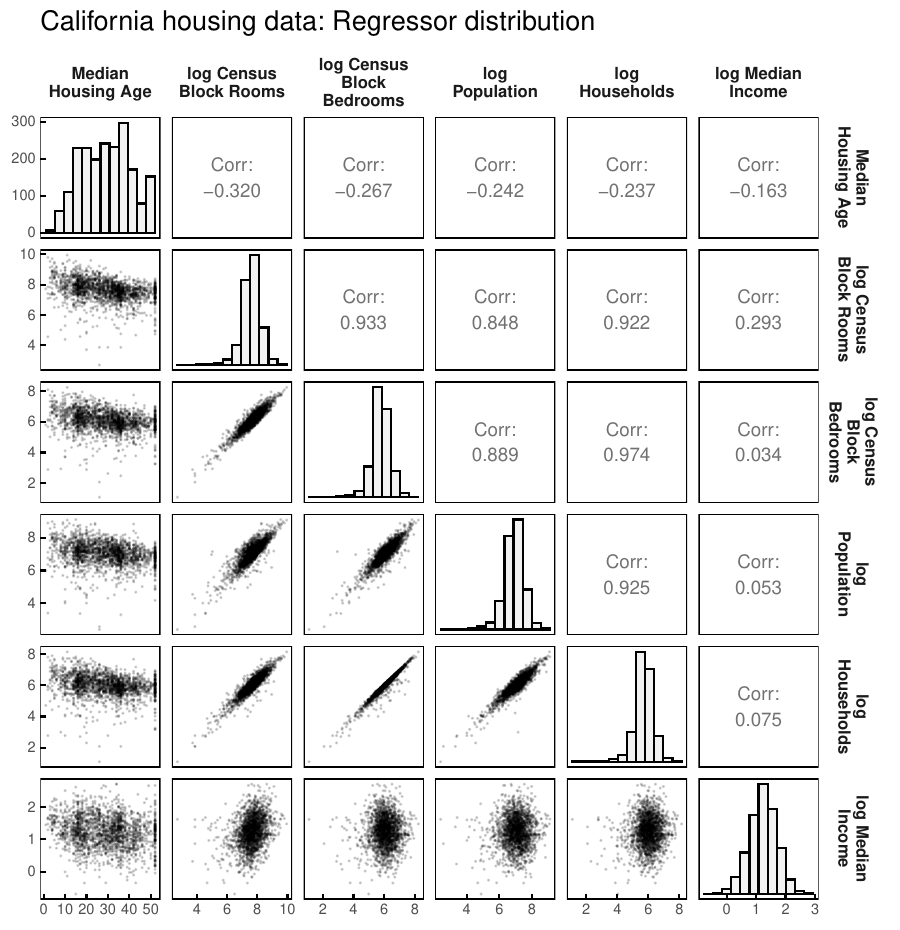}}
    \caption{Empirical distribution of the regressors from the California housing data passed to GRF.}
    \label{fig:california-housing-regressors}
\end{figure}

\end{document}